\documentclass{article}
\pdfoutput=1

\usepackage{arx}

\usepackage{natbib} 
\usepackage{mathtools} 
\usepackage{booktabs} 
\usepackage{tikz} 

\usepackage{amsmath,amssymb,amsthm,mathtools}
\usepackage{hyperref}
\usepackage{cleveref}
\usepackage{enumitem}
\usepackage{lipsum}
\usepackage{algorithm}
\usepackage{algpseudocode}
\newtheorem{assumption}{Assumption}
\newtheorem{definition}{Definition}
\newtheorem{lemma}{Lemma}
\newtheorem{proposition}{Proposition}
\newtheorem{theorem}{Theorem}
\newtheorem{corollary}{Corollary}
\newtheorem{remark}{Remark}

\newcommand{\inner}[2]{\left\langle #1,\, #2 \right\rangle}

\newcommand{\Pcal}{\mathfrak{P}}
\newcommand{\R}{\mathbb{R}}
\newcommand{\E}{\mathbb{E}}

\newcommand{\Wone}{W_1}

\newcommand{\norm}[1]{\left\lVert#1\right\rVert}

\newcommand{\argmin}{\operatorname*{arg\,min}}
\newcommand{\argmax}{\operatorname*{arg\,max}}
\newcommand{\mcs}{\mathcal{S}}
\newcommand{\mca}{\mathcal{A}} 
\newcommand{\N}{\mathbb{N}}

\newcommand{\DeltaS}{\Delta(\mcs)}
\newcommand{\DeltaA}{\Delta(\mca)}
\newcommand{\DD}{\Delta}
\newcommand{\BR}{\mathsf{BR}}
\newcommand{\WC}{\mathsf{WC}}
\newcommand{\Inv}{\mathsf{Inv}}
\newcommand{\supp}{\operatorname{supp}}

\title{Stationary Robust Mean-Field Games under Model Mismatches}

\author{
  Yue Wang \\
  Department of Electrical and Computer Engineering\\
    University of Central Florida\\
    Orlando, Florida, USA
}
 
\begin{document}
\maketitle

\begin{abstract}
Deploying multi-agent reinforcement learning (MARL) in the real world is often limited by model mismatches between the training simulators and the true environment, which could be further amplified through strategic interactions and result in severe performance degradation upon deployment. Distributional robustness offers a principled response by optimizing policies against worst-case transition models drawn from an uncertainty set, but standard robust MARL frameworks become increasingly intractable as the number of agents grows. This paper develops an infinite-horizon, stationary mean-field game framework that incorporates distributional model uncertainty directly into the population-coupled dynamics. We establish a robust dynamic programming principle with a contractive Bellman operator and prove the existence of a stationary robust mean-field equilibrium via a fixed-point argument. We further develop the first concrete algorithm with convergence guarantees. We then connect the mean-field solution to a finite-population robust game whose ambiguity sets depend on the empirical distribution, showing that the mean-field equilibrium policy induces approximate equilibrium behavior as the population size increases. Under a contractive robust-dynamics regime, we further obtain explicit non-asymptotic error bounds. Numerical experiments further illustrate the qualitative and quantitative impact of robustness under multiple uncertainty models, validating our theoretical findings.
\end{abstract}

\section{Introduction}

Multi-agent reinforcement learning (MARL), together with its Markov/stochastic-game formulation
\citep{shapley1953stochastic,littman1994markov}, has become a core paradigm for designing intelligent
multi-agent systems (MAS) that exhibit complex and coordinated behavior.
By enabling multiple decision-makers to learn and act in a shared, evolving environment, MARL has achieved great success in, e.g.,  strategic games \citep{silver2017mastering,vinyals2019grandmaster},
coordination in autonomous transportation and traffic systems \citep{shalev2016safe,hua2024multi},
and distributed robotics \citep{lowe2017multi,matignon2012independent}.

Despite this progress, a fundamental obstacle that limits the reliable deployment of MARL in the real world is the \emph{Sim-to-Real} gap \citep{zhao2020sim,peng2018sim}.
The standard workflow trains policies in a simulator and then deploys them in practice; yet, even high-fidelity simulators inevitably miss aspects of reality, including subtle physical effects, sensor imperfections, unmodeled dynamics, and latent environmental factors \citep{padakandla2020reinforcement,rajeswaran2016epopt}.
As a consequence, policies that appear optimal in simulation can be brittle under model mismatch and may degrade
severely (or fail catastrophically) upon deployment in practice \citep{pinto2017robust}. 

This challenge becomes even more acute in multi-agent systems, as the model mismatches  can be further endogenously amplified by interactions among agents.
A small deviation in one agent's realized dynamics can change its behavior; and this further modifies the effective environment faced by others, prompting responses that further alter the joint dynamics. Such feedback loops can induce strong non-stationarity beyond that caused by strategic adaptation alone
\citep{papoudakis2019dealing,canese2021multi,wong2023deep}. This amplification effect makes the system-level outcome of MARL to be 
highly sensitive to small modeling errors, and makes robustness against model mismatches a crucial concern.

A principled strategy to address model mismatches is \emph{distributional robustness}.
Rather than optimizing for a single nominal model, (distributionally) robust MARL optimize against a family of
plausible transition models (an \emph{uncertainty set}), selecting policies that maximize worst-case performance among them
\citep{zhang2020robust,kardecs2011discounted}.
This minimax viewpoint provides formal performance guarantees whenever the true environment lies in the uncertainty set,
and it often acts as a regularizer that improves generalization under perturbations
\citep{abdullah2019wasserstein,vinitsky2020robust,liu2025distributionally}.
However, existing results for distributionally robust MARL or Markov games suffer from the \emph{curse of multi-agency}, which is a fundamental obstruction that the learning complexity scales exponentially in the number of agents \citep{farhat2026sample,shi2024sample}. Thus, robust RL can become significantly inefficient and infeasible in large-scale multi-agent systems, where robustness is most needed.


On the other hand, Mean-field games (MFGs) \citep{huang2006large,lasry2007mean} offer a classical lens for large-population strategic interactions by exploiting symmetry and weak coupling.
Instead of tracking the full $N$-agent joint state, MFGs approximate agent-wise interactions through the population distribution (the
\emph{mean field}), enabling scalable equilibrium computation and, under appropriate conditions, finite-$N$ approximation guarantees \citep{carmona2013probabilistic,saldi2018markov,anahtarci2023q,cui2021approximately,yardim2023policy,yardim2024mean}.

Given the effectiveness and efficiency of MFGs in large-scale MAS, 
a natural question arises: 
\noindent\textit{Can we utilize the mean-field approximation to reduce the complexity of  \emph{stationary} (time-homogeneous), infinite-horizon distributionally robust Markov games, making robust MARL tractable under large systems?}

We answer this question affirmatively by developing a framework of discounted  stationary distributionally 
robust mean-field games. Our contributions are summarized as follows. 

\textbf{Formulation and Solvability of robust MFGs with infinite-horizon discounted reward.} We first propose the formulation of infinite-horizon robust MFGs under model mismatches, and their solution notions: robust mean-field equilibrium (MFE). We then prove the existence of a robust MFE under a standard assumption, confirming the solvability of  discounted stationary robust MFGs. Different from existing studies on finite-horizon robust MFGs \citep{langner2024markov} which rely on backward construction of non-stationary robust MFE, our study relies on a fixed-point argument to handle the stationarity.


\textbf{Convergent algorithms for robust MFE.} We then design a concrete algorithm to identify the stationary robust MFE. We first develop a fixed-point-based characterization of the robust MFE, which enables us to obtain a robust MFE by finding the fixed point of an operator. We then develop a Robust Best-Response Iteration algorithm, and prove that it converges to a fixed point (and hence finds a robust MFE) under a standard assumption. Our results represent the first convergent algorithms under robust MFGs, providing a concrete algorithmic solution to uncertain MAS.

\textbf{Approximation of finite-player robust games.} We further develop comprehensive studies on the connections between robust MFGs and finite-player robust games (as the player number becomes large), to understand the effectiveness of approximating large-scale MAS through robust MFGs. We first reveal the hardness of such an approximation under the robust setting, and, unlike the non-robust cases, the necessity of additional stabilizing assumptions. We then develop an asymptotic approximation of robust MFG under this additional assumption, and further characterize the non-asymptotic approximation rate under additional quantitative assumptions. Our results hence enable tractable robust MARL under large population. 


\section{Related Work}
\textbf{Mean-field games.} The framework of standard (non-robust) mean-field games (i.e., without uncertainty) are proposed and studied in both continuous-time (see, e.g., \citep{huang2006large,tembine2013risk,huang2010large,gomes2013continuous,lacker2016general,lacker2022case,lacker2019mean,aurell2022optimal,delarue2020master}) and in discrete-time (see, e.g., \citep{adlakha2015mean,biswas2015mean,gomes2010discrete,gast2012mean,elliott2013discrete,moon2015discrete,moon2016discrete,nourian2013linear,saldi2018markov,saldi2019approximate,Elie2020meanfield}). 
We also refer to \citep{carmona2018probabilistic,bensoussan2013mean,gomes2014mean,lauriere2022learning} for survey papers including both settings. 

Under the regime of reinforcement learning in MFGs, extensive studies on  algorithm design and convergence analysis are developed in, e.g.,  \citep{carmona2013probabilistic,saldi2020approximate,anahtarci2023q,cui2021approximately,yardim2023policy,yardim2024mean,perrin2020fictitious,guo2019learning,xie2021learning}. 

However, all of these existing works are for nominal mean-field games, and no model uncertainty is considered.


\textbf{Distributionally robust Markov decision processes and Markov games.} To address the model mismatch and the Sim-to-Real gap, distributionally robust Markov decision processes are first proposed and studied in \citep{iyengar2005robust,nilim2004robustness,wiesemann2013robust}. Based on them, robust single-agent RL has been extensively studied under different settings, e.g., online learning \citep{wang2021online,lu2024distributionally,he2025sample,ghosh2026orvit,ghosh2025scaling}, offline learning \citep{shi2024distributionally,blanchet2023double,wang2024unified,wang2024sample}, or with a generative model \citep{liu2022distributionally,panaganti2022sample,yang2022toward,shi2023curious,wang2024modelfree,xu2023improved,roch2025finite,Roch2025reduction,wang2024robust,wang2023model}. 

It is then further extended to the multi-agent regime, where the framework of distributionally robust Markov games is developed and studied firstly in \citep{kardecs2011discounted,zhang2020robust}. However, designing efficient MARL algorithms for distributionally robust Markov games is significantly challenging. Since learning the worst-case performance requires a global knowledge of the uncertainty, recent research reveals an inherent curse of multi-agency in robust MARL: the sample complexity exponentially depends on the number of agents \citep{shi2024sample,blanchet2023double,farhat2026sample}, unless additional assumptions or oracles are assumed \citep{shi2025breaking,jiao2024minimax,ma2023decentralized,li2025sample}.  
This inherent challenge prevents the deployment of robust MARL in large multi-agent systems, motivating other formulations and efficient solutions.

\textbf{Mean-field games under model mismatches.}
When model mismatches are considered in the MFG framework, a promising approach is to consider robust MFGs. The most related works to ours are \citep{langner2024markov,liang2026mean}. In \citep{langner2024markov}, the framework of robust MFG is considered under the finite-horizon settings, and the existence of a non-stationary MFE is proved. However, their proofs are based on a backward construction argument, which is infeasible under our settings, and we develop a fundamentally different proof technique to handle the infinite-horizon dependence in our problems. Another most relevant work is recently developed in \citep{liang2026mean}, where stationary robust MFGs are considered. However, this work studies a different formulation of the MFE, and only derives its existence under some strong conditions. Hence, although our studies bear similarities to them, our problem formulation and proof technique are fundamentally different. Moreover, neither of these works studies algorithm convergence or non-asymptotic approximation of robust MFGs.

Robust MFGs are also studied in continuous-time settings, e.g., \citep{moon2016linear,moon2016robust,huang2013mean,bauso2016robust}. However, the studies cannot be applied to our discrete settings, and do not subsume our studies. Another line of research studies robust mean-field control problems \citep{zaman2024robust,lauriere2025robust,xu2025robust}, which can be viewed as a cooperative MFG and shares fundamentally different objectives with ours.

\section{Preliminaries}
In massive MAS, computing exact equilibria via traditional game-theoretic methods becomes computationally intractable due to the exponential growth of the joint state-action space. Mean Field Games (MFGs) circumvent this curse by considering the asymptotic regime where the number of agents approaches infinity. In this limit, agents become infinitesimal, anonymous, and indistinguishable. Consequently, direct agent-to-agent interactions are replaced by the interaction between a single representative agent and the macroscopic state distribution of the entire population, termed the mean field. By tracking this aggregate distribution rather than individual microscopic states, the intractable $N$-player game elegantly decouples into a tractable local optimization problem coupled with a global consistency condition.

A stationary MFG \citep{anahtarci2023q, yardim2023policy, guo2019learning, xie2021learning} can be formulated as  a tuple $(\mcs, \mca, P, r,\gamma)$, where $\mcs,\mca$ are the state and action spaces, with transition  dynamics $P:\mcs\times\mca\times\DeltaS\to \DeltaS$ and rewards $r:\mcs\times\mca\times\mathcal{S}\times\DeltaS\to \mathbb{R}$. $\gamma\in(0,1)$ is the discount factor. 

As mentioned, performances in a MFG are determined by the policy of the representative agent and the overall strategy of all other agents. Specifically, these strategies are characterized as a stationary policy $\pi: \mcs\to \Delta(\mca)$ and a population distribution $\mu\in\DeltaS$. For any $(\mu, \pi) \in \Delta_\mcs \times \Pi$, the discounted infinite horizon value function is defined as 
    \begin{align*}
        J_\mu(\pi,P) := \E \left[ \sum_{t=0}^\infty \gamma^t r(s_t, a_t, s_{t+1},\mu) \middle| \substack{s_0 \sim \mu, a_t \sim \pi(s_t) \\ s_{t+1} \sim P(s_t, a_t, \mu) }\right].
    \end{align*}
The goal of a MFG is to find some equilibrium: 
\begin{definition}
A policy-population pair $(\mu^\star, \pi^\star) \in \Delta_\mcs \times \Pi$ is called a Nash-equilibrium (or mean field equilibrium) if the two conditions hold:
    \begin{align*}
        \textit{Stability: } \quad &\mu^\star(s') = \sum_{s,a}\mu^\star(s)\pi^\star(a|s)P(s'|s,a,\mu^\star), \notag \\
        \textit{Optimality: } \quad &J_{\mu^\star}(\pi^\star,P) = \max_{\pi \in \Pi} J_{\mu^\star}(\pi,P). 
    \end{align*}
\end{definition}

\section{Infinite-horizon stationary robust mean-field games}
We then introduce the major objective of our studies, the stationary (time-homogeneous) robust mean-field games. Similarly to a standard MFG, a robust MFG is specified as $(\mcs,\mca, \mathfrak{P}, r)$ \citep{langner2024markov,liang2026mean}, where $\mcs,\mca$ are the finite state and action spaces and $r$ is the reward function. The potential Sim-to-Real gap is modeled through the model uncertainty sets (independently) defined through a set-valued map as 
\begin{align}  \mathfrak{P}: &\mcs\times\mca\times\DeltaS\rightrightarrows \DeltaS,\nonumber\\
&(s,a,\mu)\mapsto \mathfrak{P}(s,a,\mu)\subseteq \DeltaS,
\end{align}
where each element of $\mathfrak{P}(s,a,\mu)$ is a candidate transition law for the next state from current $(s,a)$ and population $\mu$.


Under the model mismatches $\mathfrak{P}$, robust learning takes a principle of pessimism, and considers the worst-case performance over the uncertainty set $\mathfrak{P}$. For a pair $(\pi,\mu)$, the robust value function is defined as: 
\begin{align}\label{eq:robust-value}
V_\mu^\pi  :=\inf_{p\in \mathfrak{P}(\mu)}J_\mu(\pi,p),
V_\mu:=\sup_\pi\inf_{p\in \mathfrak{P}(\mu)}J_\mu(\pi,p),
\end{align}
where $\mathfrak{P}(\mu)\triangleq \otimes_{(s,a)} \mathfrak{P}(s,a,\mu)$ and the infimum takes over kernels
with $p(\cdot\mid s,a)\in\mathfrak{P}(s,a,\mu)$ for any $(s,a)$.

We then extend the definition of MFE and introduce the notion of robust equilibrium as follows.
\begin{definition}[Stationary robust mean-field equilibrium]\label{def:smfe}
A triple $(\mu^\star,\pi^\star,p^\star)$ with $\mu^\star\in\DeltaS$,
$\pi^\star:\mcs\to\DeltaA$, and $p^\star:\mcs\times\mca\times\DeltaS\to\DeltaS$
is a {stationary robust mean-field equilibrium} if:\footnote{In Proposition \ref{prop:history}, we showed that the worst-case can be achieved by stationary kernel, instead of history-dependent ones.}
\begin{enumerate}[label=(\roman*),leftmargin=2.2em]
\item \textbf{Robust optimality:} $\pi^\star$ attains \eqref{eq:robust-value} at $\mu^\star$:
\begin{align}
V_{\mu^\star} =\inf_{p}J_{\mu^\star}(\pi^\star,p).
\end{align}
\item \textbf{Worst-case kernel:} $p^\star$ attains the infimum against $\pi^\star$:
\begin{align}
V_{\mu^\star} =J_{\mu^\star}(\pi^\star,p^\star).
\end{align}
\item \textbf{Consistency (stationarity):} $\mu^\star$ is invariant under $(\pi^\star,p^\star)$:  for any $s'\in\mcs$:
\begin{equation}\label{eq:inv-smfe}
\mu^\star(s')
=\sum_{s\in\mcs}\mu^\star(s)\sum_{a\in\mca}\pi^\star(a|s)\,p^\star(s'|s,a,\mu^\star).
\end{equation}
\end{enumerate}
\end{definition}

Specifically, if we view the three components as three players: a representative player (who controls $\pi$), the population player (controls $\mu$), and an environment player (controls $p$), a robust MFE is a triple where the three players achieve a balanced state. Such a notion is an extension of the Nash equilibrium of standard non-robust MFGs to the worst-case \citep{huang2006large,lasry2007mean}. Moreover, a similar notion of robust MFE is studied in \citep{langner2024markov} for the non-stationary finite-horizon setting. 


\section{Solvability of Robust MFGs}
In this section, we study the solvability of infinite-horizon robust MFGs, i.e., if the stationary robust MFE in \Cref{def:smfe} always exists. 

We first make the following standard assumptions. 
\begin{assumption}\label{ass:stat}
We assume the following conditions hold: (1).  $\mcs$ and $\mca$ are finite; (2). For every $(s,a,\mu)$, the set $\mathfrak{P}(s,a,\mu)$ is nonempty,
convex, and compact in $\DeltaS$.
Moreover, for every fixed $(s,a)$ the correspondence $\mu\mapsto\mathfrak{P}(s,a,\mu)$ is continuous:  
\begin{enumerate}[label=(\alph*)]
\item (\textbf{Closed graph}) if $\mu_n\to \mu$, $p_n\to p$, and $p_n\in \mathfrak{P}(s,a,\mu_n)$ for all $n$, then $p\in \mathfrak{P}(s,a,\mu)$.
\item (\textbf{Lower hemicontinuity}) if $\mu_n\to \mu$ and $p\in \mathfrak{P}(s,a,\mu)$, then there exist $p_n\in \mathfrak{P}(s,a,\mu_n)$ such that $p_n\to p$.
\end{enumerate}
And (3). $r$ is bounded and continuous in $\mu$ (w.r.t. the Euclidean topology on $\DeltaS$). 
\end{assumption}
\begin{remark}
Parts (1) and (3) are standard in standard MFG studies \citep{huang2006large}; Part (2) is adapted from the continuity assumption on the transition kernel in standard MFGs and standard robust RL literature \citep{iyengar2005robust,wang2023robust,wang2025bellman}, and can be easily satisfied by many standard uncertainty sets, e.g., distributionally ambiguous sets defined by divergence.  
\end{remark}

Under these standard assumptions, we then derive the existence result as follows. 
\begin{theorem}[Existence of a stationary robust mean-field equilibrium]\label{thm:existence-smfe}
Under Assumption~\ref{ass:stat}, there exists a stationary robust mean-field equilibrium
$(\mu^\star,\pi^\star,p^\star)$.
\end{theorem}

Our result hence implies that, for an infinite horizon robust MFG with discounted reward, there always exists a mean-field equilibrium, and the game is always solvable. We also highlight that our proof is fundamentally different from the finite-horizon case \citep{langner2024markov}, where the non-stationary mean-field equilibrium can be constructed through a backward induction, whereas in our setting, the existence of a stationary equilibrium is derived through fixed-point arguments.

\section{Approximation of Finite-Player Robust Games}
In this section, we aim to show that, under some assumptions, the policy $\pi^\star$ of a robust MFE constitutes an approximate Nash equilibrium of the finite-player robust Markov game (defined below). We study both asymptotic and non-asymptotic convergence, developing a comprehensive understanding of the connections.

\subsection{$N$-player robust games}
In this section, we first introduce the finite $N$-player robust game. Fix $N\in\N$, a $N$-player robust game is specified as $(N,\mcs,\mca,\mathfrak{P},r,\mu^\star)$. For each agent $i\in N=\{1,\dots,n\}$, it has $s_t^i\in\mcs$ and $a_t^i\in\mca$ as its local (individual) state/action, with $s_0^i \sim \mu^\star$. Denote the joint state and action as  $s_t^N:=(s_t^1,\dots,s_t^n)\in\mcs^N$ and $a_t^N:=(a_t^1,\dots,a_t^n)\in\mca^N$. 
At each time-step $t$, the empirical state distribution $e^N \in\DeltaS$ is
\begin{equation}\label{eq:empirical}
e^N(s_t^N):=\frac1n\sum_{i=1}^n\delta_{s_t^N=s_t^i}, \quad \forall s^N_t\in\mathcal{S}.
\end{equation}
For a pair $(s^N,a^N)\in\mcs^N\times\mca^N$, the environment transition will follow some kernel from the uncertainty set under the empirical state distribution $e^N(\cdot)$
\begin{equation}\label{eq:PN}
\mathfrak{P}^N(s^N,a^N):= \bigotimes_{i=1}^n \mathfrak{P}(s^i,a^i,e^N(s^N)).
\end{equation}

In the game, let $\Pi:=\{\pi:\mcs\to\DeltaA\}$ and $\Pi^N:=\Pi^N$ be profiles
$\pi^N=(\pi^1,\dots,\pi^N)$ with $\pi^i\in\Pi$. For $s^N$, the joint action distribution from the joint (product) policy $\pi^N$ is 
\begin{align}
    \pi^N(da^N\mid s^N):=\bigotimes_{i=1}^n\pi^i(da^i\mid s^i).
\end{align}
Given a profile $\pi^N\in \Pi^N$, consider sequences of (possibly time-inhomogeneous) transition kernels
$(p_t^N)_{t\ge 0}$ on $\mcs^N$ such that for every $t$ and every $(s^N,a^N)\in \mcs^N\times \mca^N$,
\begin{align}
p_t^N(\cdot \mid s^N,a^N)\ \in\ \mathfrak{P}^N(s^N,a^N).
\end{align}
Given $(\pi^N,(p_t^N)_{t\ge0})$, let $\mathbb{P}^{\pi^N,(p_t^N)}$ denote the induced law of the controlled process
generated by $s_0^N\sim (\mu^\star)^{\otimes N}$, $a_t^N\sim \pi^N(\cdot\mid s_t^N)$, and
$s_{t+1}^N\sim p_t^N(\cdot\mid s_t^N,a_t^N)$.
We write $\E^{\pi^N,(p_t^N)}$ for expectation under $\mathbb{P}^{\pi^N,(p_t^N)}$.

The robust payoff of player $i$ is then defined as 
\begin{align}\label{eq:JiN}
&J_i^N(\pi^N)\\
&:= \inf_{(p_t^N)}\;
\E^{\pi^N,(p_t^N)}\Big[\sum_{t\ge0}\gamma^t\, r(s_t^i,a_t^i,s_{t+1}^i, e^N(s_t^N))\Big],\nonumber
\end{align}
where the infimum ranges over all sequences $(p_t^N)_{t\ge0}$ satisfying the constraint above. 


With this objective, we can define the Nash equilibrium notions as in standard games. 
\begin{definition}\label{def:eps-mne}
A profile $\pi^{N,\star}\in\Pi^N$ is an $\varepsilon$-Nash equilibrium if for all $i$, $J_i^N(\pi^{N,\star})+\varepsilon\ \ge\ \sup_{\pi\in\Pi}J_i^N(\pi^{N,\star,-i},\pi)$, where $(\pi^{N,\star,-i},\pi)$ is the joint policy where agent $i$ takes $\pi$ while others follow $\pi^{N,\star}$.  
\end{definition}

\begin{remark}
    The $N$-player robust games we defined are closely related to standard robust Markov games \citep{zhang2020robust,kardecs2011discounted} (see \Cref{app:n-player and drmg}).
\end{remark}



\subsection{Asymptotic Approximation of $N$-Player Robust Games}
\label{sec:asymptotic-approximation}

In this section, we show that the robust MFG provides an asymptotic approximation of the finite-player robust games defined above. Specifically, we show that, when the number of agents is sufficiently large, the equilibrium policy of the robust MFG is an $\varepsilon$-Nash equilibrium of the $N$-player robust game.

Define the symmetric equilibrium profile
\begin{align}
\pi^{N|\star}:=(\pi^\star,\ldots,\pi^\star)\in \Pi^N .
\end{align}
Fix any unilateral deviation sequence $(\pi^{(N)})_{N\in\mathbb N}\subset \Pi$ for player $1$, and set
\begin{align}
\pi^{N|(N)} := (\pi^{(N)},\pi^\star,\ldots,\pi^\star)\in \Pi^N .
\end{align}
For each $N$, let $P^{N|(N)}$ be a minimizing law attaining
$J_1^N(\pi^{N|(N)})$, whose existence is guaranteed by Lemma~\ref{lem:finite-player-robust-dp} in Appendix~C. Under $P^{N|(N)}$, define the one-step law
\begin{align}
Q_t^{N|(N)}
:=
\operatorname{Law}_{P^{N|(N)}}\!\bigl(s_t^1,a_t^1,s_{t+1}^1,e^N(s_t^N)\bigr).
\end{align}
Next, define the proxy chain $P^{*(N)}$ by
\begin{align}
s_0\sim \mu^\star,
a_t\sim \pi^{(N)}(\cdot\mid s_t),
s_{t+1}\sim p^\star(\cdot\mid s_t,a_t),
\end{align}
and let
\begin{align}
Q_t^{*(N)}
:=
\operatorname{Law}_{P^{*(N)}}\!\bigl(s_t,a_t,s_{t+1},\mu^\star\bigr),
\qquad t\in\mathbb N_0 .
\end{align}

\begin{assumption}
\label{ass:one-step-law}
For every deviation sequence $(\pi^{(N)})_{N}\subset\Pi$, every $t\in\N_0$, and some choice of minimizing laws $P^{N|(N)}$,
\begin{align}
    W_1\bigl(Q^{N|(N)}_t,\,Q^{*(N)}_t\bigr)\longrightarrow 0\qquad (N\to\infty),
\end{align}
with $W_1$ on $\Delta(Z)$ as in \eqref{eq:Z}.
\end{assumption}

This assumption states that, along any unilateral deviation sequence, the deviating player asymptotically faces the same one-step law as in the proxy mean-field model driven by $(\mu^\star,\pi^{(N)},p^\star)$. In particular, both the player-side transition and the population term entering the reward converge, at the level of one-step distributions, to their mean-field counterparts.

In the robust setting, such a stabilization property does not follow from Assumption~\ref{ass:stat} alone, because nature can use the remaining $N-1$ players to alter the empirical distribution in a way that persists at the level of the one-step law. The following counterexample shows that this additional assumption is genuinely needed.

\begin{theorem}
\label{thm:counterexample-onestep}
Fix $\gamma\in(0,1)$ and $\kappa>1/\gamma$, and set
\[
\varepsilon_0\;:=\;\frac{\gamma\kappa-1}{1-\gamma}\;>\;0.
\]
There exist a robust mean-field game satisfying Assumption~\ref{ass:stat} and a stationary robust mean-field equilibrium $(\mu^\star,\pi^\star,p^\star)$ of it such that:
\begin{enumerate}[label=(\alph*)]
  \item Assumption~\ref{ass:one-step-law} fails for the constant deviation sequence $\pi^{(N)}\equiv\pi^\star$ under \emph{every} choice of minimizing laws $P^{N|(N)}$; and
  \item for \emph{every} $N\in\mathbb{N}$, the symmetric profile $\pi^{N|\star}$ is not an $\varepsilon$-Nash equilibrium of the $N$-player robust game for any $\varepsilon<\varepsilon_0$.
\end{enumerate}
\end{theorem}

We then derive the asymptotic approximation result.

\begin{theorem}
\label{thm:asymptotic-eps-nash}
Assume Assumptions~\ref{ass:stat} and~\ref{ass:one-step-law}. Let
$(\mu^\star,\pi^\star,p^\star)$ be a stationary robust mean-field equilibrium and let
\begin{align}
\pi^{N|\star}:=(\pi^\star,\ldots,\pi^\star)\in \Pi^N .
\end{align}
Then, for every $\varepsilon>0$, there exists $N(\varepsilon)\in\mathbb N$ such that, for all $N\ge N(\varepsilon)$, the profile $\pi^{N|\star}$ is an $\varepsilon$-Nash equilibrium of the $N$-player robust game.
\end{theorem}

\subsection{Non-Asymptotic Approximation}
\label{sec:nonasymptotic-approximation}

We now derive finite-$N$ approximation bounds. We will mainly study the difference between the actual $N$-player robust game and the proxy chain driven by the equilibrium worst-case kernel $p^\star$. Specifically, let
\begin{align}
X := \mathcal{S}\times\mathcal{A}\times\mathcal{S},
\qquad
Z := X\times \Delta(\mathcal{S}),
\end{align}\label{eq:Z}
and equip $Z$ with the metric
\begin{align}
d_Z\bigl((x,\nu),(\tilde x,\tilde \nu)\bigr)
:=
\mathbf 1_{\{x\neq \tilde x\}} + \|\nu-\tilde \nu\|_1,
\end{align}
and let $W_1$ be the corresponding Wasserstein-$1$ distance on $\mathcal P(Z)$.

In contrast with the asymptotic approximation result, the finite-$N$ analysis requires an explicit rate rather than mere weak convergence. We therefore impose a quantitative hypothesis directly on the one-step law faced by the deviating player. Our assumption does not require a particular realization of the minimizing $N$-player kernel, and it only controls the law of the local transition together with the empirical distribution.

\begin{assumption}
\label{ass:proxy-comparison}
There exists a deterministic array
$(\delta_{N,t})_{N\in\mathbb N,\; t\in\mathbb N_0}$ with $\delta_{N,t}\ge 0$ such that, for every $N\in\mathbb N$, every unilateral deviation policy $\pi\in \Pi$, and every $t\in\mathbb N_0$,
\begin{align}
W_1\!\left(Q_t^{N,\pi},Q_t^\pi\right)\le \delta_{N,t},
\end{align}
where $Q_t^{N,\pi}$ is the law of $\bigl(s_t^1,a_t^1,s_{t+1}^1,e^N(s_t^N)\bigr)$ 
under the actual $N$-player robust game with profile
$\pi^{N|\star,-1}\oplus \pi$, and $Q_t^\pi$ is the law of $\bigl(s_t,a_t,s_{t+1},\mu^\star\bigr)$ 
under the proxy chain
\begin{align}
s_0\sim \mu^\star,\quad
a_t\sim \pi(\cdot\mid s_t),\quad
s_{t+1}\sim p^\star(\cdot\mid s_t,a_t).
\end{align}
\end{assumption}
\begin{remark}
\label{rem:interpret-nonasymp}
Assumption~\ref{ass:proxy-comparison} is the quantitative counterpart of the one-step law convergence used in the asymptotic approximation. It is formulated directly at the level of
\[
Z_t^{N,\pi}:=(s_t^1,a_t^1,s_{t+1}^1,\mu_t^N)
\]
under the deviating $N$-player robust game and the proxy variable
\[
Z_t^\pi:=(s_t,a_t,s_{t+1},\mu^\star)
\]
under the mean-field worst-case kernel $p^\star$.

We note that the projections $(x,\nu)\mapsto \nu$ and $(x,\nu)\mapsto x$ are $1$-Lipschitz under the product metric used in the assumption. Consequently, Assumption~\ref{ass:proxy-comparison} simultaneously controls the empirical distribution and the local one-step dynamics. In particular,
\[
W_1\bigl(\mathrm{Law}(\mu_t^N),\delta_{\mu^\star}\bigr)\le \delta_{N,t},
\text{ hence }
\mathbb{E}\bigl[\|\mu_t^N-\mu^\star\|_1\bigr]\le \delta_{N,t},
\]
and also
\[
W_1\!\Bigl(\mathrm{Law}(s_t^1,a_t^1,s_{t+1}^1),\mathrm{Law}(s_t,a_t,s_{t+1})\Bigr)\le \delta_{N,t},
\]
where on $\mathcal{S}\times \mathcal{A}\times \mathcal{S}$ we use the discrete metric. Since the latter space is finite, this is equivalent to total-variation control of the local triple. Thus $\delta_{N,t}$ measures a combined finite-$N$ Nash-certainty-equivalence error for both the deviator and the population.


Finally, the bound is required uniformly over unilateral deviations $\pi$, because the Nash gap involves the supremum over all such deviations. The dependence of $\delta_{N,t}$ on $t$ is also natural: finite-$N$ coupling errors may accumulate with time, and the theorem only needs the discounted series in the final bound to remain finite.
\end{remark}

We further make the following standard Lipschitz assumption on reward, and derive the following finite-$N$ equilibrium guarantee\footnote{In \Cref{app:finiteN-proxy}, we derive the results under a more general H\"older assumption.}. 
\begin{assumption}
\label{ass:holder-reward}
There exist $L_r>0$ such that, for all
$s\in S$, $a\in A$, $s'\in S$, and all $\nu,\tilde \nu\in \Delta(\mathcal{S})$,
\begin{align}
|r(s,a,s',\nu)-r(s,a,s',\tilde \nu)|
\le
L_r \|\nu-\tilde \nu\|_1.
\end{align}
\end{assumption}

\begin{theorem}
\label{thm:finiteN-holder}
Assume Assumptions~\ref{ass:stat}, \ref{ass:proxy-comparison}, and~\ref{ass:holder-reward}. Let
$(\mu^\star,\pi^\star,p^\star)$ be a stationary robust mean-field equilibrium and let
\begin{align}
\pi^{N|\star}:=(\pi^\star,\ldots,\pi^\star)\in \Pi^N .
\end{align}
Then $\pi^{N|\star}$ is an $\varepsilon_N$-Nash equilibrium of the $N$-player robust game, where
\begin{align}
\varepsilon_N
:=
2\sum_{t=0}^\infty \gamma^t
\Bigl(
2\|r\|_\infty
+
L_r
\Bigr)\delta_{N,t}.
\end{align}
Moreover, if there exists $C_\delta>0$ such that\footnote{The discussion of this condition is deferred to \Cref{sec:appendix-d5}.}
\begin{align}
\delta_{N,t}\le \frac{C_\delta(1+t)}{\sqrt N}
\qquad\forall\, N\in\mathbb N,\; t\in\mathbb N_0,
\end{align}
then
\begin{align}
\varepsilon_N
\le \frac{2(2\|r\|_\infty+L_r)C_\delta}{(1-\gamma)^2\sqrt N}=\mathcal{O}\!\left(N^{-1/2}\right).
\end{align}
\end{theorem}



\section{Algorithmic Solutions for Robust MFGs}
In this section, we further propose an iterative algorithm to compute a mean-field equilibrium of the robust MFG. We will first develop a fixed-point characterization of the mean-field equilibrium, and then design an algorithm with convergence guarantees.

\subsection{Mean-Field Equilibria as Fixed Points}

Fix $\mu\in\Delta(\mathcal{S})$. For $v\in\mathbb{R}^S$ define the robust $Q$-operator as 
\begin{align}
(Q_\mu v)(s,a)\triangleq
\min_{P\in\mathfrak{P}(s,a,\mu)}
\sum_{s'\in S} P(s')\Big(r(s,a,s',\mu)+\gamma v(s')\Big),
\nonumber
\end{align}
and the robust Bellman operator as
\begin{align}
(T_\mu v)(s) \triangleq \max_{a\in A} (Q_\mu v)(s,a).
\label{eq:alg_Tmu}
\end{align}
Since $T_\mu$ is a $\gamma$-contraction in $\|\cdot\|_\infty$, it admits a unique fixed point $v_\mu$ satisfying $v_\mu = T_\mu v_\mu,$ 
and we further define the optimal robust $Q$-function by $Q_\mu(s,a) ~:=~ (Q_\mu v_\mu)(s,a).$

Moreover, a robust best response at $\mu$ can be obtained by selecting  
\begin{align}
a_\mu(s)\in \arg\max_{a\in A} Q_\mu(s,a),\qquad \pi_\mu(\cdot\mid s)=\delta_{a_\mu(s)},
\label{eq:alg_policy_selection}
\end{align}
and for each $(s,a)$, set $p_\mu(\cdot\mid s,a)$ as 
\begin{align} 
\arg\min_{P\in\mathfrak{P}(s,a,\mu)}
\E_{s'\sim P}\Big[r(s,a,s',\mu)+\gamma v_\mu(s')\Big].
\label{eq:alg_kernel_selection}
\end{align}
This triple $(\mu,\pi_\mu,p_\mu)$ hence induces a Markov kernel on $\mcs$: 
\begin{align}
K_\mu(s'\mid s) ~:=~ \sum_{a\in A} \pi_\mu(a\mid s)\,p_\mu(s'\mid s,a),
\label{eq:alg_Kmu}
\end{align}
and we define an operator $F:\DeltaS\to\DeltaS$ as
\begin{align}
F(\mu) ~:=~ \mu K_\mu.
\label{eq:alg_F}
\end{align}

We then prove that, any fixed point of $F$ is a MFE.
\begin{lemma}
\label{lem:alg_fixed_point_equiv}
Let $\mu\in\Delta(\mathcal{S})$ and suppose $(\pi_\mu,p_\mu)$ are selected as in
\eqref{eq:alg_policy_selection}--\eqref{eq:alg_kernel_selection}.
If $\mu$ satisfies the fixed-point condition
\begin{align}
\mu = F(\mu)=\mu K_\mu,
\label{eq:alg_mu_fixed_point}
\end{align}
then $(\mu,\pi_\mu,p_\mu)$ is a stationary robust MFE.
\end{lemma}
This result thus enables us to reduce the problem of solving a robust MFG to finding a fixed point of $F$. 
Based on this, we design our algorithm as follows.


\begin{algorithm}[!htb]
\caption{Robust Best-Response Iteration}
\label{alg:rb_picard}
\begin{algorithmic}[1]
\State \textbf{Input} initial distribution $\mu^0 \in \Delta(\mathcal{S})$; stepsize $\alpha$; DP tolerance $\varepsilon_{\mathrm{DP}} > 0$; mean-field tolerance $\varepsilon_{\mu} > 0$.
\For{$k=0,1,2,\dots$}
\State Initialize $v^{(0)} \in \mathbb{R}^S$, $m\leftarrow 0$
\While{$\|v^{(m+1)} - v^{(m)}\|_\infty > \varepsilon_{\mathrm{DP}}$}
    \State $v^{(m+1)} \leftarrow T_{\mu^k} v^{(m)}$
    \State $m \leftarrow m+1$
    \EndWhile
    \State $v_{\mu^k} \leftarrow v^{(m)}$
    \State  Select $a^k(s) \in \arg\max_{a \in A} (Q_{\mu^k} v_{\mu^k})(s,a)$ \State $\pi^k(\cdot \mid s) \leftarrow \delta_{a^k(s)}$
    \State Select $p^k(\cdot \mid s,a) \in \arg\min_{P \in \mathfrak{P}(s,a,\mu^k)} \sum_{s' \in S} P(s') \Big( r(s,a,s',\mu^k) + \gamma v_{\mu^k}(s') \Big)$.
    \State $K^k(s' \mid s) \leftarrow \sum_{a \in A} \pi^k(a \mid s) p^k(s' \mid s,a)$ 
    \State $\tilde{\mu}^{k+1} \leftarrow \mu^k K^k$.
    \State $\mu^{k+1} \leftarrow (1-\alpha)\mu^k + \alpha \tilde{\mu}^{k+1}$.
    \If{$\|\mu^{k+1} - \mu^k\|_1 \le \varepsilon_{\mu}$}
        \State \textbf{break}
    \EndIf
\EndFor
\State \Return $(\mu^{k+1}, \pi^k, p^k)$.
\end{algorithmic}
\end{algorithm}

\subsection{Convergence Analysis}
\label{subsec:alg_convergence}
Then we derive the convergence analysis of the algorithm. We first highlight that, even in a non-robust stationary MFG, it is not always guaranteed that $F$ has a fixed point \citep{huang2006large}, and additional structural assumptions are needed, e.g., \citep{cardaliaguet2017learning,perrin2020fictitious,geist2021concave,perolat2022mastering}. We hence adapt the standard contraction assumption as \footnote{To make \Cref{ass:alg_contractivity} more applicable, we further extend our algorithm design and results by replacing $\arg\max$ in \eqref{eq:alg_policy_selection} by a softmax policy, which allows us to ensure convergence under weaker and more verifiable assumptions. See our discussion in \Cref{subsec:soft_rb_picard}.}.

\begin{assumption}
\label{ass:alg_contractivity}
For a Markov kernel $K$, define its Dobrushin coefficient by
\begin{align}
\alpha(K) ~:=~ \max_{s,\tilde s\in S} d_{\mathrm{TV}}\big(K(\cdot\mid s),K(\cdot\mid \tilde s)\big).
\label{eq:alg_dobrushin}
\end{align}
Assume there exist constants $\rho_{\mathrm{mix}}\in[0,1)$ and $L_K\ge 0$ such that for all $\mu,\tilde\mu\in\Delta(\mathcal{S})$:
\begin{align}
\alpha(K_\mu) &\le \rho_{\mathrm{mix}},
\label{eq:alg_mix_assump}\\
\max_{s\in S}\|K_\mu(\cdot\mid s)-K_{\tilde\mu}(\cdot\mid s)\|_1
&\le
L_K\|\mu-\tilde\mu\|_1.
\label{eq:alg_LK_assump}
\end{align}
And it holds that $\rho ~:=~ \rho_{\mathrm{mix}}+L_K ~<~ 1.$
\end{assumption}
Notably, as discussed in \Cref{subsec:alg_discussion}, a broad range of distributional uncertainty sets satisfy the approximate Lipschitz condition, and thus our results hold for these sets with small Lipschitz coefficients. 
\begin{remark}
The assumption is standard and inevitable even in non-robust MFGs \citep{huang2006large,carlini2014fully,
guo2019learning,anahtarci2023q,yardim2024mean,yang2026mean}, and any best-response-based or iteration-typed algorithms can be unstable without it \citep{cui2021approximately,
chassagneux2019numerical,
lauriere2021numerical}. We hence adapt this assumption to our robust setting.

On the other hand, there are extensive studies developing different techniques to ensure stability and convergence under weaker conditions, e.g., \citep{cardaliaguet2017learning,hadikhanloo2019finite,hadikhanloo2018learning,perrin2020fictitious,
geist2021concave,
lavigne2023generalized,
delarue2025exploration,
subramanian2019reinforcement,
mguni2018decentralised,
angiuli2022unified,
xie2021learning}. We leave adaptations of these techniques to robust MFGs as our future work. 
\end{remark}

Under this assumption, we can prove the convergence of our \Cref{alg:rb_picard}. Note that the algorithm updates may not be exact due to tolerance errors. We thus develop our analysis with an implemented population update $\widehat F(\mu^k)$ satisfying
\begin{align}
\|\widehat F(\mu^k)-F(\mu^k)\|_1 \le \varepsilon_k.
\label{eq:alg_inexact_update}
\end{align}

\begin{theorem}
\label{thm:alg_inexact}
Assume Assumption~\ref{ass:alg_contractivity} and that Algorithm~\ref{alg:rb_picard} uses the inexact update
$\mu^{k+1}=(1-\alpha)\mu^k+\alpha\widehat F(\mu^k)$ with \eqref{eq:alg_inexact_update}.
Then, denoting $q:=1-\alpha(1-\rho)<1$, it holds that
\begin{align}
\|\mu^{k}-\mu^\star\|_1
~\le~
q^{k}\|\mu^{0}-\mu^\star\|_1
~+~
\alpha\sum_{j=0}^{k-1} q^{k-1-j}\varepsilon_j.
\end{align}
In particular, if $\sup_j\varepsilon_j\le \bar\varepsilon$, then
$\limsup_{k\to\infty}\|\mu^{k}-\mu^\star\|_1 \le \bar\varepsilon/(1-\rho)$.
\end{theorem}

Our result hence implies that Algorithm~\ref{alg:rb_picard} will converge to the population distribution $\mu^\star$ of a mean-field equilibrium. Moreover, by learning the corresponding optimal robust policy and the worst-case kernel $(\pi_{\mu^\star},p_{\mu^\star})$ under $\mu^\star$, we obtain a mean-field equilibrium. Our algorithm thus stands as the first concrete algorithm for stationary infinite-horizon robust MFGs with convergence guarantees.

\begin{remark} In our algorithm, computing a robust MFE thus reduces to iterating between (1) solving a robust MDP at fixed $\mu$ via contraction-based methods, and (2) updating $\mu$ via the induced Markov kernel. For a robust MDP, each robust Bellman update cost is polynomial in $S,A$ \citep{iyengar2005robust,nilim2004robustness}. Thus, the total complexity of solving a robust MFG does not scale with $N$, whereas the complexity of solving a standard robust Markov game generally scales as $A^N$ \citep{farhat2026sample}. Thus, our robust MFG provides a formulation of large MAS under model mismatches and circumvents the curse of multi-agency.
\end{remark}

\section{Numerical Experiments}
\label{sec:experiments}
In this section, we develop numerical experiments to validate our theoretical results. 

\textbf{Experiment environments.}
In our experiments, we construct a robust MFG with 
$
\mcs=\{0,1,\dots,S-1\}$  and $ \mca=\{0,1\}$. Action $a=0$ corresponds to \emph{stay}, while $a=1$ corresponds to \emph{move clockwise} (to $s+1$) with slip.

Let $u\in\DeltaS$ denote the uniform distribution on $\mcs$ and let $\eta\in(0,1)$ be a mixing parameter. We first define transition kernels $\bar p_0(\cdot\mid s,a)$ under different $(s,a)$ pairs, and set  the nominal transition kernel as the uniformly-mixed kernel $p_0(\cdot\mid s,a)\;:=\;(1-\eta)\,\bar p_0(\cdot\mid s,a)\;+\;\eta\,u(\cdot).$

Fix a goal state $s_g\in\mcs$ and parameters $R_g>0$ (goal reward), $\lambda_{\mathrm{cong}}\ge 0$ (congestion strength),and action costs $c(0),c(1)\ge 0$. We set the one-step reward as $r(s,a,s',\mu)\;:=\;R_g\,\mathbf 1\{s'=s_g\}\;-\;\lambda_{\mathrm{cong}}\,\mu(s')\;-\;c(a).$

We evaluate robustness under two rectangular ambiguity models defined by $D$ being $L_1$-norm and KL-divergence (for simplicity, we set them to be independent of $\mu$): $\mathfrak{P}_(s,a)
:=\Big\{p:\ D(p,p_0(\cdot\mid s,a))\le \rho\Big\}.$

\textbf{Experiment results.}
We then develop two experimental protocols to validate our theoretical results.

\textit{Robustness of robust MFGs.} Under both uncertainty sets, we will first obtain the robust MFE through our \Cref{alg:rb_picard} under a specific radius, and evaluate the robust value function of it under robust MFGs with different radii. Moreover, we find the non-robust MFE and plot its robust value under these robust MFGs as baselines. 

We plot the robust value v.s. the uncertainty set radius in \Cref{fig:robust}. As shown by the results, our robust MFE is much more reliable and stable when facing model uncertainties, whereas the vanilla MFE suffers from severe performance degradations under model mismatches. These experiments hence verify the enhanced robustness and stability of our robust MFG formulation, validating our theoretical results. 

\begin{figure}[!htb]
    \centering
\includegraphics[width=0.95\linewidth]{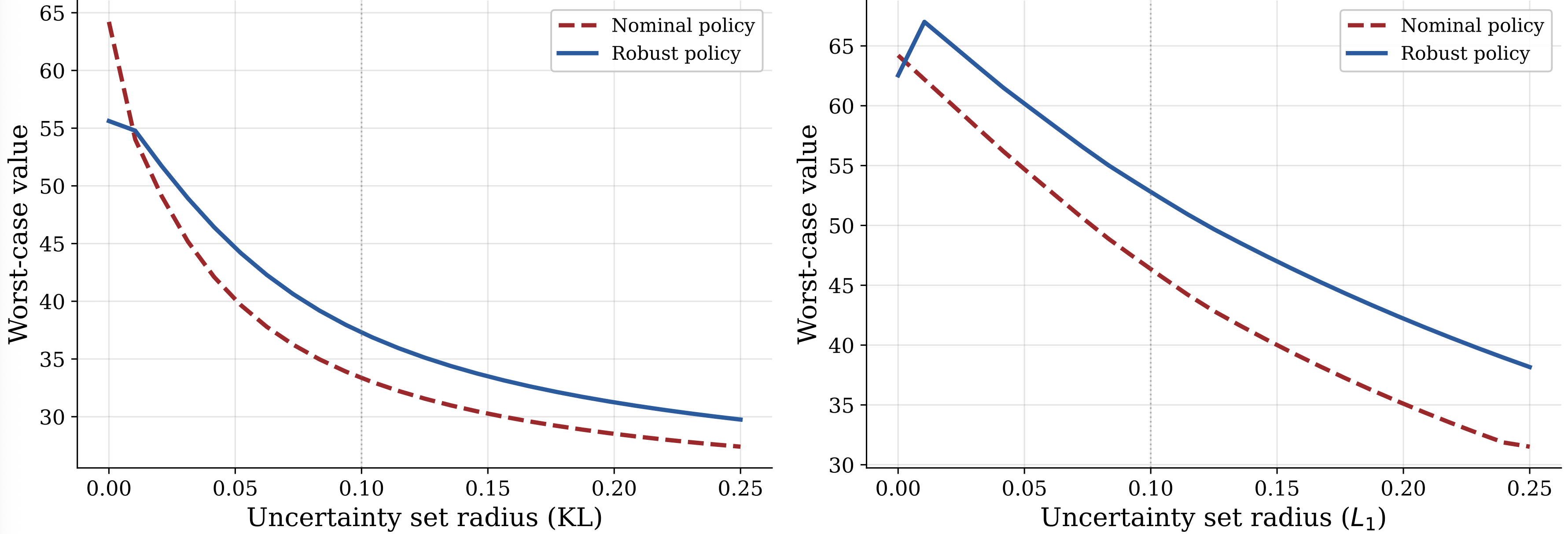}
    \caption{Robust MFE v.s. Non-robust MFE}
    \label{fig:robust}
\end{figure}

\textit{Approximation of finite-player robust games.} We then fix an ambiguity radius and compute the corresponding robust MFE $(\mu^\star,\pi^\star,p^\star)$.
Then, for each $N$, we compute the Nash gap of $\pi^\star$ under $\mu^\star$ and the corresponding $N$-player robust game, and plot the gap v.s. the number of players. 

Our results are presented in \Cref{fig:approximation}. As shown in the results, when the player number $N$ increases, the Nash gap of the robust MFE diminishes, i.e., it becomes an $\varepsilon$-Nash equilibrium in the $N$-player game. Moreover, the approximation slope is $\frac{1}{2}$ in the $\log$-scale, validating our non-asymptotic convergence rate. These results hence validate our theoretical results. 

\begin{figure}[!htb]
   \includegraphics[width=0.95\linewidth]{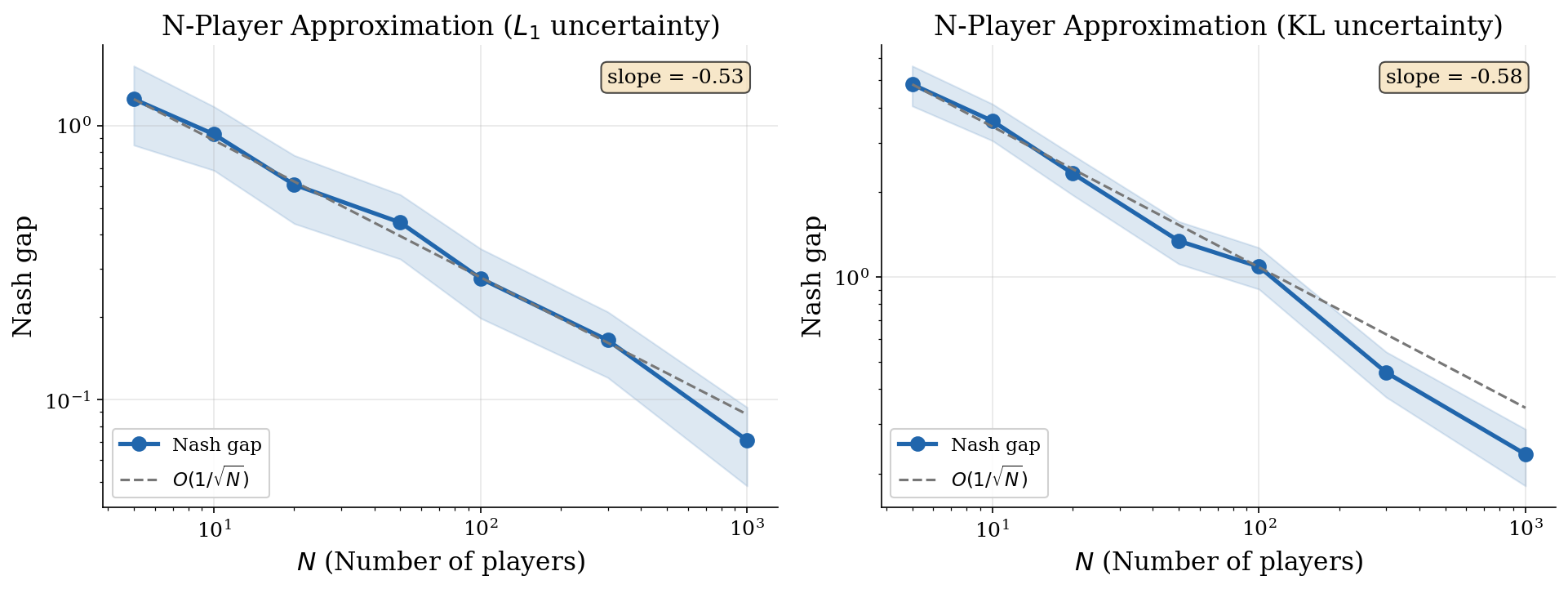}
    \caption{Approximation of $N$-Player Robust Games}
    \label{fig:approximation}
\end{figure}

\section{Conclusion}
In this paper, we proposed and investigated stationary robust mean-field games (MFGs). We first proved its solvability, and developed comprehensive studies on the approximation of the robust MFGs to finite-agent robust games, revealing both asymptotic and non-asymptotic approximations as the multi-agent system becomes large.  We further designed a concrete algorithm for robust MFGs with provable convergence guarantees. Our studies provide a tractable foundation for large-population multi-agent systems under model mismatches, which serve as a potential efficient solution to close the Sim-to-Real gap in large MAS.

\section*{Acknowledgements}
This work was supported by an Amazon Research Award, Fall 2025. Any opinions, findings, and conclusions or recommendations expressed in this material are those of the author(s) and do not reflect the views of Amazon.

\newpage
\bibliography{uai2026-template}

\newpage
\appendix
\onecolumn
\appendix
\onecolumn

\section{Stationary discounted robust mean-field games}\label{app:smfg}

\subsection{Robust discounted dynamic programming}

For fixed $\mu\in\DeltaS$ and $v\in\R^{\mcs}$ define the robust $Q$-operator
\begin{equation}\label{eq:Qop}
(\mathcal{Q}_\mu v)(s,a)
:=\min_{P\in\mathfrak{P}(s,a,\mu)}\ \sum_{s'\in\mcs} P(s')\Big(r(s,a,s',\mu)+\gamma v(s')\Big).
\end{equation}
Define the robust Bellman operator $\mathcal{T}_\mu:\R^{\mcs}\to\R^{\mcs}$ by
\begin{equation}\label{eq:Bellman-op}
(\mathcal{T}_\mu v)(s):=\max_{a\in\mca}(\mathcal{Q}_\mu v)(s,a).
\end{equation}

\begin{lemma}\label{lem:contraction}
For each fixed $\mu\in\DeltaS$, the operator $\mathcal{T}_\mu$ is a contraction on
$(\R^{\mcs},\|\cdot\|_\infty)$ with modulus $\gamma$. Hence there is a unique
$v_\mu\in\R^{\mcs}$ such that
\begin{equation}\label{eq:Bellman-fp}
v_\mu=\mathcal{T}_\mu v_\mu.
\end{equation}
\end{lemma}

\begin{proof}
Fix $\mu\in\DeltaS$ and $v,w\in\R^{\mcs}$. Fix $s\in\mcs$ and $a\in\mca$.
For any $P\in\mathfrak{P}(s,a,\mu)$,
\begin{align*}
&\Big|\sum_{s'}P(s')\big(r(s,a,s',\mu)+\gamma v(s')\big)
-\sum_{s'}P(s')\big(r(s,a,s',\mu)+\gamma w(s')\big)\Big|\\
&\qquad=\gamma\Big|\sum_{s'}P(s')\big(v(s')-w(s')\big)\Big|
\le \gamma\|v-w\|_\infty,
\end{align*}
since $\sum_{s'}P(s')=1$ and $|v(s')-w(s')|\le\|v-w\|_\infty$ for all $s'$.
Taking the minimum over $P\in\mathfrak{P}(s,a,\mu)$ gives
$|(\mathcal{Q}_\mu v)(s,a)-(\mathcal{Q}_\mu w)(s,a)|\le\gamma\|v-w\|_\infty$.
Taking the maximum over $a\in\mca$ yields
$|(\mathcal{T}_\mu v)(s)-(\mathcal{T}_\mu w)(s)|\le\gamma\|v-w\|_\infty$ for all $s$.
Thus $\|\mathcal{T}_\mu v-\mathcal{T}_\mu w\|_\infty\le\gamma\|v-w\|_\infty$.
Since $\gamma\in(0,1)$, Banach’s fixed point theorem implies existence and uniqueness of
$v_\mu$ solving \eqref{eq:Bellman-fp}.
\end{proof}

Define the optimal robust $Q$-function
\begin{equation}\label{eq:Qstar}
Q_\mu(s,a):=(\mathcal{Q}_\mu v_\mu)(s,a),
\end{equation}
so $v_\mu(s)=\max_{a\in\mca}Q_\mu(s,a)$.

\begin{lemma}\label{lem:selectors}
Fix $\mu\in\DeltaS$.
\begin{enumerate}[label=(\roman*),leftmargin=2.2em]
\item For every $(s,a)\in\mcs\times\mca$ the set
\begin{align}
\widehat{\mathfrak{P}}(s,a,\mu)
:=\argmin_{P\in\mathfrak{P}(s,a,\mu)}\sum_{s'\in\mcs}P(s')\big(r(s,a,s',\mu)+\gamma v_\mu(s')\big)
\end{align}
is nonempty, convex, and compact.
\item For every $s\in\mcs$ the set of optimal actions
\begin{align}
\mathsf D(s,\mu):=\argmax_{a\in\mca}Q_\mu(s,a)
\end{align}
is nonempty.
\end{enumerate}
\end{lemma}

\begin{proof}
(i) Fix $(s,a)$. The objective is continuous and affine in $P$. By Assumption~\ref{ass:stat}(ii),
$\mathfrak{P}(s,a,\mu)$ is nonempty and compact, hence the minimum is attained and the argmin set
is nonempty and compact. Convexity follows because $\mathfrak{P}(s,a,\mu)$ is convex and the objective is affine.

(ii) Since $\mca$ is finite, the maximum of $a\mapsto Q_\mu(s,a)$ is attained, so $\mathsf D(s,\mu)\neq\emptyset$.
\end{proof}

\begin{proposition}\label{prop:robust-dp}
Fix $\mu\in\DeltaS$. Let $v_\mu$ be the unique fixed point of $\mathcal{T}_\mu$.
Let $\pi$ be any stationary policy satisfying $\supp\pi(\cdot|s)\subseteq \mathsf D(s,\mu)$ for all $s$, and let $p$ be any stationary kernel with $p(\cdot|s,a)\in\widehat{\mathfrak{P}}(s,a,\mu)$ for all $(s,a)$. Then:
\begin{enumerate}[label=(\roman*),leftmargin=2.2em]
\item For every $s\in\mcs$,
\begin{align}
v_\mu(s)=\inf_{p'\in\mathfrak{P}(s,a,\mu)}J_\mu(s;\pi,p')=J_\mu(s;\pi,p).
\end{align}
\item For every stationary policy $\pi'$,
\begin{align}
\inf_{p'}J_\mu(s;\pi',p')\le v_\mu(s)\qquad\forall s\in\mcs.
\end{align}
In particular, $v_\mu$ coincides with the robust value \eqref{eq:robust-value} and every $\pi$ supported on the greedy
sets attains the supremum in \eqref{eq:robust-value}.
\end{enumerate}
\end{proposition}

\begin{proof}
  
For any stationary pair $(\pi,p)$ with $p(\cdot\mid s,a)\in\mathfrak{P}(s,a,\mu)$, 
define for $v\in\R^{\mcs}$:
\begin{align}
(\mathcal{T}^{\pi,p}_\mu v)(s)
:=\sum_{a\in\mca}\pi(a\mid s)\sum_{s'\in\mcs}\big(r(s,a,s',\mu)+\gamma v(s')\big)p(s'\mid s,a).
\end{align}
Exactly as in Lemma~\ref{lem:contraction}, $\mathcal{T}^{\pi,p}_\mu$ is a contraction with modulus $\gamma$, and it is affine. 
Hence it has a unique fixed point, call it $u^{\pi,p}_\mu$.
Define partial sums
\begin{align}
J^{(n)}_\mu(s;\pi,p):=\E^{\pi,p}_{s_0=s}\Big[\sum_{t=0}^{n-1}\gamma^t r(s_t,a_t,s_{t+1},\mu)\Big].
\end{align}
Then $J^{(0)}_\mu(\cdot;\pi,p)\equiv 0$ and the Markov property yields the recursion
\begin{align}
J^{(n+1)}_\mu(\cdot;\pi,p)=\mathcal{T}^{\pi,p}_\mu J^{(n)}_\mu(\cdot;\pi,p).
\end{align}
Thus $J^{(n)}_\mu=(\mathcal{T}^{\pi,p}_\mu)^n 0\to u^{\pi,p}_\mu$ in $\|\cdot\|_\infty$.
Since $r$ is bounded and $\gamma\in(0,1)$, $J^{(n)}_\mu(s;\pi,p)\to J_\mu(s;\pi,p)$ for each $s$.
Therefore $u^{\pi,p}_\mu=J_\mu(\cdot;\pi,p)$.

Define
\begin{align}
(\mathcal{T}^{\pi}_\mu v)(s):=\sum_{a\in\mca}\pi(a\mid s)\min_{P\in\mathfrak{P}(s,a,\mu)}
\sum_{s'\in\mcs}\big(r(s,a,s',\mu)+\gamma v(s')\big)P(s').
\end{align}
Again, $\mathcal{T}^{\pi}_\mu$ is a contraction with modulus $\gamma$, hence has a unique fixed point $u^\pi_\mu$.

We show $u^\pi_\mu(s)=\inf_{p}J_\mu(s;\pi,p)$.
For any admissible $p$, by definition of the minimum, $\mathcal{T}^{\pi,p}_\mu v\ge \mathcal{T}^{\pi}_\mu v$ pointwise.
Iterating from $0$ gives $(\mathcal{T}^{\pi,p}_\mu)^n 0\ge (\mathcal{T}^{\pi}_\mu)^n 0$ for all $n$\footnote{Here we use the monotonicity of the operators, proved in Lemma \ref{lem:monotone}.}.
Taking limits and using Step~1,
\begin{align}
J_\mu(\cdot;\pi,p)=\lim_{n\to\infty}(\mathcal{T}^{\pi,p}_\mu)^n 0\ge
\lim_{n\to\infty}(\mathcal{T}^{\pi}_\mu)^n 0=u^\pi_\mu.
\end{align}
Hence $\inf_{p}J_\mu(s;\pi,p)\ge u^\pi_\mu(s)$.

Conversely, for each $(s,a)$ choose a minimizer
$p^{\pi,u^\pi_\mu}(\cdot\mid s,a)\in\argmin_{P\in\mathfrak{P}(s,a,\mu)}\sum_{s'}(r+\gamma u^\pi_\mu)P(s')$.
Then $\mathcal{T}^{\pi}_\mu u^\pi_\mu=\mathcal{T}^{\pi,p^{\pi,u^\pi_\mu}}_\mu u^\pi_\mu=u^\pi_\mu$.
By the result above, this implies $u^\pi_\mu=J_\mu(\cdot;\pi,p^{\pi,u^\pi_\mu})$ and hence $\inf_p J_\mu\le u^\pi_\mu$.
Thus $\inf_p J_\mu=u^\pi_\mu$.

For every $v$, we have that
\begin{align}
\mathcal{T}_\mu v=\max_{a\in\mca}(\mathcal{Q}_\mu v)(\cdot,a)=\sup_{\pi}\ \mathcal{T}^{\pi}_\mu v,
\end{align}
where the supremum is attained by deterministic $\pi$ choosing a maximizing action.
Hence $\mathcal{T}^{\pi}_\mu v\le \mathcal{T}_\mu v$ for all $\pi$, so by contraction iteration from $0$,
\begin{align}
u^\pi_\mu\le v_\mu,\quad\forall \pi.
\end{align}
Thus $\sup_\pi u^\pi_\mu(s)\le v_\mu(s)$.

Now choose $\pi$ supported on maximizers of $Q_\mu$: $\supp \pi(\cdot|s)\subset \mathsf D(s,\mu)$ for any $s$. Then for each $s$,
\begin{align}
    \mathcal{T}^{\pi}_\mu v_\mu(s)=\sum_a \pi(a|s)\mathcal{Q}_\mu v_\mu (s,a)=\sum_a \pi(a|s)Q_\mu(s,a)=\max_a Q_\mu(s,a)=v_\mu(s),
\end{align}
where the third equality holds because $\pi(\cdot|s)$ puts all mass on $\arg\max_a Q_\mu(s,a)$.
Thus $v_\mu$ is a fixed point of $\mathcal{T}^{\pi}_\mu$, hence by uniqueness $u^{\pi}_\mu=v_\mu$.
Therefore $\sup_\pi u^\pi_\mu(s)\ge v_\mu(s)$ and $\sup_\pi u^\pi_\mu(s)=v_\mu(s)$.

Since $p(\cdot\mid s,a)\in\widehat{\mathfrak{P}}(s,a,\mu)$ for each $(s,a)$.
Then by construction,
\begin{align}
(\mathcal{T}^{\pi,p}_\mu v_\mu)(s)=v_\mu(s)\quad\forall s,
\end{align}
so $v_\mu$ is the fixed point of $\mathcal{T}^{\pi,p}_\mu$ and it yields
$J_\mu(\cdot;\pi,p)=v_\mu$.
This proves (i). Part (ii) follows from $\sup_\pi\inf_p J_\mu=v_\mu$ proved above.
\end{proof}

\begin{lemma}[Monotonicity]\label{lem:monotone}
Fix $\mu\in\Delta(\mathcal{S})$, a stationary policy $\pi$, and a stationary kernel $p$ with $p(\cdot|s,a)\in\mathfrak{P}(s,a,\mu)$. Each of the operators $\mathcal{T}_\mu$, $\mathcal{T}^\pi_\mu$, $\mathcal{T}^{\pi,p}_\mu$ defined above is monotone: $v\le w$ componentwise implies $\mathcal{T} v\le \mathcal{T} w$. Moreover each commutes with constant shifts in the standard way: $\mathcal{T}(v+c\textbf{1})=\mathcal{T} v + \gamma c\,\textbf{1}$ for $c\in\R$.
\end{lemma}
 
\begin{proof}
For $\mathcal{T}^{\pi,p}_\mu$ both claims are immediate since the coefficients $\pi(a|s)p(s'|s,a)\ge 0$ sum (over $s'$, for each $a$) to one. For the inner robust term, fix $(s,a)$ and let $f_P(v):=\sum_{s'}P(s')\bigl[r(s,a,s',\mu)+\gamma v(s')\bigr]$. If $v\le w$ then $f_P(v)\le f_P(w)$ for every $P$. Let $P_w$ attain $\min_P f_P(w)$. Then
$\min_P f_P(v)\le f_{P_w}(v)\le f_{P_w}(w)=\min_P f_P(w)$,
so $v\mapsto\min_{P\in\mathfrak{P}(s,a,\mu)}f_P(v)$ is monotone. Averaging over $\pi(\cdot|s)$ (for $\mathcal{T}^\pi_\mu$) or maximizing over $a$ (for $\mathcal{T}_\mu$) preserves monotonicity. The shift property follows since $\sum_{s'}P(s')\gamma c = \gamma c$ for every $P\in\Delta(\mathcal{S})$, and shifts commute with $\min$, averages, and $\max$.
\end{proof}

\begin{lemma}\label{lem:reduction}
Fix $\mu\in\DeltaS$ and let $(\pi^\mu,p^\mu)$ be as in Proposition~\ref{prop:robust-dp}.
Define for $v\in\R^{\mcs}$ the (non-robust) Bellman operator with \emph{fixed} kernel $p^\mu$:
\begin{align}
(\widetilde{\mathcal{T}}_\mu v)(s):=\max_{a\in\mca}\sum_{s'\in\mcs}p^\mu(s'\mid s,a)\big(r(s,a,s',\mu)+\gamma v(s')\big).
\end{align}
Then $\widetilde{\mathcal{T}}_\mu$ is a contraction with modulus $\gamma$ and its unique fixed point equals $v_\mu$.
Consequently, for every initial distribution $\lambda\in\DeltaS$,
\begin{equation}\label{eq:reduction-value}
\sum_{s\in\mcs}\lambda(s)v_\mu(s)
=
\sup_{\pi:\mcs\to\DeltaA}\ \E^{\pi,p^\mu}_{s_0\sim\lambda}\Big[\sum_{t=0}^\infty\gamma^t r(s_t,a_t,s_{t+1},\mu)\Big],
\end{equation}
and the supremum is attained at $\pi=\pi^\mu$.
\end{lemma}

\begin{proof}
\textbf{Step 1: $\widetilde{\mathcal{T}}_\mu$ is a contraction.}
Fix $v,w\in\R^{\mcs}$. For any $s$ and $a$,
\begin{align}
\Big|\sum_{s'}p^\mu(s'\mid s,a)\gamma(v(s')-w(s'))\Big|
\le \gamma\|v-w\|_\infty.
\end{align}
Taking the maximum over $a$ gives
$|(\widetilde{\mathcal{T}}_\mu v)(s)-(\widetilde{\mathcal{T}}_\mu w)(s)|\le\gamma\|v-w\|_\infty$.
Thus $\|\widetilde{\mathcal{T}}_\mu v-\widetilde{\mathcal{T}}_\mu w\|_\infty\le\gamma\|v-w\|_\infty$.

\textbf{Step 2: $v_\mu$ is a fixed point of $\widetilde{\mathcal{T}}_\mu$.}
By definition of $p^\mu(\cdot\mid s,a)\in\widehat{\mathfrak{P}}(s,a,\mu)$,
for every $(s,a)$,
\begin{align}
(\mathcal{Q}_\mu v_\mu)(s,a)
=\sum_{s'}p^\mu(s'\mid s,a)\big(r(s,a,s',\mu)+\gamma v_\mu(s')\big).
\end{align}
Therefore,
\begin{align}
(\widetilde{\mathcal{T}}_\mu v_\mu)(s)
=\max_{a}\sum_{s'}p^\mu(s'\mid s,a)\big(r+\gamma v_\mu\big)
=\max_{a}(\mathcal{Q}_\mu v_\mu)(s,a)
=(\mathcal{T}_\mu v_\mu)(s)=v_\mu(s).
\end{align}
So $v_\mu$ is a fixed point of $\widetilde{\mathcal{T}}_\mu$.
By Step~1 and Banach’s theorem, the fixed point is unique, hence it equals $v_\mu$.

\textbf{Step 3: Representation as an MDP value.}
Fix any stationary policy $\pi$. Define the policy-evaluation operator
\begin{align}
(\widetilde{\mathcal{T}}_\mu^\pi v)(s):=\sum_{a}\pi(a\mid s)\sum_{s'}p^\mu(s'\mid s,a)\big(r(s,a,s',\mu)+\gamma v(s')\big).
\end{align}
It is a contraction with modulus $\gamma$ and thus has a unique fixed point $v^{\pi,p^\mu}$.
The standard contraction iteration argument (identical to Step~1 of Proposition~\ref{prop:robust-dp})
shows that for each $s$,
\begin{align}
v^{\pi,p^\mu}(s)=\E^{\pi,p^\mu}_{s_0=s}\Big[\sum_{t=0}^\infty\gamma^t r(s_t,a_t,s_{t+1},\mu)\Big].
\end{align}
By definition of $\widetilde{\mathcal{T}}_\mu$ as the pointwise maximum over actions,
$v_\mu$ is the optimal value for the MDP with kernel $p^\mu$, hence
$v^{\pi,p^\mu}(s)\le v_\mu(s)$ for all $s$ and equality holds for $\pi=\pi^\mu$.
Averaging over $s_0\sim\lambda$ gives \eqref{eq:reduction-value}.
\end{proof}

Note that under rectangularity, we will have that $V_\mu(\lambda)=\sum_s \lambda(s)v_\mu(s)$, as we proved in the following result.

\begin{proposition}
\label{prop:statewise-vs-distributional}
Fix $\mu\in\DeltaS$ and consider the rectangular admissible kernel class
\begin{align}
\mathcal P(\mu):=\Big\{p:\mcs\times\mca\to\DeltaS:\ p(\cdot\mid s,a)\in\mathfrak{P}(s,a,\mu)\ \forall(s,a)\Big\}.
\end{align}
For $\lambda\in\DeltaS$ define the distributional robust value
\begin{align}
V_\mu(\lambda):=\sup_{\pi}\inf_{p\in\mathcal P(\mu)}\E^{\pi,p}_{s_0\sim\lambda}
\Big[\sum_{t=0}^\infty\gamma^t r(s_t,a_t,s_{t+1},\mu)\Big].
\end{align}
Let $v_\mu:\mcs\to\R$ be the (unique) fixed point of the robust Bellman operator $T_\mu$,
equivalently $v_\mu(s)=\sup_{\pi}\inf_{p\in\mathcal P(\mu)}J_\mu(s;\pi,p)$ for each $s$.
Then for every $\lambda\in\DeltaS$,
\begin{equation}
\label{eq:equiv-V-and-v}
V_\mu(\lambda)=\sum_{s\in\mcs}\lambda(s)\,v_\mu(s)=:\langle \lambda,v_\mu\rangle.
\end{equation}
In particular, taking $\lambda=\mu$ gives $V_\mu(\mu)=\langle \mu,v_\mu\rangle$.
\end{proposition}

\begin{proof}
Fix $\lambda\in\DeltaS$.

For any stationary policy $\pi$ and admissible stationary kernel $p\in\mathcal P(\mu)$,
by conditioning on $s_0$ we have
\begin{align}
J_\mu(\lambda;\pi,p)
=\E^{\pi,p}_{s_0\sim\lambda}\Big[\sum_{t=0}^\infty\gamma^t r(s_t,a_t,s_{t+1},\mu)\Big]
=\sum_{s\in\mcs}\lambda(s)\,J_\mu(s;\pi,p).
\end{align}
 
By Proposition~\ref{prop:robust-dp} (robust discounted DP under $(s,a)$-rectangularity),
there exist a stationary policy $\pi^\mu$ and an admissible stationary kernel $p^\mu\in\mathcal P(\mu)$ such that
for every $s\in\mcs$,
\begin{equation}
\label{eq:saddle-all-states}
v_\mu(s)=\inf_{p\in\mathcal P(\mu)}J_\mu(s;\pi^\mu,p)=J_\mu(s;\pi^\mu,p^\mu).
\end{equation}
Multiplying \eqref{eq:saddle-all-states} by $\lambda(s)$ and summing over $s$ yields
\begin{equation}
\label{eq:lowerbound}
\inf_{p\in\mathcal P(\mu)}J_\mu(\lambda;\pi^\mu,p)
=J_\mu(\lambda;\pi^\mu,p^\mu)
=\sum_{s}\lambda(s)v_\mu(s)=\langle \lambda,v_\mu\rangle.
\end{equation}
Therefore, by definition of $V_\mu(\lambda)$,
\begin{align}
V_\mu(\lambda)\ge \inf_{p\in\mathcal P(\mu)}J_\mu(\lambda;\pi^\mu,p)=\langle \lambda,v_\mu\rangle.
\end{align}

We then claim that for every stationary policy $\pi$ and every state $s$,
\begin{equation}
\label{eq:MDP-upper}
J_\mu(s;\pi,p^\mu)\le v_\mu(s).
\end{equation}
Indeed, define the standard (non-robust) Bellman operator for the fixed kernel $p^\mu$:
\begin{align}
(\widetilde T_\mu v)(s):=\max_{a\in\mca}\sum_{s'\in\mcs}p^\mu(s'\mid s,a)\big(r(s,a,s',\mu)+\gamma v(s')\big).
\end{align}
Because $p^\mu(\cdot\mid s,a)$ attains the inner minimum in the robust Bellman operator at $v_\mu$,
the robust fixed point identity $v_\mu=T_\mu v_\mu$ implies $v_\mu=\widetilde T_\mu v_\mu$.
Since $\widetilde T_\mu$ is a contraction with modulus $\gamma$, $v_\mu$ is its unique fixed point, hence it is the
optimal value function for the MDP with transition kernel $p^\mu$. Therefore \eqref{eq:MDP-upper} holds.

Using \eqref{eq:MDP-upper} and Step 1, for any $\pi$ we have
\begin{align}
\inf_{p\in\mathcal P(\mu)}J_\mu(\lambda;\pi,p)\le J_\mu(\lambda;\pi,p^\mu)
=\sum_s \lambda(s)J_\mu(s;\pi,p^\mu)
\le \sum_s \lambda(s)v_\mu(s)=\langle \lambda,v_\mu\rangle.
\end{align}
Taking $\sup_\pi$ of the left-hand side yields
$V_\mu(\lambda)\le \langle \lambda,v_\mu\rangle$.

Thus it holds that   $V_\mu(\lambda)\ge \langle \lambda,v_\mu\rangle$ and 
$V_\mu(\lambda)\le \langle \lambda,v_\mu\rangle$, hence equality \eqref{eq:equiv-V-and-v} holds.
\end{proof}

\begin{proposition}\label{prop:history}
Fix $\mu\in\Delta(\mathcal{S})$ and a stationary policy $\pi$. We name $(q_t)_{t\ge0}$  an {admissible history-dependent adversary} if each $q_t$ maps histories $h_t=(s_0,a_0,\dots,s_{t-1},a_{t-1},s_t)$ and actions $a_t$ to $q_t(\cdot\,|\,h_t,a_t)\in\mathfrak{P}(s_t,a_t,\mu)$. Let $J_\mu(s;\pi,(q_t))$ be the induced discounted reward from $s_0=s$. Then
\[
\inf_{(q_t)\ \mathrm{history\text{-}dep.}} J_\mu(s;\pi,(q_t))
=\inf_{p\ \mathrm{stationary}} J_\mu(s;\pi,p)
= u^\pi_\mu(s)\qquad\forall s\in\mathcal{S} .
\]
\end{proposition}
 
\begin{proof}
``$\le$'': every stationary kernel induces an admissible history-dependent adversary, so the left infimum is no larger; the right identity directly comes from the definition.
 
``$\ge$'': fix any admissible $(q_t)$ and write $u:=u^\pi_\mu$. Conditioning on the history $h_t$ (i.e., $s_t$ is known, $a_t\sim\pi(\cdot|s_t)$, $s_{t+1}\sim q_t(\cdot|h_t,a_t)$),
\[
\E\bigl[r(s_t,a_t,s_{t+1},\mu)+\gamma u(s_{t+1})\,\big|\,h_t\bigr]
=\sum_a\pi(a|s_t)\sum_{s'}q_t(s'|h_t,a)\bigl[r(s_t,a,s',\mu)+\gamma u(s')\bigr]
\ge (\mathcal{T}^\pi_\mu u)(s_t)=u(s_t),
\]
since $q_t(\cdot|h_t,a)\in\mathfrak{P}(s_t,a,\mu)$ and $\mathcal{T}^\pi_\mu$ takes the minimum over that set. Multiplying by $\gamma^t$, taking expectations, and summing for $t=0,\dots,T-1$ telescopes to
$\E\sum_{t=0}^{T-1}\gamma^t r(s_t,a_t,s_{t+1},\mu)\ \ge\ u(s)-\gamma^T\,\E\,u(s_T)$.
Since $\|u\|_\infty\le\|r\|_\infty/(1-\gamma)$, letting $T\to\infty$ then gives $J_\mu(s;\pi,(q_t))\ge u(s)$.
\end{proof}

\subsection{Existence of a stationary robust mean-field equilibrium}
In this section we study existence of robust MFE.

Let
\begin{align}
\Pi:=\prod_{s\in S}\DD(A),
\qquad
\mathcal K:=\prod_{(s,a)\in S\times A}\DD(S).
\end{align}
Thus an element $\pi\in\Pi$ is a stationary policy and an element $p\in\mathcal K$ is a stationary kernel slice $p(\cdot|s,a)$.

For fixed $\mu\in\DD(S)$, let $v_\mu$ be the unique fixed point of $T_\mu$, and define
\begin{align}
Q_\mu(s,a):=
\min_{P\in P(s,a,\mu)}
\sum_{s'\in S} P(s')\bigl(r(s,a,s',\mu)+\gamma v_\mu(s')\bigr).
\end{align}
Set
\begin{align}
D(s,\mu):=\argmax_{a\in A}Q_\mu(s,a),
\end{align}
and
\begin{align}
\widehat P(s,a,\mu):=
\argmin_{P\in P(s,a,\mu)}
\sum_{s'\in S} P(s')\bigl(r(s,a,s',\mu)+\gamma v_\mu(s')\bigr).
\end{align}

For $(\pi,p)\in \Pi\times \mathcal K$, define the induced Markov kernel on $\mathcal{S}$ as
\begin{align}
K_{\pi,p}(s'|s):=\sum_{a\in A}\pi(a|s)\,p(s'|s,a).
\end{align}

Now define three correspondences:
\begin{align}
\BR(\mu)
&:=
\Bigl\{\pi\in\Pi:\supp \pi(\cdot|s)\subset D(s,\mu)\ \forall s\in S\Bigr\},\\
\WC(\mu)
&:=
\Bigl\{p\in\mathcal K: p(\cdot|s,a)\in \widehat P(s,a,\mu)\ \forall (s,a)\in S\times A\Bigr\},\\
\Inv(\pi,p)
&:=
\Bigl\{\eta\in\DD(S): \eta K_{\pi,p}=\eta\Bigr\}.
\end{align}
Finally set
\begin{align}
\Phi(\mu,\pi,p):=\Inv(\pi,p)\times \BR(\mu)\times \WC(\mu).
\end{align}

\subsubsection{Preparatory lemmas}

\begin{lemma}[Continuity of the robust Bellman data]
\label{lem:continuity-bellman}
Under Assumption \ref{ass:stat}:
\begin{enumerate}[label=(\roman*)]
    \item for every fixed $v\in\mathbb R^S$, the map $(\mu,s,a)\mapsto (Q_\mu v)(s,a)$ is continuous;
    \item the map $\mu\mapsto v_\mu$ is continuous in $\|\cdot\|_\infty$;
    \item the map $\mu\mapsto Q_\mu(s,a)$ is continuous for every $(s,a)\in S\times A$.
\end{enumerate}
\end{lemma}

\begin{proof}
For fixed $v$, define
\begin{align}
H_{s,a}(\mu,P;v):=
\sum_{s'\in S}P(s')
\bigl(r(s,a,s',\mu)+\gamma v(s')\bigr).
\end{align}
By boundedness and continuity of $r$ in $\mu$, $H_{s,a}$ is continuous in $(\mu,P)$. Since $P(s,a,\mu)$ is compact-valued and continuous as a correspondence, Berge's maximum theorem implies that
\begin{align}
(\mu,s,a)\longmapsto (Q_\mu v)(s,a)
=
\min_{P\in P(s,a,\mu)}H_{s,a}(\mu,P;v)
\end{align}
is continuous.

Now let $\mu_n\to\mu$. Since $v_{\mu_n}=T_{\mu_n}v_{\mu_n}$ and $v_\mu=T_\mu v_\mu$,
\begin{align}
\|v_{\mu_n}-v_\mu\|_\infty
\le
\|T_{\mu_n}v_{\mu_n}-T_{\mu_n}v_\mu\|_\infty
+
\|T_{\mu_n}v_\mu-T_\mu v_\mu\|_\infty
\le
\gamma \|v_{\mu_n}-v_\mu\|_\infty
+
\|T_{\mu_n}v_\mu-T_\mu v_\mu\|_\infty.
\end{align}
Hence
\begin{align}
(1-\gamma)\|v_{\mu_n}-v_\mu\|_\infty
\le
\|T_{\mu_n}v_\mu-T_\mu v_\mu\|_\infty\to 0,
\end{align}
because the first part of the lemma gives continuity of $T_\mu v_\mu$ in $\mu$. Therefore $v_{\mu_n}\to v_\mu$.

Finally,
\begin{align}
Q_\mu(s,a)=
\min_{P\in P(s,a,\mu)}
\sum_{s'}P(s')
\bigl(r(s,a,s',\mu)+\gamma v_\mu(s')\bigr),
\end{align}
so the continuity of $\mu\mapsto v_\mu$ and Berge's theorem again imply the continuity of $\mu\mapsto Q_\mu(s,a)$.
\end{proof}

\begin{lemma}[Properties of the fixed-point correspondence]
\label{lem:phi-properties}
Under Assumption \ref{ass:stat}:
\begin{enumerate}[label=(\roman*)]
    \item $\BR(\mu)$, $\WC(\mu)$, and $\Inv(\pi,p)$ are all nonempty, compact, and convex;
    \item the graphs of $\BR$, $\WC$, and $\Inv$ are closed;
    \item therefore $\Phi$ has nonempty compact convex values and closed graph.
\end{enumerate}
\end{lemma}

\begin{proof}
\emph{Step 1: $\BR(\mu)$.}
For each state $\mathcal{S}$, $D(s,\mu)\neq\varnothing$ because $\mathcal{A}$ is finite. The set of distributions on $\mathcal{A}$ supported on $D(s,\mu)$ is a simplex, hence nonempty, compact, and convex. Taking the finite product over $\mathcal{S}$ gives the same properties for $\BR(\mu)$.

To prove closed graph, suppose $\mu_n\to\mu$, $\pi_n\to\pi$, and $\pi_n\in\BR(\mu_n)$. Fix $s\in S$ and $a\in A$ with $\pi(a|s)>0$. For $n$ large enough one has $\pi_n(a|s)>0$, hence $a\in D(s,\mu_n)$. By Lemma~\ref{lem:continuity-bellman}, $Q_{\mu_n}(s,\cdot)\to Q_\mu(s,\cdot)$ uniformly over the finite set $\mathcal{A}$. Therefore
\begin{align}
Q_\mu(s,a)=\lim_{n\to\infty}Q_{\mu_n}(s,a)
=
\lim_{n\to\infty}\max_{b\in A}Q_{\mu_n}(s,b)
=
\max_{b\in A}Q_\mu(s,b),
\end{align}
so $a\in D(s,\mu)$. Thus $\pi\in\BR(\mu)$.

\emph{Step 2: $\WC(\mu)$.}
For each $(s,a)$, the objective defining $\widehat P(s,a,\mu)$ is continuous and affine in $P$, while $P(s,a,\mu)$ is nonempty, compact, and convex. Hence $\widehat P(s,a,\mu)$ is nonempty, compact, and convex. Taking the finite product over $(s,a)$ gives the same properties for $\WC(\mu)$.

To prove closed graph, suppose $\mu_n\to\mu$, $p_n\to p$, and $p_n\in \WC(\mu_n)$. Fix $(s,a)$. By the closed-graph property of the ambiguity correspondence, $p(\cdot|s,a)\in P(s,a,\mu)$. Let $q\in P(s,a,\mu)$ be arbitrary. By lower hemicontinuity of $P(s,a,\cdot)$, there are $q_n\in P(s,a,\mu_n)$ such that $q_n\to q$. Since $p_n(\cdot|s,a)$ minimizes the robust Bellman objective at $\mu_n$,
\begin{align}
\sum_{s'} p_n(s'|s,a)\bigl(r(s,a,s',\mu_n)+\gamma v_{\mu_n}(s')\bigr)
\le
\sum_{s'} q_n(s')\bigl(r(s,a,s',\mu_n)+\gamma v_{\mu_n}(s')\bigr).
\end{align}
Passing to the limit and using Lemma~\ref{lem:continuity-bellman} gives
\begin{align}
\sum_{s'} p(s'|s,a)\bigl(r(s,a,s',\mu)+\gamma v_\mu(s')\bigr)
\le
\sum_{s'} q(s')\bigl(r(s,a,s',\mu)+\gamma v_\mu(s')\bigr).
\end{align}
Since $q$ was arbitrary in $P(s,a,\mu)$, $p(\cdot|s,a)\in \widehat P(s,a,\mu)$. Hence $p\in\WC(\mu)$.

\emph{Step 3: $\Inv(\pi,p)$.}
For fixed $(\pi,p)$, the map
\begin{align}
\eta\longmapsto \eta K_{\pi,p}
\end{align}
is continuous from $\DD(S)$ into itself. Because $\DD(S)$ is compact and convex, Brouwer's fixed point theorem gives at least one invariant distribution. The set $\Inv(\pi,p)$ is the solution set of the linear equation $\eta K_{\pi,p}=\eta$ inside $\DD(S)$, hence it is compact and convex.

For closed graph, let $(\pi_n,p_n,\eta_n)\to(\pi,p,\eta)$ with $\eta_n\in\Inv(\pi_n,p_n)$. Since
\begin{align}
\eta_n(s')=
\sum_{s\in S}\eta_n(s)\sum_{a\in A}\pi_n(a|s)p_n(s'|s,a)
\qquad\forall s'\in S,
\end{align}
passing to the limit in these finite sums yields $\eta K_{\pi,p}=\eta$. Hence $\eta\in\Inv(\pi,p)$.

\emph{Step 4: product correspondence.}
The product of correspondences with nonempty compact convex values and closed graph has the same properties, so $\Phi$ does as well.
\end{proof}

\subsubsection{Existence}

\begin{theorem}
\label{thm:corrected-existence}
Under Assumption \ref{ass:stat}, there exists a stationary robust mean-field equilibrium.
\end{theorem}

\begin{proof}
The space
\begin{align}
X:=\DD(S)\times \Pi \times \mathcal K
\end{align}
is nonempty, compact, and convex in a finite-dimensional Euclidean space. By Lemma~\ref{lem:phi-properties}, the correspondence $\Phi:X\rightrightarrows X$ has nonempty compact convex values and closed graph. Hence Kakutani's fixed point theorem gives a triple $(\mu^\star,\pi^\star,p^\star)\in X$ such that
\begin{align}
(\mu^\star,\pi^\star,p^\star)\in \Phi(\mu^\star,\pi^\star,p^\star).
\end{align}

From the definition of $\Phi$, the fixed-point relation means:
\begin{enumerate}[label=(\roman*)]
    \item $\mu^\star\in\Inv(\pi^\star,p^\star)$, i.e.
    \begin{align}
    \mu^\star(s')
    =
    \sum_{s\in S}\mu^\star(s)\sum_{a\in A}\pi^\star(a|s)p^\star(s'|s,a)
    \qquad\forall s'\in S;
    \end{align}
    \item $\pi^\star\in\BR(\mu^\star)$, so $\supp \pi^\star(\cdot|s)\subset D(s,\mu^\star)$ for every state $\mathcal{S}$;
    \item $p^\star\in\WC(\mu^\star)$, so $p^\star(\cdot|s,a)\in \widehat P(s,a,\mu^\star)$ for every $(s,a)$.
\end{enumerate}

Now apply Proposition \ref{prop:robust-dp} at the population $\mu^\star$. Because $\pi^\star$ is pointwise optimal at every state and $p^\star$ is a pointwise minimizing kernel at every state-action pair, Proposition \ref{prop:robust-dp} yields
\begin{align}
V_{\mu^\star}
=
\inf_p J_{\mu^\star}(\pi^\star,p)
=
J_{\mu^\star}(\pi^\star,p^\star).
\end{align}
Together with item (i), this is exactly a stationary robust mean-field equilibrium.
\end{proof}

\section{Approximation of the $N$-player robust games}\label{app:approx}

\subsection{$N$-player robust games and DRMGs}\label{app:n-player and drmg}
We first note that, our $N$-player robust game is a robust Markov game on the enlarged state space $\mcs^N$, with 
\begin{itemize}
    \item joint state $s^N$,
    \item joint action $a^N$,
    \item Player $i$'s reward $r_i(s^N,a^N,(s')^N)=r\bigl(s^i,a^i,(s')^i,e^N(s^N)\bigr)$,
\end{itemize}
with a $(s,a)$-rectangular uncertainty set $\mathfrak{P}^N(s^N,a^N)$. Thus, our $N$-player robust game is a special robust Markov game on the joint state space $\mcs^N$ with mean-field coupling. 

Notably, the objective function (\eqref{eq:JiN}) in $N$-player robust games is defined as 
\begin{align}\label{eq:N-player-payoff}
J_i^N(\pi^N)= \inf_{(p_t^N)}\;
\E^{\pi^N,(p_t^N)}\Big[\sum_{t\ge0}\gamma^t\, r(s_t^i,a_t^i,s_{t+1}^i, e^N(s_t^N))\Big], 
\end{align}
where each agent concerns its own worst-case performance under the uncertainty set. Although in \eqref{eq:JiN}, the worst-case kernel is taken w.r.t. non-stationary kernel sequence $(p_t)$, we will show the worst-case is achieved at a stationary kernel. Thus, \eqref{eq:JiN} matches the robust value function of the standard robust Markov game.

\subsection{Counterexample}
\label{app:counterexample}
\begin{theorem}\label{thm:cex}
Fix $\gamma\in(0,1)$ and $\kappa>1/\gamma$, and set
\[
\varepsilon_0\;:=\;\frac{\gamma\kappa-1}{1-\gamma}\;>\;0.
\]
There exist a robust mean-field game satisfying Assumption~\ref{ass:stat} and a stationary robust mean-field equilibrium $(\mu^\star,\pi^\star,p^\star)$ of it such that:
\begin{enumerate}[label=(\alph*)]
  \item Assumption~\ref{ass:one-step-law} fails for the constant deviation sequence $\pi^{(N)}\equiv\pi^\star$ under \emph{every} choice of minimizing laws $P^{N|(N)}$; and
  \item for \emph{every} $N\in\mathbb{N}$, the symmetric profile $\pi^{N|\star}$ is not an $\varepsilon$-Nash equilibrium of the $N$-player robust game for any $\varepsilon<\varepsilon_0$.
\end{enumerate}
\end{theorem}
 
\begin{proof}
\textbf{The model.} Let $\mathcal{S}=\{0,1\}$, $\mathcal{A}=\{b,o\}$, $\mathfrak{P}(s,a,\mu)=\Delta(\mathcal{S})$ for all $(s,a,\mu)$ (full ambiguity), and
\[
r(s,a,s',\mu)\;=\;\begin{cases}1-\kappa\,\mu(1), & a=b,\\ 0, & a=o.\end{cases}
\]
Assumption~\ref{ass:stat} holds: rewards are bounded and affine (hence continuous) in $\mu$; the constant correspondence $\mu\mapsto\Delta(\mathcal{S})$ is nonempty, convex, compact valued, with closed graph, and lower hemicontinuous.
 
\textbf{Step 1: $(\mu^\star,\pi^\star,p^\star)=(\delta_0,\;\pi^\star(b|\cdot)\equiv1,\;p^\star(\cdot|s,a)\equiv\delta_0)$ is a stationary robust MFE.}
At the population $\mu^\star=\delta_0$ we have $\mu^\star(1)=0$, so $r(s,b,s',\mu^\star)=1$ and $r(s,o,s',\mu^\star)=0$ for all $s,s'$. Since the reward does not depend on the successor state, for every stationary policy $\pi$ and every admissible stationary kernel $p$,
\[
J_{\mu^\star}(\pi,p)\;=\;\sum_{t\ge0}\gamma^t\,P^{\pi,p}(a_t=b)\;\le\;\frac{1}{1-\gamma},
\]
with equality when $\pi=\pi^\star$, for \emph{every} $p$. Hence
\[
V_{\mu^\star}\;=\;\sup_\pi\inf_{p}J_{\mu^\star}(\pi,p)\;=\;\frac{1}{1-\gamma}\;=\;\inf_p J_{\mu^\star}(\pi^\star,p)\;=\;J_{\mu^\star}(\pi^\star,p^\star),
\]
which verifies Definition~\ref{def:smfe}(i)--(ii). Consistency (Definition~\ref{def:smfe}(iii)) holds because $K_{\pi^\star,p^\star}(\cdot|s)=\delta_0$ for every $s$, so $\delta_0 K_{\pi^\star,p^\star}=\delta_0=\mu^\star$.
 
\textbf{Step 2: Assumption~\ref{ass:one-step-law} fails under every choice of minimizing laws.}
Take the constant deviation sequence $\pi^{(N)}\equiv\pi^\star$, so $\pi^{N|(N)}=\pi^{N|\star}$ and every player always plays $b$. Fix $N$ and let $(p^N_t)_{t\ge0}$ be \emph{any} admissible kernel sequence for the $N$-player robust game. Since $s_0^i\sim\mu^\star=\delta_0$ i.i.d., we have $e_N(s^N_0)=\delta_0$ almost surely, and the realized reward of player~$1$ at time $t$ is $1-\kappa\, e_N(s^N_t)(1)$. Therefore the payoff under $(p^N_t)$ equals
\begin{equation}\label{eq:cex-payoff}
\sum_{t\ge0}\gamma^t\Big(1-\kappa\,\E\big[e_N(s^N_t)(1)\big]\Big)
\;=\;\frac{1}{1-\gamma}\;-\;\kappa\sum_{t\ge1}\gamma^t\,\E\big[e_N(s^N_t)(1)\big].
\end{equation}
Because $e_N(\cdot)(1)\in[0,1]$ pointwise, \eqref{eq:cex-payoff} is bounded below by
\[
\frac{1}{1-\gamma}-\kappa\sum_{t\ge1}\gamma^t\;=\;\frac{1-\gamma\kappa}{1-\gamma},
\]
\emph{with equality if and only if} $\E[e_N(s^N_t)(1)]=1$ for every $t\ge1$, i.e., if and only if $s^i_t=1$ almost surely for all players $i$ and all $t\ge1$. The lower bound is attained by the admissible constant product kernel $p^N(\cdot\,|\,s^N,a^N):=\delta_1\otimes\cdots\otimes\delta_1\in\mathfrak{P}^N(s^N,a^N)$, which sends every player to state $1$ at time $1$ and keeps them there. Hence
\begin{equation}\label{eq:cex-value}
J^N_1(\pi^{N|\star})\;=\;\frac{1-\gamma\kappa}{1-\gamma},
\end{equation}
and, crucially, \emph{every} law $P^{N|(N)}$ attaining the infimum satisfies $s^1_1=1$ almost surely. Consequently, under \emph{every} minimizing law,
\[
Q^{N|(N)}_0\;=\;\delta_{(0,\,b,\,1,\,\delta_0)},
\qquad\text{whereas}\qquad
Q^{\star(N)}_0\;=\;\delta_{(0,\,b,\,0,\,\delta_0)}
\]
(the proxy chain has $s_0=0$, $a_0=b$, $s_1\sim p^\star(\cdot|0,b)=\delta_0$, population coordinate $\mu^\star=\delta_0$). These are Dirac masses at points $z\neq z'$ of $Z$ with $d_Z(z,z')=\textbf{1}\{(0,b,1)\neq(0,b,0)\}+\norm{\delta_0-\delta_0}_1=1$, so
\[
\Wone\big(Q^{N|(N)}_0,\,Q^{\star(N)}_0\big)\;=\;1\qquad\text{for every }N\text{ and every choice of minimizing laws.}
\]
Thus (16) fails at $t=0$ for this deviation sequence regardless of how the minimizing laws are chosen, i.e., Assumption~\ref{ass:one-step-law} does not hold. This proves (a).
 
\textbf{Step 3: uniform Nash gap.}
Consider the unilateral deviation $\pi(o|s)\equiv1$ by player~$1$. Its reward is identically $0$ along every trajectory, regardless of the kernel sequence, so $J^N_1(\pi^{N|\star,-1}\oplus\pi)=0$ for every $N$. Combining with \eqref{eq:cex-value} and $\kappa>1/\gamma$,
\[
\sup_{\pi'}J^N_1(\pi^{N|\star,-1}\oplus\pi')\;-\;J^N_1(\pi^{N|\star})
\;\ge\;0-\frac{1-\gamma\kappa}{1-\gamma}
\;=\;\frac{\gamma\kappa-1}{1-\gamma}\;=\;\varepsilon_0\;>\;0
\]
for every $N$. Hence $\pi^{N|\star}$ is not an $\varepsilon$-Nash equilibrium for any $\varepsilon<\varepsilon_0$, proving (b).
\end{proof}

\subsection{Auxiliary facts}
\label{app:auxiliary-C}





\begin{lemma}[Finite-player robust dynamic programming]
\label{lem:finite-player-robust-dp}
Fix $N\in\mathbb N$, a profile $\pi^N\in \Pi^N$, and an agent
$i\in\{1,\dots,N\}$. Define the one-step reward on
$S^N\times A^N\times S^N$ by
\[
\rho_i(s^N,a^N,s'{}^N)
:= r\!\bigl(s^i,a^i,(s')^i,e^N(s^N)\bigr).
\]
Then the following hold.

\begin{enumerate}
\item There exists a unique function $u^{N,i}:S^N\to\mathbb R$
satisfying
\[
u^{N,i}(s^N)
=
\sum_{a^N\in A^N}\pi^N(a^N\mid s^N)
\min_{p\in P^N(s^N,a^N)}
\sum_{s'{}^N\in S^N}
p(s'{}^N)\Bigl(\rho_i(s^N,a^N,s'{}^N)+\gamma u^{N,i}(s'{}^N)\Bigr).
\]

\item The robust payoff in \eqref{eq:N-player-payoff} satisfies
\[
J_i^N(\pi^N)
=
\mathbb E_{s_0^N\sim (\mu^\star)^{\otimes N}}
\bigl[u^{N,i}(s_0^N)\bigr].
\]

\item The infimum in \eqref{eq:N-player-payoff} is attained by a
stationary kernel
\[
p^{N,i,\star}:S^N\times A^N\to \Delta(S^N)
\]
with
\[
p^{N,i,\star}(\cdot\mid s^N,a^N)\in P^N(s^N,a^N)
\qquad\forall (s^N,a^N)\in S^N\times A^N.
\]
More precisely, one may choose $p^{N,i,\star}$ so that
\[
\sum_{s'{}^N}
p^{N,i,\star}(s'{}^N\mid s^N,a^N)
\Bigl(\rho_i(s^N,a^N,s'{}^N)+\gamma u^{N,i}(s'{}^N)\Bigr)
=
\min_{p\in P^N(s^N,a^N)}
\sum_{s'{}^N}
p(s'{}^N)\Bigl(\rho_i(s^N,a^N,s'{}^N)+\gamma u^{N,i}(s'{}^N)\Bigr),
\]
and, if we use the constant sequence $p_t^N\equiv p^{N,i,\star}$, then
\[
J_i^N(\pi^N)
=
\mathbb E^{\pi^N,p^{N,i,\star}}
\!\left[
\sum_{t=0}^\infty \gamma^t
\rho_i(s_t^N,a_t^N,s_{t+1}^N)
\right].
\]
\end{enumerate}
\end{lemma}

\begin{proof}
For $u\in \mathbb R^{S^N}$, define the Bellman operator
$T^{N,i}:\mathbb R^{S^N}\to \mathbb R^{S^N}$ by
\[
(T^{N,i}u)(s^N)
:=
\sum_{a^N\in A^N}\pi^N(a^N\mid s^N)
\min_{p\in P^N(s^N,a^N)}
\sum_{s'{}^N\in S^N}
p(s'{}^N)\Bigl(\rho_i(s^N,a^N,s'{}^N)+\gamma u(s'{}^N)\Bigr).
\]
By Lemma~\ref{lem:PN-compact}, each set $P^N(s^N,a^N)$ is nonempty and
compact, so every inner minimum is attained.

We first prove (i). Fix $u,v\in \mathbb R^{S^N}$ and $s^N\in S^N$.
For each $a^N\in A^N$, define
\[
\Psi_{a^N}(u)
:=
\min_{p\in P^N(s^N,a^N)}
\sum_{s'{}^N\in S^N}
p(s'{}^N)\Bigl(\rho_i(s^N,a^N,s'{}^N)+\gamma u(s'{}^N)\Bigr).
\]
Choose
\[
p_u\in \arg\min_{p\in P^N(s^N,a^N)}
\sum_{s'{}^N}
p(s'{}^N)\Bigl(\rho_i(s^N,a^N,s'{}^N)+\gamma u(s'{}^N)\Bigr).
\]
Then
\begin{align*}
\Psi_{a^N}(u)-\Psi_{a^N}(v)
&\le
\sum_{s'{}^N}
p_u(s'{}^N)
\Bigl(\rho_i(s^N,a^N,s'{}^N)+\gamma u(s'{}^N)\Bigr) \\
&\quad -
\sum_{s'{}^N}
p_u(s'{}^N)
\Bigl(\rho_i(s^N,a^N,s'{}^N)+\gamma v(s'{}^N)\Bigr) \\
&=
\gamma
\sum_{s'{}^N}
p_u(s'{}^N)\bigl(u(s'{}^N)-v(s'{}^N)\bigr)
\le \gamma \|u-v\|_\infty.
\end{align*}
Exchanging the roles of $u$ and $v$ gives
\[
|\Psi_{a^N}(u)-\Psi_{a^N}(v)|\le \gamma \|u-v\|_\infty.
\]
Averaging with respect to $\pi^N(\cdot\mid s^N)$ yields
\[
|(T^{N,i}u)(s^N)-(T^{N,i}v)(s^N)|
\le \gamma \|u-v\|_\infty.
\]
Taking the supremum over $s^N\in S^N$, we obtain
\[
\|T^{N,i}u-T^{N,i}v\|_\infty \le \gamma \|u-v\|_\infty.
\]
Since $\gamma\in(0,1)$, $T^{N,i}$ is a contraction on the complete
metric space $(\mathbb R^{S^N},\|\cdot\|_\infty)$. By Banach's fixed-point
theorem, there exists a unique fixed point $u^{N,i}$ such that
\[
u^{N,i}=T^{N,i}u^{N,i}.
\]
This proves (i).

We next prove (ii) and (iii). For each $(s^N,a^N)\in S^N\times A^N$,
choose
\[
p^{N,i,\star}(\cdot\mid s^N,a^N)
\in
\arg\min_{p\in P^N(s^N,a^N)}
\sum_{s'{}^N}
p(s'{}^N)\Bigl(\rho_i(s^N,a^N,s'{}^N)+\gamma u^{N,i}(s'{}^N)\Bigr).
\]
Because $S^N\times A^N$ is finite, this defines a stationary selector
with no measurability issue. By construction,
\begin{equation}
\sum_{s'{}^N}
p^{N,i,\star}(s'{}^N\mid s^N,a^N)
\Bigl(\rho_i(s^N,a^N,s'{}^N)+\gamma u^{N,i}(s'{}^N)\Bigr)
=
\min_{p\in P^N(s^N,a^N)}
\sum_{s'{}^N}
p(s'{}^N)\Bigl(\rho_i(s^N,a^N,s'{}^N)+\gamma u^{N,i}(s'{}^N)\Bigr).
\label{eq:finiteN-minimizer-pointwise}
\end{equation}

Fix an initial state $x^N\in S^N$, and consider the controlled Markov
chain generated by the stationary profile $\pi^N$ and the stationary
kernel $p^{N,i,\star}$, started from $s_0^N=x^N$.
Using the fixed-point identity $u^{N,i}=T^{N,i}u^{N,i}$ together with
\eqref{eq:finiteN-minimizer-pointwise}, we get
\[
u^{N,i}(x^N)
=
\mathbb E_{x^N}^{\pi^N,p^{N,i,\star}}
\Bigl[
\rho_i(s_0^N,a_0^N,s_1^N)+\gamma u^{N,i}(s_1^N)
\Bigr].
\]
Iterating this identity yields, for every $T\ge 1$,
\[
u^{N,i}(x^N)
=
\mathbb E_{x^N}^{\pi^N,p^{N,i,\star}}
\left[
\sum_{t=0}^{T-1}\gamma^t \rho_i(s_t^N,a_t^N,s_{t+1}^N)
+\gamma^T u^{N,i}(s_T^N)
\right].
\]
Since $|\rho_i|\le \|r\|_\infty$, the fixed point is bounded:
\[
\|u^{N,i}\|_\infty
=
\|T^{N,i}u^{N,i}\|_\infty
\le
\|r\|_\infty+\gamma \|u^{N,i}\|_\infty,
\]
hence
\[
\|u^{N,i}\|_\infty \le \frac{\|r\|_\infty}{1-\gamma}.
\]
Therefore $\gamma^T u^{N,i}(s_T^N)\to 0$ in $L^1$, and letting
$T\to\infty$ gives
\[
u^{N,i}(x^N)
=
\mathbb E_{x^N}^{\pi^N,p^{N,i,\star}}
\left[
\sum_{t=0}^{\infty}\gamma^t \rho_i(s_t^N,a_t^N,s_{t+1}^N)
\right].
\]
Averaging over $x^N=s_0^N\sim (\mu^\star)^{\otimes N}$, we obtain
\[
\mathbb E^{\pi^N,p^{N,i,\star}}
\left[
\sum_{t=0}^{\infty}\gamma^t \rho_i(s_t^N,a_t^N,s_{t+1}^N)
\right]
=
\mathbb E_{s_0^N\sim (\mu^\star)^{\otimes N}}
\bigl[u^{N,i}(s_0^N)\bigr].
\]
So the stationary selector $p^{N,i,\star}$ achieves the value
$\mathbb E[u^{N,i}(s_0^N)]$.

It remains to show that no admissible nonstationary sequence can do
better for the minimizing player. Let $(q_t^N)_{t\ge 0}$ be any
admissible sequence with
\[
q_t^N(\cdot\mid s^N,a^N)\in P^N(s^N,a^N)
\qquad\forall t,\ (s^N,a^N).
\]
Consider the controlled process under $\pi^N$ and $(q_t^N)_{t\ge 0}$.
For every state $s^N$ and every time $t$, by the definition of the
minimum in $T^{N,i}$,
\begin{align*}
u^{N,i}(s^N)
&=
\sum_{a^N}\pi^N(a^N\mid s^N)
\min_{p\in P^N(s^N,a^N)}
\sum_{y^N} p(y^N)\Bigl(\rho_i(s^N,a^N,y^N)+\gamma u^{N,i}(y^N)\Bigr) \\
&\le
\sum_{a^N}\pi^N(a^N\mid s^N)
\sum_{y^N} q_t^N(y^N\mid s^N,a^N)
\Bigl(\rho_i(s^N,a^N,y^N)+\gamma u^{N,i}(y^N)\Bigr).
\end{align*}
Thus, along the controlled process,
\[
u^{N,i}(s_t^N)
\le
\mathbb E^{\pi^N,(q_t^N)}
\Bigl[
\rho_i(s_t^N,a_t^N,s_{t+1}^N)+\gamma u^{N,i}(s_{t+1}^N)
\,\big|\, s_t^N
\Bigr].
\]
Taking expectations, multiplying by $\gamma^t$, and summing from
$t=0$ to $T-1$, we get the telescoping estimate
\[
\mathbb E^{\pi^N,(q_t^N)}
\left[
\sum_{t=0}^{T-1}\gamma^t \rho_i(s_t^N,a_t^N,s_{t+1}^N)
\right]
\ge
\mathbb E^{\pi^N,(q_t^N)}\!\bigl[u^{N,i}(s_0^N)\bigr]
-
\mathbb E^{\pi^N,(q_t^N)}\!\bigl[\gamma^T u^{N,i}(s_T^N)\bigr].
\]
Again, boundedness of $u^{N,i}$ implies the last term tends to $0$ as
$T\to\infty$, so
\[
\mathbb E^{\pi^N,(q_t^N)}
\left[
\sum_{t=0}^{\infty}\gamma^t \rho_i(s_t^N,a_t^N,s_{t+1}^N)
\right]
\ge
\mathbb E_{s_0^N\sim (\mu^\star)^{\otimes N}}
\bigl[u^{N,i}(s_0^N)\bigr].
\]
Since $(q_t^N)_{t\ge 0}$ was arbitrary, every admissible sequence yields
a payoff at least $\mathbb E[u^{N,i}(s_0^N)]$, while the stationary
selector $p^{N,i,\star}$ attains exactly this value. Therefore
\[
J_i^N(\pi^N)
=
\mathbb E_{s_0^N\sim (\mu^\star)^{\otimes N}}
\bigl[u^{N,i}(s_0^N)\bigr],
\]
and the infimum is attained by the stationary kernel $p^{N,i,\star}$.
This proves (ii) and (iii).
\end{proof}

Fix a stationary robust mean-field equilibrium
$(\mu^\star,\pi^\star,p^\star)$.
Define the equilibrium proxy value
\begin{align}
V^\star
:=
\sum_{s\in S}\mu^\star(s)\,v_{\mu^\star}(s).
\end{align}
By Lemma~\ref{lem:reduction} applied at $\mu=\mu^\star$,
\begin{align}
V^\star
=
\sup_{\pi:S\to\Delta(A)}
\mathbb E_{s_0\sim\mu^\star}^{\pi,p^\star}
\!\left[\sum_{t=0}^\infty \gamma^t
r(s_t,a_t,s_{t+1},\mu^\star)\right],
\end{align}
and the supremum is attained at $\pi=\pi^\star$.

\begin{lemma}\label{lem:PN-compact}
Assume Assumption~\ref{ass:stat}. Then for every $N$ and $(s^N,a^N)$, the set
$\mathfrak{P}^N(s^N,a^N)$ defined in \eqref{eq:PN} is nonempty and compact.
\end{lemma}

\begin{proof}
Non-emptiness: choose for each $i$ some $p^i(\cdot\mid s^i,a^i)\in\mathfrak{P}(s^i,a^i,e^N(s^N))$ and take the product.

Compactness: each $\mathfrak{P}(s^i,a^i,e^N(s^N))$ is compact in the simplex $\DeltaS$.
Their finite product is compact in $(\DeltaS)^N$. The map $(p^1,\dots,p^N)\mapsto\bigotimes_{i=1}^N p^i$
is continuous (finite products of coordinates), hence the image $\mathfrak{P}^N(s^N,a^N)$ is compact.
\end{proof}

\begin{lemma}\label{lem:composition}
Let $X$ be finite and $Y$ Polish. Let $\lambda_n\in\Pcal(X)$ and kernels $K_n(\cdot\mid x)\in\Pcal(Y)$.
If $\lambda_n\rightarrow\lambda$ and $K_n(\cdot\mid x)\rightarrow K(\cdot\mid x)$ for each $x\in X$,
then the measures $\Lambda_n(dx,dy):=\lambda_n(dx)K_n(dy\mid x)$ satisfy $\Lambda_n\rightarrow\Lambda(dx,dy):=\lambda(dx)K(dy\mid x)$.
\end{lemma}

\begin{proof}
Since $X$ is finite, weak convergence of $\lambda_n$ is pointwise convergence of masses.
For bounded continuous $g:X\times Y\to\R$,
\begin{align}
\int g\,d\Lambda_n=\sum_{x\in X}\lambda_n(\{x\})\int_Y g(x,y)\,K_n(dy\mid x).
\end{align}
For each $x$, $y\mapsto g(x,y)$ is bounded continuous, so the inner integrals converge by $K_n(\cdot\mid x)\rightarrow K(\cdot\mid x)$,
and the coefficients converge by $\lambda_n(\{x\})\to\lambda(\{x\})$. Sum is finite, so the limits pass through the sum.
\end{proof}

\subsection{Convergence of discounted payoffs}
\label{app:discounted-payoff-convergence}

\begin{lemma}\label{lem:uc}
Let $(Z,d_Z)$ be a compact metric space and $f:Z\to\R$ continuous, with modulus of continuity $\omega_f(\delta):=\sup\{|f(z)-f(z')|:d_Z(z,z')\le\delta\}$. Then for all $\mu,\nu\in\Delta(Z)$ and all $\delta>0$,
\begin{equation}\label{eq:uc}
\Big|\int_Z f\,d\mu-\int_Z f\,d\nu\Big|\;\le\;\omega_f(\delta)\;+\;\frac{2\norm{f}_\infty}{\delta}\,\Wone(\mu,\nu).
\end{equation}
Consequently, if $(\mu_N)_N,(\nu_N)_N\subset\Delta(Z)$ satisfy $\Wone(\mu_N,\nu_N)\to0$, then $\int f\,d\mu_N-\int f\,d\nu_N\to0$.
\end{lemma}
 
\begin{proof}
Since $Z$ is compact, $f$ is bounded and uniformly continuous, so $\omega_f(\delta)<\infty$ for all $\delta>0$ and $\omega_f(\delta)\downarrow0$ as $\delta\downarrow0$. Because $Z$ is a compact Polish space, the infimum defining $\Wone(\mu,\nu)$ is attained by some coupling $\Gamma\in\Delta(Z\times Z)$ with marginals $\mu,\nu$. Writing $(Z,Z')\sim\Gamma$,
\[
\Big|\int f\,d\mu-\int f\,d\nu\Big|
\le\E_\Gamma\big|f(Z)-f(Z')\big|
\le \omega_f(\delta)+2\norm{f}_\infty\,P_\Gamma\big(d_Z(Z,Z')>\delta\big)
\le \omega_f(\delta)+\frac{2\norm{f}_\infty} 
{\delta}\,\E_\Gamma\,d_Z(Z,Z'),
\]
where the last step is Markov's inequality; since $\E_\Gamma d_Z(Z,Z')=\Wone(\mu,\nu)$, \eqref{eq:uc} follows. For the convergence claim, take $\limsup_{N}$ in \eqref{eq:uc} and then let $\delta\downarrow0$.
\end{proof}
 
\begin{proposition}\label{prop:actual-to-proxy-asymptotic}
Under Assumptions~\ref{ass:stat} and~\ref{ass:one-step-law}, for every deviation sequence $(\pi^{(N)})_N$,
\[
\lim_{N\to\infty}\Big(J^N_1\big(\pi^{N|\star,-1}\oplus\pi^{(N)}\big)\;-\;J^{\star}\big(\pi^{(N)}\big)\Big)\;=\;0,
\]
where $J^\star(\pi^{(N)})$ is the proxy payoff.
\end{proposition}
 
\begin{proof}
Define $\bar r:Z\to\R$ by $\bar r(s,a,s',\nu):=r(s,a,s',\nu)$. By Assumption~\ref{ass:stat}(1) it is bounded by $\norm{r}_\infty$; it is continuous on $Z$ because $\mathcal{S},\mathcal{A}$ are finite (so the first three coordinates are discrete) and $\nu\mapsto r(s,a,s',\nu)$ is continuous for each fixed $(s,a,s')$ by Assumption~\ref{ass:stat}(3). The space $Z=(\mathcal{S}\times\mathcal{A}\times\mathcal{S})\times\Delta(\mathcal{S})$ is compact.
 
Let $P^{N|(N)}$ be minimizing laws as in Assumption~\ref{ass:one-step-law}. Since $|r|\le\norm{r}_\infty$ and $\gamma\in(0,1)$, Fubini's theorem gives the absolutely convergent expansions
\[
J^N_1\big(\pi^{N|\star,-1}\oplus\pi^{(N)}\big)=\sum_{t\ge0}\gamma^t\int_Z \bar r\,dQ^{N|(N)}_t,
\qquad
J^\star\big(\pi^{(N)}\big)=\sum_{t\ge0}\gamma^t\int_Z \bar r\,dQ^{\star(N)}_t .
\]
Set $\Delta_{N,t}:=\int\bar r\,dQ^{N|(N)}_t-\int\bar r\,dQ^{\star(N)}_t$. For each fixed $t$, Assumption~\ref{ass:one-step-law} gives $\Wone(Q^{N|(N)}_t,Q^{\star(N)}_t)\to0$ as $N\to\infty$, so Lemma~\ref{lem:uc} applied with $f=\bar r$ yields $\Delta_{N,t}\to0$. Moreover $|\Delta_{N,t}|\le2\norm{r}_\infty$ uniformly in $(N,t)$, and $\sum_t\gamma^t\cdot2\norm{r}_\infty<\infty$; dominated convergence (over $t$, with respect to the summable envelope $2\norm{r}_\infty\gamma^t$) gives $\sum_t\gamma^t\Delta_{N,t}\to0$, which is the claim.
\end{proof}

\subsection{Approximate Nash equilibrium}
\label{app:approximate-nash}

\begin{theorem}
\label{thm:appendix-asymptotic-eps-nash}
Assume Assumptions~\ref{ass:stat} and~\ref{ass:one-step-law}. Let
\begin{align}
\pi^{N|\star}:=(\pi^\star,\ldots,\pi^\star)\in \Pi^N .
\end{align}
Then, for every $\varepsilon>0$, there exists $N(\varepsilon)\in\mathbb N$ such that, for all $N\ge N(\varepsilon)$, the profile $\pi^{N|\star}$ is an $\varepsilon$-Nash equilibrium.
\end{theorem}

\begin{proof}
By symmetry of the $N$-player robust game, it suffices to prove the $\varepsilon$-best-response inequality for player $1$.

\smallskip
\noindent\textbf{Step 1: equilibrium payoff converges to $V^\star$.}
Apply Proposition~\ref{prop:actual-to-proxy-asymptotic} to the constant deviation sequence
$\pi^{(N)}\equiv \pi^\star$. Then
$\pi^{N|(N)}=\pi^{N|\star}$ for every $N$, and the proxy chain coincides with the Markov chain driven by
$(\mu^\star,\pi^\star,p^\star)$. Hence
\begin{align}
\lim_{N\to\infty} J_1^N(\pi^{N|\star}) = V^\star .
\end{align}

\smallskip
\noindent\textbf{Step 2: every deviation sequence has proxy payoff at most $V^\star$.}
Fix any deviation sequence $(\pi^{(N)})_{N\in\mathbb N}\subset \Pi$ and set
$\pi^{N|(N)}=(\pi^{(N)},\pi^\star,\ldots,\pi^\star)$.
By Proposition~\ref{prop:actual-to-proxy-asymptotic},
\begin{align}
J_1^N(\pi^{N|(N)})
-
\mathbb E_{P^{*(N)}}\!\left[
\sum_{t=0}^\infty \gamma^t r(s_t,a_t,s_{t+1},\mu^\star)
\right]
\to 0 .
\end{align}
For each $N$, the proxy payoff on the right-hand side is bounded above by $V^\star$, because $V^\star$ is the optimal value of the fixed-kernel MDP with kernel $p^\star$ and initial law $\mu^\star$. Therefore
\begin{align}
\limsup_{N\to\infty} J_1^N(\pi^{N|(N)}) \le V^\star .
\end{align}

\smallskip
\noindent\textbf{Step 3: contradiction argument.}
Assume, toward a contradiction, that the conclusion is false. Then there exists $\varepsilon>0$ and a subsequence $(N_k)_{k\in\mathbb N}$ such that, for each $k$, one can find a unilateral deviation policy $\widehat\pi^{(N_k)}\in \Pi$ satisfying
\begin{align}
J_1^{N_k}(\pi^{N_k|\star}) + \varepsilon
<
J_1^{N_k}\!\bigl(\pi^{N_k|\star,-1}\oplus \widehat\pi^{(N_k)}\bigr).
\end{align}
Extend these deviations to a full sequence $(\pi^{(N)})_{N\in\mathbb N}$ by setting
\begin{align}
\pi^{(N)}
=
\begin{cases}
\widehat\pi^{(N_k)}, & N=N_k \text{ for some } k,\\
\pi^\star, & \text{otherwise}.
\end{cases}
\end{align}
Then
\begin{align}
J_1^{N_k}(\pi^{N_k|\star}) + \varepsilon
<
J_1^{N_k}(\pi^{N_k|(N_k)})
\qquad\forall\, k.
\end{align}
Taking $\limsup_{k\to\infty}$ and using Steps~1 and~2 yields
\begin{align}
V^\star+\varepsilon
\le
\limsup_{k\to\infty} J_1^{N_k}(\pi^{N_k|(N_k)})
\le
V^\star,
\end{align}
a contradiction.

Therefore, for every $\varepsilon>0$, there exists $N(\varepsilon)$ such that
\begin{align}
J_1^N(\pi^{N|\star})+\varepsilon
\ge
\sup_{\pi\in\Pi} J_1^N(\pi^{N|\star,-1}\oplus \pi)
\qquad\forall\, N\ge N(\varepsilon).
\end{align}
By symmetry, the same bound holds for every player, so $\pi^{N|\star}$ is an $\varepsilon$-Nash equilibrium for all sufficiently large $N$.
\end{proof}

\section{Non-Asymptotic Finite-$N$ Approximation via Proxy Comparison}
\label{app:finiteN-proxy}

This appendix proves the finite-$N$ result stated in Section~\ref{sec:nonasymptotic-approximation}. The key point is that the comparison is made between the actual $N$-player robust game and the proxy chain driven by the equilibrium kernel $p^\star$.

\subsection{Metrics and one-step laws}
\label{app:metrics-onestep}

Let
\begin{align}
X := \mathcal{S}\times\mathcal{A}\times\mathcal{S},
\qquad
Z := X\times \Delta(\mathcal{S}).
\end{align}
We equip $X$ with the discrete metric
\begin{align}
d_X(x,\tilde x) := \mathbf 1_{\{x\neq \tilde x\}},
\end{align}
and $Z$ with
\begin{align}
d_Z\bigl((x,\nu),(\tilde x,\tilde \nu)\bigr)
:=
d_X(x,\tilde x)+\|\nu-\tilde \nu\|_1
=
\mathbf 1_{\{x\neq \tilde x\}}+\|\nu-\tilde \nu\|_1.
\end{align}
We write $W_1$ for the corresponding Wasserstein-$1$ distance on $\mathcal P(Z)$.

Fix $N\in\mathbb N$ and a unilateral deviation policy $\pi\in \Pi$.
Let $P^{N,\pi}$ be a minimizing law attaining
\begin{align}
J_1^N(\pi^{N|\star,-1}\oplus \pi),
\end{align}
whose existence follows from Lemma~\ref{lem:finite-player-robust-dp}, and define
\begin{align}
Q_t^{N,\pi}
:=
\operatorname{Law}_{P^{N,\pi}}
\!\bigl(s_t^1,a_t^1,s_{t+1}^1,e^N(s_t^N)\bigr),
\qquad t\in\mathbb N_0 .
\end{align}

Define the proxy chain $P^\pi$ by
\begin{align}
s_0\sim \mu^\star,\qquad
a_t\sim \pi(\cdot\mid s_t),\qquad
s_{t+1}\sim p^\star(\cdot\mid s_t,a_t),
\end{align}
and let
\begin{align}
Q_t^\pi
:=
\operatorname{Law}_{P^\pi}\!\bigl(s_t,a_t,s_{t+1},\mu^\star\bigr),
\qquad t\in\mathbb N_0 .
\end{align}

Throughout this appendix, we assume Assumption~\ref{ass:proxy-comparison} and the following  Assumption~\ref{ass:holder-reward-2} (which generalizes the Lipschitz  Assumption~\ref{ass:holder-reward}).

\begin{assumption}[H\"older reward regularity]
\label{ass:holder-reward-2}
There exist $L_r>0$ and $\alpha\in(0,1]$ such that, for all
$s\in S$, $a\in A$, $s'\in S$, and all $\nu,\tilde \nu\in \Delta(S)$,
\[
|r(s,a,s',\nu)-r(s,a,s',\tilde \nu)|
\le
L_r \|\nu-\tilde \nu\|_1^\alpha .
\]
\end{assumption}

\subsection{Discussion of Assumption~\ref{ass:proxy-comparison}}
\label{sec:discussion-quant-proxy}

Assumption~\ref{ass:proxy-comparison} should be viewed as a quantitative robust propagation-of-chaos / Nash-certainty-equivalence hypothesis, extended from non-robust MFGs. In a standard non-robust mean-field model, analogous estimates are often obtained from regularity of the dynamics together with concentration of the empirical measure. In the present robust setting, however, the minimizing kernel in the $N$-player game may exploit the degrees of freedom of the remaining $N-1$ players to alter the empirical distribution in a way that does not vanish automatically as $N\to\infty$. For this reason, the baseline compactness and continuity assumptions used for existence do not imply Assumption~\ref{ass:proxy-comparison}; indeed, the counterexample from the asymptotic subsection already shows that even qualitative convergence may fail without an additional stabilization property.

At the same time, Assumption~\ref{ass:proxy-comparison} is substantially weaker than imposing a concrete structural sufficient condition inside the theorem. It is stated directly at the level of the observable one-step law and is agnostic to the mechanism producing that law: no product-form realization of the minimizing kernel, no particular selector for the worst-case transition, and no contractivity property of the robust population map is built into the statement. Any model-specific coupling, concentration, or perturbation argument that yields the bound in Assumption~\ref{ass:proxy-comparison} can therefore be substituted directly into the theorem without changing the payoff-comparison proof.

The assumption is also weaker than a trajectory-level approximation requirement. We do not ask for a single coupling of the entire paths of the $N$-player game and the proxy chain. Instead, we only require control of the one-step law at each time $t$, because this is exactly the object that appears in the discounted sum of stage rewards. The array $(\delta_{N,t})$ is allowed to depend on $t$, which permits moderate accumulation of finite-$N$ errors before discounting. The theorem only needs
\[
\sum_{t\ge 0}\gamma^t\Bigl(2\|r\|_\infty \delta_{N,t}+L_r\delta_{N,t}^{\alpha}\Bigr)<\infty.
\]

Finally, the uniformity over unilateral deviations is essential for the $\varepsilon_N$-Nash conclusion. If the approximation bound were available only for a fixed deviation policy, then one would obtain only a deviation-dependent comparison estimate, not a single Nash-gap bound after taking the supremum over all deviations. In many stable regimes one expects $\delta_{N,t}$ to exhibit the usual $N^{-1/2}$ scaling, possibly multiplied by a factor with mild growth in $t$, but we do not encode any such verification mechanism into the theorem because the purpose of the theorem is to isolate the weakest quantitative input needed for the finite-$N$ comparison.

\subsection{Proxy-payoff comparison under H\"older rewards}
\label{app:proxy-payoff-holder}

\begin{proposition}
\label{prop:proxy-payoff-holder}
Assume Assumptions~\ref{ass:stat}, \ref{ass:proxy-comparison}, and~\ref{ass:holder-reward-2}. Fix $N\in\mathbb N$ and a unilateral deviation policy $\pi\in \Pi$. Then
\begin{align}
\left|
J_1^N(\pi^{N|\star,-1}\oplus \pi)
-
\mathbb E_{P^\pi}\!\left[
\sum_{t=0}^\infty \gamma^t r(s_t,a_t,s_{t+1},\mu^\star)
\right]
\right|
\le
\sum_{t=0}^\infty \gamma^t
\Bigl(
2\|r\|_\infty\,\delta_{N,t}
+
L_r\,\delta_{N,t}^\alpha
\Bigr).
\end{align}
\end{proposition}

\begin{proof}
Fix $N$ and $\pi$.
Since $r$ is bounded, Tonelli's theorem yields
\begin{align}
J_1^N(\pi^{N|\star,-1}\oplus \pi)
=
\sum_{t=0}^\infty \gamma^t \int r \, dQ_t^{N,\pi},
\end{align}
and
\begin{align}
\mathbb E_{P^\pi}\!\left[
\sum_{t=0}^\infty \gamma^t r(s_t,a_t,s_{t+1},\mu^\star)
\right]
=
\sum_{t=0}^\infty \gamma^t \int r \, dQ_t^\pi .
\end{align}

Fix $t\in\mathbb N_0$, and let $\Gamma_t^\pi$ be an optimal coupling of
$Q_t^{N,\pi}$ and $Q_t^\pi$. Write
\begin{align}
(Z_t^N,Z_t^\pi)\sim \Gamma_t^\pi,
\end{align}
with
\begin{align}
Z_t^N=(X_t^N,\nu_t^N)\in X\times \Delta(\mathcal{S}),
\qquad
Z_t^\pi=(X_t^\pi,\mu^\star)\in X\times \Delta(\mathcal{S}).
\end{align}
Then
\begin{align*}
\left|\int r\, dQ_t^{N,\pi}-\int r\, dQ_t^\pi\right|
&\le
\int
\left|
r(X_t^N,\nu_t^N)-r(X_t^\pi,\mu^\star)
\right|
\, d\Gamma_t^\pi\\
&\le
2\|r\|_\infty\,\Gamma_t^\pi(X_t^N\neq X_t^\pi)
+
L_r \int \|\nu_t^N-\mu^\star\|_1^\alpha \, d\Gamma_t^\pi .
\end{align*}
Since
\begin{align}
\mathbf 1_{\{X_t^N\neq X_t^\pi\}}
\le
d_Z(Z_t^N,Z_t^\pi),
\qquad
\|\nu_t^N-\mu^\star\|_1
\le
d_Z(Z_t^N,Z_t^\pi),
\end{align}
we obtain
\begin{align}
\Gamma_t^\pi(X_t^N\neq X_t^\pi)
\le
\int d_Z \, d\Gamma_t^\pi
=
W_1(Q_t^{N,\pi},Q_t^\pi)
\le
\delta_{N,t}.
\end{align}
Moreover, by concavity of $x\mapsto x^\alpha$ on $[0,\infty)$,
\begin{align}
\int \|\nu_t^N-\mu^\star\|_1^\alpha \, d\Gamma_t^\pi
\le
\left(
\int \|\nu_t^N-\mu^\star\|_1 \, d\Gamma_t^\pi
\right)^\alpha
\le
\left(
\int d_Z \, d\Gamma_t^\pi
\right)^\alpha
\le
\delta_{N,t}^\alpha .
\end{align}
Hence
\begin{align}
\left|\int r\, dQ_t^{N,\pi}-\int r\, dQ_t^\pi\right|
\le
2\|r\|_\infty\,\delta_{N,t}
+
L_r\,\delta_{N,t}^\alpha .
\end{align}
Multiplying by $\gamma^t$ and summing over $t\ge 0$ proves the claim.
\end{proof}

\subsection{Finite-$N$ $\varepsilon_N$-Nash bound}
\label{app:finiteN-epsnash-proof}

\begin{theorem}
\label{thm:appendix-finiteN-holder}
Assume Assumptions~\ref{ass:stat}, \ref{ass:proxy-comparison}, and~\ref{ass:holder-reward-2}. Let
$(\mu^\star,\pi^\star,p^\star)$ be a stationary robust mean-field equilibrium and let
\begin{align}
\pi^{N|\star}:=(\pi^\star,\ldots,\pi^\star)\in \Pi^N .
\end{align}
Then $\pi^{N|\star}$ is an $\varepsilon_N$-Nash equilibrium, where
\begin{align}
\varepsilon_N
:=
2\sum_{t=0}^\infty \gamma^t
\Bigl(
2\|r\|_\infty\,\delta_{N,t}
+
L_r\,\delta_{N,t}^\alpha
\Bigr).
\end{align}
\end{theorem}

\begin{proof}
Define
\begin{align}
B_N
:=
\sum_{t=0}^\infty \gamma^t
\Bigl(
2\|r\|_\infty\,\delta_{N,t}
+
L_r\,\delta_{N,t}^\alpha
\Bigr).
\end{align}
Also define the equilibrium proxy value
\begin{align}
V^\star
:=
\sup_{\pi:S\to\Delta(A)}
\mathbb E_{s_0\sim\mu^\star}^{\pi,p^\star}
\!\left[
\sum_{t=0}^\infty \gamma^t r(s_t,a_t,s_{t+1},\mu^\star)
\right].
\end{align}
By Lemma~\ref{lem:reduction}, the supremum is attained at $\pi=\pi^\star$.

Apply Proposition~\ref{prop:proxy-payoff-holder} with $\pi=\pi^\star$. Then
\begin{align}
J_1^N(\pi^{N|\star})
\ge
V^\star - B_N .
\end{align}
Now fix any unilateral deviation policy $\pi\in \Pi$. Applying Proposition~\ref{prop:proxy-payoff-holder} to this $\pi$ gives
\begin{align}
J_1^N(\pi^{N|\star,-1}\oplus \pi)
\le
\mathbb E_{P^\pi}\!\left[
\sum_{t=0}^\infty \gamma^t r(s_t,a_t,s_{t+1},\mu^\star)
\right]
+
B_N
\le
V^\star + B_N .
\end{align}
Therefore,
\begin{align}
J_1^N(\pi^{N|\star,-1}\oplus \pi)-J_1^N(\pi^{N|\star})
\le
2B_N
=
2\sum_{t=0}^\infty \gamma^t
\Bigl(
2\|r\|_\infty\,\delta_{N,t}
+
L_r\,\delta_{N,t}^\alpha
\Bigr).
\end{align}
Taking the supremum over $\pi\in \Pi$ gives the desired best-response bound for player $1$. By symmetry of the game, the same estimate holds for every player, and thus $\pi^{N|\star}$ is an $\varepsilon_N$-Nash equilibrium.
\end{proof}

\begin{corollary}[Rate under $\delta_{N,t}\le C_\delta(1+t)/\sqrt N$]
\label{cor:holder-rate}
Assume the hypotheses of Theorem~\ref{thm:appendix-finiteN-holder}, and suppose that there exists $C_\delta>0$ such that
\begin{align}
\delta_{N,t}\le \frac{C_\delta(1+t)}{\sqrt N}
\qquad\forall\, N\in\mathbb N,\; t\in\mathbb N_0 .
\end{align}
Then
\begin{align}
\varepsilon_N
\le
\frac{4\|r\|_\infty C_\delta}{\sqrt N}
\sum_{t=0}^\infty \gamma^t(1+t)
+
\frac{2L_r C_\delta^\alpha}{N^{\alpha/2}}
\sum_{t=0}^\infty \gamma^t(1+t)^\alpha,
\end{align}
and hence
\begin{align}
\varepsilon_N = O\!\left(N^{-\alpha/2}\right).
\end{align}
\end{corollary}

\begin{proof}
Under the assumed estimate on $\delta_{N,t}$,
\begin{align}
2\|r\|_\infty\,\delta_{N,t}
\le
\frac{2\|r\|_\infty C_\delta(1+t)}{\sqrt N},
\qquad
L_r\,\delta_{N,t}^\alpha
\le
\frac{L_r C_\delta^\alpha(1+t)^\alpha}{N^{\alpha/2}}.
\end{align}
Insert these bounds into the formula for $\varepsilon_N$ from
Theorem~\ref{thm:appendix-finiteN-holder}. Since
\begin{align}
\sum_{t=0}^\infty \gamma^t(1+t)<\infty,
\qquad
\sum_{t=0}^\infty \gamma^t(1+t)^\alpha<\infty,
\end{align}
the stated estimate follows. Because $\alpha\in(0,1]$, the slower-decaying term is $N^{-\alpha/2}$, which gives the displayed rate.
\end{proof}

\begin{corollary}[Lipschitz specialization]
\label{cor:appendix-lipschitz}
Assume Assumption~\ref{ass:proxy-comparison}, and assume, instead of Assumption~\ref{ass:holder-reward-2}, that
\begin{align}
|r(s,a,s',\nu)-r(s,a,s',\tilde \nu)|
\le
L_r \|\nu-\tilde \nu\|_1
\qquad
\forall\, s,a,s',\nu,\tilde \nu .
\end{align}
Then $\pi^{N|\star}$ is an $\varepsilon_N^{\mathrm{Lip}}$-Nash equilibrium with
\begin{align}
\varepsilon_N^{\mathrm{Lip}}
=
2(2\|r\|_\infty+L_r)\sum_{t=0}^\infty \gamma^t \delta_{N,t}.
\end{align}
If, in addition,
\begin{align}
\delta_{N,t}\le \frac{C_\delta(1+t)}{\sqrt N}
\qquad\forall\, N\in\mathbb N,\; t\in\mathbb N_0 ,
\end{align}
then
\begin{align}
\varepsilon_N^{\mathrm{Lip}}
\le
\frac{2(2\|r\|_\infty+L_r)C_\delta}{(1-\gamma)^2\sqrt N}.
\end{align}
\end{corollary}

\begin{proof}
Set $\alpha=1$ in Theorem~\ref{thm:appendix-finiteN-holder}. Then
\begin{align}
\varepsilon_N
=
2\sum_{t=0}^\infty \gamma^t
\bigl(2\|r\|_\infty+L_r\bigr)\delta_{N,t}
=
2(2\|r\|_\infty+L_r)\sum_{t=0}^\infty \gamma^t \delta_{N,t},
\end{align}
which proves the first statement.
If, moreover,
$\delta_{N,t}\le C_\delta(1+t)/\sqrt N$, then
\begin{align}
\varepsilon_N^{\mathrm{Lip}}
\le
\frac{2(2\|r\|_\infty+L_r)C_\delta}{\sqrt N}
\sum_{t=0}^\infty \gamma^t(1+t).
\end{align}
Since
\begin{align}
\sum_{t=0}^\infty \gamma^t(1+t)=\frac{1}{(1-\gamma)^2},
\end{align}
the explicit bound follows.
\end{proof}

\subsection{A concrete sufficient regime}
\label{sec:appendix-d5}

The non-asymptotic theorem is stated under the law-level proxy-comparison assumption. In this section, we aim to verify that assumption in a concrete stable regime.

For a unilateral deviation policy $\pi\in\Pi$ and $N\in\mathbb N$, let
\[
\pi^{N,\pi}:=(\pi,\pi^\star,\ldots,\pi^\star)\in \Pi^N.
\]

\begin{assumption}[A sufficient stable regime]
\label{ass:delta-sufficient}
There exist a selector $\bar p:S\times A\times \Delta(S)\to \Delta(S)$ and constants $L_P\ge 0$ and $\rho_{\mathrm{mix}}\in[0,1)$ such that:
\begin{enumerate}
    \item For every $(s,a,\nu)\in S\times A\times \Delta(S)$,
    \[
    \bar p(\cdot\mid s,a,\nu)\in P(s,a,\nu).
    \]
    \item The selector is Lipschitz in the population argument:
    \[
    \max_{(s,a)\in S\times A}
    \bigl\|\bar p(\cdot\mid s,a,\nu)-\bar p(\cdot\mid s,a,\tilde\nu)\bigr\|_1
    \le L_P\|\nu-\tilde\nu\|_1,
    \qquad \forall \nu,\tilde\nu\in \Delta(S).
    \]
    \item For each $\nu\in\Delta(S)$, define the $\pi^\star$-controlled kernel
    \[
    K_\nu(s'\mid s):=\sum_{a\in A}\pi^\star(a\mid s)\,\bar p(s'\mid s,a,\nu).
    \]
    Its Dobrushin coefficient satisfies
    \[
    \alpha(K_\nu):=\max_{s,\tilde s\in S} d_{\mathrm{TV}}\bigl(K_\nu(\cdot\mid s),K_\nu(\cdot\mid \tilde s)\bigr)
    \le \rho_{\mathrm{mix}},
    \qquad \forall \nu\in\Delta(S).
    \]
    \item The selector is compatible with the mean-field worst-case kernel at equilibrium:
    \[
    p^\star(\cdot\mid s,a)=\bar p(\cdot\mid s,a,\mu^\star),
    \qquad \forall (s,a)\in S\times A.
    \]
    \item For every $N\in\mathbb N$ and every unilateral deviation policy $\pi\in\Pi$, there exists an admissible minimizing kernel sequence $(p_t^{N,\pi,\star})_{t\ge 0}$ attaining the infimum in the definition of $J_1^N(\pi^{N,\pi})$ such that for every $t\ge 0$ and every $(s^N,a^N)\in S^N\times A^N$,
    \[
    p_t^{N,\pi,\star}(\cdot\mid s^N,a^N)
    =\bigotimes_{i=1}^N \bar p(\cdot\mid s^i,a^i,e_N(s^N)).
    \]
\end{enumerate}
Finally, assume
\[
\rho:=\rho_{\mathrm{mix}}+L_P<1.
\]
\end{assumption}

Assumption~\ref{ass:delta-sufficient} is strong, but it is only used to justify a benchmark rate for the quantitative proxy comparison. The product-form realization in part~(5) is precisely what prevents the minimizing $N$-player kernel from introducing additional cross-agent dependence beyond the empirical distribution.

For fixed $N$ and $\pi\in\Pi$, let $\mathbb P^{N,\pi}$ denote the law induced by the profile $\pi^{N,\pi}$ and the minimizing kernel sequence from Assumption~\ref{ass:delta-sufficient}(5). Let
\[
\mu_t^N:=e_N(s_t^N).
\]
Let $\mathbb P^\pi$ denote the proxy chain defined by
\[
s_0\sim \mu^\star,
\qquad a_t\sim \pi(\cdot\mid s_t),
\qquad s_{t+1}\sim p^\star(\cdot\mid s_t,a_t),
\]
and write
\[
Q_t^{N,\pi}:=\mathrm{Law}_{\mathbb P^{N,\pi}}(s_t^1,a_t^1,s_{t+1}^1,\mu_t^N),
\qquad
Q_t^\pi:=\mathrm{Law}_{\mathbb P^\pi}(s_t,a_t,s_{t+1},\mu^\star).
\]

\begin{lemma}[Contraction of the robust population map]
\label{lem:appendix-d5-contraction}
Under Assumption~\ref{ass:delta-sufficient}, the map
\[
F(\nu):=\nu K_\nu,
\qquad \nu\in\Delta(S),
\]
satisfies
\[
\|F(\nu)-F(\tilde\nu)\|_1\le \rho\,\|\nu-\tilde\nu\|_1,
\qquad \forall \nu,\tilde\nu\in\Delta(S).
\]
In particular, $F$ is a contraction on $(\Delta(S),\|\cdot\|_1)$.
\end{lemma}

\begin{proof}
Fix $\nu,\tilde\nu\in\Delta(S)$. Add and subtract $\nu K_{\tilde\nu}$:
\[
\|F(\nu)-F(\tilde\nu)\|_1
=\|\nu K_\nu-\tilde\nu K_{\tilde\nu}\|_1
\le \|\nu K_\nu-\nu K_{\tilde\nu}\|_1+\|\nu K_{\tilde\nu}-\tilde\nu K_{\tilde\nu}\|_1.
\]
For the first term,
\[
\|\nu K_\nu-\nu K_{\tilde\nu}\|_1
\le \max_{s\in S}\|K_\nu(\cdot\mid s)-K_{\tilde\nu}(\cdot\mid s)\|_1.
\]
Moreover, for each $s\in S$,
\begin{align*}
\|K_\nu(\cdot\mid s)-K_{\tilde\nu}(\cdot\mid s)\|_1
&=\left\|\sum_{a\in A}\pi^\star(a\mid s)\Bigl(\bar p(\cdot\mid s,a,\nu)-\bar p(\cdot\mid s,a,\tilde\nu)\Bigr)\right\|_1 \\
&\le \sum_{a\in A}\pi^\star(a\mid s)\,\bigl\|\bar p(\cdot\mid s,a,\nu)-\bar p(\cdot\mid s,a,\tilde\nu)\bigr\|_1 \\
&\le L_P\|\nu-\tilde\nu\|_1.
\end{align*}
Hence
\[
\|\nu K_\nu-\nu K_{\tilde\nu}\|_1\le L_P\|\nu-\tilde\nu\|_1.
\]
For the second term, the standard Dobrushin contraction inequality gives
\[
\|\nu K_{\tilde\nu}-\tilde\nu K_{\tilde\nu}\|_1
\le \alpha(K_{\tilde\nu})\,\|\nu-\tilde\nu\|_1
\le \rho_{\mathrm{mix}}\|\nu-\tilde\nu\|_1.
\]
Combining the two bounds yields
\[
\|F(\nu)-F(\tilde\nu)\|_1\le (L_P+\rho_{\mathrm{mix}})\|\nu-\tilde\nu\|_1 = \rho\|\nu-\tilde\nu\|_1.
\]
\end{proof}

\begin{lemma}[Uniform empirical-distribution error]
\label{lem:appendix-d5-empirical}
Under Assumption~\ref{ass:delta-sufficient}, for every $N\in\mathbb N$, every unilateral deviation policy $\pi\in\Pi$, and every $t\ge 0$,
\[
\mathbb E^{N,\pi}\bigl[\|\mu_t^N-\mu^\star\|_1\bigr]
\le \frac{2\sqrt{|S|}}{(1-\rho)\sqrt N}+\frac{2}{(1-\rho)N}.
\]
Consequently,
\[
\sup_{t\ge 0}\mathbb E^{N,\pi}\bigl[\|\mu_t^N-\mu^\star\|_1\bigr]
\le \frac{C_\mu}{\sqrt N},
\qquad
C_\mu:=\frac{2\sqrt{|S|}+2}{1-\rho}.
\]
The bound is uniform over the unilateral deviation policy $\pi$.
\end{lemma}

\begin{proof}
Fix $N$ and $\pi$, and write
\[
a_t:=\mathbb E^{N,\pi}\bigl[\|\mu_t^N-\mu^\star\|_1\bigr].
\]
Let
\[
\mathcal F_t:=\sigma(s_0^N,\ldots,s_t^N)
\]
be the state filtration up to time $t$. By Assumption~\ref{ass:delta-sufficient}(5), conditional on $\mathcal F_t$ the next states $s_{t+1}^1,\ldots,s_{t+1}^N$ are independent, and their conditional laws are
\[
\mathbb P^{N,\pi}(s_{t+1}^1\in\cdot\mid \mathcal F_t)
=\sum_{a\in A}\pi(a\mid s_t^1)\,\bar p(\cdot\mid s_t^1,a,\mu_t^N),
\]
and, for $i\ge 2$,
\[
\mathbb P^{N,\pi}(s_{t+1}^i\in\cdot\mid \mathcal F_t)
=\sum_{a\in A}\pi^\star(a\mid s_t^i)\,\bar p(\cdot\mid s_t^i,a,\mu_t^N)
=K_{\mu_t^N}(\cdot\mid s_t^i).
\]
Define the two random probability vectors
\[
G_t(\cdot):=\sum_{a\in A}\pi(a\mid s_t^1)\,\bar p(\cdot\mid s_t^1,a,\mu_t^N),
\qquad
H_t(\cdot):=K_{\mu_t^N}(\cdot\mid s_t^1).
\]
Then
\begin{align*}
\mathbb E^{N,\pi}[\mu_{t+1}^N\mid \mathcal F_t]
&=\frac1N G_t + \frac1N\sum_{i=2}^N K_{\mu_t^N}(\cdot\mid s_t^i) \\
&= \mu_t^N K_{\mu_t^N} + \frac1N\bigl(G_t-H_t\bigr) \\
&= F(\mu_t^N)+\frac1N\bigl(G_t-H_t\bigr).
\end{align*}
Set
\[
\xi_{t+1}:=\mu_{t+1}^N-\mathbb E^{N,\pi}[\mu_{t+1}^N\mid \mathcal F_t].
\]
We first bound the fluctuation term. For any $s'\in S$,
\[
\mu_{t+1}^N(s')=\frac1N\sum_{i=1}^N \mathbf 1_{\{s_{t+1}^i=s'\}}.
\]
Conditional on $\mathcal F_t$, the indicators on the right-hand side are independent Bernoulli random variables, hence
\[
\mathrm{Var}\bigl(\mu_{t+1}^N(s')\mid \mathcal F_t\bigr)
\le \frac1{N^2}\sum_{i=1}^N \mathbb P^{N,\pi}(s_{t+1}^i=s'\mid \mathcal F_t).
\]
Summing over $s'\in S$ gives
\[
\sum_{s'\in S}\mathrm{Var}\bigl(\mu_{t+1}^N(s')\mid \mathcal F_t\bigr)\le \frac1N.
\]
Therefore, by Jensen and Cauchy--Schwarz,
\begin{align*}
\mathbb E^{N,\pi}[\|\xi_{t+1}\|_1\mid \mathcal F_t]
&\le \sqrt{|S|}\,\mathbb E^{N,\pi}[\|\xi_{t+1}\|_2\mid \mathcal F_t] \\
&\le \sqrt{|S|}\,\Bigl(\mathbb E^{N,\pi}[\|\xi_{t+1}\|_2^2\mid \mathcal F_t]\Bigr)^{1/2} \\
&= \sqrt{|S|}\,\left(\sum_{s'\in S}\mathrm{Var}\bigl(\mu_{t+1}^N(s')\mid \mathcal F_t\bigr)\right)^{1/2} \\
&\le \sqrt{\frac{|S|}{N}}.
\end{align*}
Next, by the triangle inequality,
\[
\|\mu_{t+1}^N-\mu^\star\|_1
\le \|\xi_{t+1}\|_1
+\|F(\mu_t^N)-F(\mu^\star)\|_1
+\frac1N\|G_t-H_t\|_1.
\]
By Assumption~\ref{ass:delta-sufficient}(4) and the consistency condition of the mean-field equilibrium,
\[
F(\mu^\star)=\mu^\star.
\]
Also, $G_t,H_t\in \Delta(S)$, so $\|G_t-H_t\|_1\le 2$. Taking expectations and using Lemma~\ref{lem:appendix-d5-contraction} yields
\[
a_{t+1}
\le \rho a_t + \sqrt{\frac{|S|}{N}} + \frac{2}{N}.
\]
Iterating this recursion gives
\[
a_t
\le \rho^t a_0 + \frac{1-\rho^t}{1-\rho}\left(\sqrt{\frac{|S|}{N}}+\frac{2}{N}\right).
\]
It remains to bound $a_0$. Since $\mu_0^N$ is the empirical distribution of $N$ i.i.d. samples from $\mu^\star$, the same variance argument as above gives
\[
a_0=\mathbb E^{N,\pi}[\|\mu_0^N-\mu^\star\|_1]\le \sqrt{\frac{|S|}{N}}.
\]
Hence
\begin{align*}
a_t
&\le \sqrt{\frac{|S|}{N}} + \frac{1}{1-\rho}\left(\sqrt{\frac{|S|}{N}}+\frac{2}{N}\right) \\
&= \frac{2-\rho}{1-\rho}\sqrt{\frac{|S|}{N}} + \frac{2}{(1-\rho)N} \\
&\le \frac{2\sqrt{|S|}}{(1-\rho)\sqrt N}+\frac{2}{(1-\rho)N}.
\end{align*}
The stated uniform $C_\mu/\sqrt N$ bound follows because $N^{-1}\le N^{-1/2}$ for $N\ge 1$.
\end{proof}

\begin{proposition}[One-step law comparison]
\label{prop:appendix-d5-onestep}
Under Assumption~\ref{ass:delta-sufficient}, for every $N\in\mathbb N$, every unilateral deviation policy $\pi\in\Pi$, and every $t\ge 0$,
\[
W_1\bigl(Q_t^{N,\pi},Q_t^\pi\bigr)
\le \mathbb E^{N,\pi}\bigl[\|\mu_t^N-\mu^\star\|_1\bigr]
+L_P\sum_{k=0}^t \mathbb E^{N,\pi}\bigl[\|\mu_k^N-\mu^\star\|_1\bigr].
\]
In particular,
\[
W_1\bigl(Q_t^{N,\pi},Q_t^\pi\bigr)
\le \frac{C_\mu}{\sqrt N}\bigl(1+L_P(t+1)\bigr).
\]
\end{proposition}

\begin{proof}
Fix $N$, $\pi$, and $t$. We construct a coupling between
\[
Z_t^{N,\pi}:=(s_t^1,a_t^1,s_{t+1}^1,\mu_t^N)
\qquad\text{and}\qquad
Z_t^\pi:=(s_t,a_t,s_{t+1},\mu^\star).
\]
Start with $s_0=s_0^1$, where $s_0^1\sim \mu^\star$.

At time $t$, given $(s_t^1,s_t)$, proceed as follows.
\begin{enumerate}
    \item If $s_t^1=s_t$, sample a single action $a_t\sim \pi(\cdot\mid s_t)$ and set $a_t^1:=a_t$.
    \item If $s_t^1\neq s_t$, couple $a_t^1\sim \pi(\cdot\mid s_t^1)$ and $a_t\sim \pi(\cdot\mid s_t)$ arbitrarily.
\end{enumerate}
Thus,
\[
\{s_t^1=s_t\}\subseteq \{a_t^1=a_t\}.
\]
Next, conditional on $(s_t^1,a_t^1,s_t,a_t,\mu_t^N)$,
\begin{itemize}
    \item under $\mathbb P^{N,\pi}$, the first player's next-state law is $\bar p(\cdot\mid s_t^1,a_t^1,\mu_t^N)$, by Assumption~\ref{ass:delta-sufficient}(5);
    \item under $\mathbb P^\pi$, the proxy next-state law is $p^\star(\cdot\mid s_t,a_t)=\bar p(\cdot\mid s_t,a_t,\mu^\star)$, by Assumption~\ref{ass:delta-sufficient}(4).
\end{itemize}
Whenever $(s_t^1,a_t^1)=(s_t,a_t)$, couple $s_{t+1}^1$ and $s_{t+1}$ by a maximal coupling of
\[
\bar p(\cdot\mid s_t,a_t,\mu_t^N)
\qquad\text{and}\qquad
\bar p(\cdot\mid s_t,a_t,\mu^\star).
\]
When $(s_t^1,a_t^1)\neq (s_t,a_t)$, couple $s_{t+1}^1$ and $s_{t+1}$ arbitrarily.

Define
\[
e_t:=\mathbb P(s_t^1\neq s_t)
\]
under this coupling. Then $e_0=0$, and by the union bound,
\begin{align*}
e_{t+1}
&\le \mathbb P(s_t^1\neq s_t)
+\mathbb P\bigl(s_t^1=s_t,\ a_t^1=a_t,\ s_{t+1}^1\neq s_{t+1}\bigr) \\
&\le e_t + \mathbb E\Bigl[d_{\mathrm{TV}}\bigl(\bar p(\cdot\mid s_t,a_t,\mu_t^N),\bar p(\cdot\mid s_t,a_t,\mu^\star)\bigr)\Bigr] \\
&\le e_t + \frac{L_P}{2}\,\mathbb E^{N,\pi}\bigl[\|\mu_t^N-\mu^\star\|_1\bigr].
\end{align*}
By induction,
\[
e_t\le \frac{L_P}{2}\sum_{k=0}^{t-1}\mathbb E^{N,\pi}\bigl[\|\mu_k^N-\mu^\star\|_1\bigr],
\qquad
e_t+e_{t+1}\le L_P\sum_{k=0}^{t}\mathbb E^{N,\pi}\bigl[\|\mu_k^N-\mu^\star\|_1\bigr].
\]
Now recall that the metric on $Z=(S\times A\times S)\times \Delta(S)$ is
\[
d_Z\bigl((x,\nu),(\tilde x,\tilde\nu)\bigr)=\mathbf 1_{\{x\neq \tilde x\}}+\|\nu-\tilde\nu\|_1.
\]
Therefore,
\[
d_Z\bigl(Z_t^{N,\pi},Z_t^\pi\bigr)
=\mathbf 1_{\{(s_t^1,a_t^1,s_{t+1}^1)\neq (s_t,a_t,s_{t+1})\}}+\|\mu_t^N-\mu^\star\|_1.
\]
By construction, if both $s_t^1=s_t$ and $s_{t+1}^1=s_{t+1}$ occur, then necessarily $a_t^1=a_t$, and hence
\[
\mathbf 1_{\{(s_t^1,a_t^1,s_{t+1}^1)\neq (s_t,a_t,s_{t+1})\}}
\le \mathbf 1_{\{s_t^1\neq s_t\}}+\mathbf 1_{\{s_{t+1}^1\neq s_{t+1}\}}.
\]
Taking expectations gives
\begin{align*}
\mathbb E\bigl[d_Z(Z_t^{N,\pi},Z_t^\pi)\bigr]
&\le e_t+e_{t+1}+\mathbb E^{N,\pi}\bigl[\|\mu_t^N-\mu^\star\|_1\bigr] \\
&\le L_P\sum_{k=0}^{t}\mathbb E^{N,\pi}\bigl[\|\mu_k^N-\mu^\star\|_1\bigr]
+\mathbb E^{N,\pi}\bigl[\|\mu_t^N-\mu^\star\|_1\bigr].
\end{align*}
Since $W_1(Q_t^{N,\pi},Q_t^\pi)$ is the infimum of this expected cost over all couplings, the first inequality follows. The second inequality is immediate from Lemma~\ref{lem:appendix-d5-empirical}.
\end{proof}

\begin{corollary}
\label{cor:appendix-d5-delta}
Under Assumption~\ref{ass:delta-sufficient}, it holds with
\[
\delta_{N,t}:=\frac{C_\mu}{\sqrt N}\bigl(1+L_P(t+1)\bigr),
\qquad
C_\mu:=\frac{2\sqrt{|S|}+2}{1-\rho}.
\]
\end{corollary}

\begin{proof}
This is exactly Proposition~\ref{prop:appendix-d5-onestep}.
\end{proof}

\begin{corollary}[Resulting finite-$N$ Nash-gap bound]
\label{cor:appendix-d5-holder-rate}
Assume, in addition, the $\alpha$-H\"older reward condition; namely, for some $\alpha\in(0,1]$ and $L_r\ge 0$,
\[
|r(s,a,s',\nu)-r(s,a,s',\tilde\nu)|\le L_r\|\nu-\tilde\nu\|_1^{\alpha},
\qquad \forall s,a,s',\nu,\tilde\nu.
\]
Let $R_\infty:=\|r\|_\infty$. Then the Nash-gap bound satisfies
\begin{align*}
\varepsilon_N
\le \frac{4R_\infty C_\mu}{\sqrt N}
\left(\frac{1}{1-\gamma}+\frac{L_P}{(1-\gamma)^2}\right)  + \frac{2L_r C_\mu^{\alpha}}{N^{\alpha/2}}
\left(\frac{1}{1-\gamma}+\frac{L_P^{\alpha}}{(1-\gamma)^2}\right).
\end{align*}
In particular,
\[
\varepsilon_N = O\bigl(N^{-\alpha/2}\bigr).
\]
When $\alpha=1$ (the Lipschitz case), this becomes
\[
\varepsilon_N
\le \frac{2(2R_\infty+L_r)C_\mu}{\sqrt N}
\left(\frac{1}{1-\gamma}+\frac{L_P}{(1-\gamma)^2}\right).
\]
\end{corollary}

\begin{proof}
By Corollary~\ref{cor:appendix-d5-delta}, Theorem~\ref{thm:appendix-finiteN-holder} applies with
\[
\delta_{N,t}=\frac{C_\mu}{\sqrt N}\bigl(1+L_P(t+1)\bigr).
\]
Using Theorem~\ref{thm:appendix-finiteN-holder},
\[
\varepsilon_N
\le 2\sum_{t\ge 0}\gamma^t\Bigl(2R_\infty\delta_{N,t}+L_r\delta_{N,t}^{\alpha}\Bigr).
\]
For the linear term,
\[
\sum_{t\ge 0}\gamma^t\delta_{N,t}
= \frac{C_\mu}{\sqrt N}\sum_{t\ge 0}\gamma^t\bigl(1+L_P(t+1)\bigr)
= \frac{C_\mu}{\sqrt N}\left(\frac{1}{1-\gamma}+\frac{L_P}{(1-\gamma)^2}\right).
\]
For the H\"older term, since $\alpha\in(0,1]$ and $(x+y)^\alpha\le x^\alpha+y^\alpha$ for $x,y\ge 0$,
\[
\bigl(1+L_P(t+1)\bigr)^\alpha
\le 1+L_P^\alpha (t+1)^\alpha
\le 1+L_P^\alpha (t+1).
\]
Hence
\begin{align*}
\sum_{t\ge 0}\gamma^t\delta_{N,t}^{\alpha}
&\le \frac{C_\mu^{\alpha}}{N^{\alpha/2}}\sum_{t\ge 0}\gamma^t\Bigl(1+L_P^\alpha(t+1)\Bigr) \\
&= \frac{C_\mu^{\alpha}}{N^{\alpha/2}}\left(\frac{1}{1-\gamma}+\frac{L_P^\alpha}{(1-\gamma)^2}\right).
\end{align*}
Substituting the last two estimates into the theorem yields the stated bound. The Lipschitz specialization is the case $\alpha=1$.
\end{proof}

\section{Algorithmic solutions to Robust MFGs}
\subsection{Discussion: verifying the Lipschitz-mixing assumptions}
\label{subsec:alg_discussion}

\paragraph{(1) Verifying mixing via a Doeblin/minorization condition.}
A convenient sufficient condition for \eqref{eq:alg_mix_assump} is: there exist $\eta\in(0,1]$ and $\psi\in\Delta(\mathcal{S})$
such that for all $\mu$, all $s\in S$,
\begin{align}
K_\mu(\cdot\mid s) \ge \eta\,\psi(\cdot)\quad \text{(componentwise)}.
\end{align}
Then $\alpha(K_\mu)\le 1-\eta$ uniformly over $\mu$. 

Thus, as long as the kernel $K_\mu$ has a full support set of $\mcs$, (i.e., every entry of $K_\mu$ is positive), the condition will be satisfied. We highlight that this condition can be satisfied by distributional uncertainty set with small radius, and is also considered in standard robust RL literature, e.g., \citep{wang2023finite,chen2025sample}. 

\paragraph{(2) Bounding $L_K$ by policy and nature sensitivities.}
From $K_\mu(\cdot\mid s)=\sum_a \pi_\mu(a\mid s)p_\mu(\cdot\mid s,a)$,
one can bound, for each $s$,
\begin{align}
\|K_\mu(\cdot\mid s)-K_{\tilde\mu}(\cdot\mid s)\|_1
\le
\underbrace{\max_{a}\|p_\mu(\cdot\mid s,a)-p_{\tilde\mu}(\cdot\mid s,a)\|_1}_{\text{worst-case kernel sensitivity}}
+
\underbrace{\|\pi_\mu(\cdot\mid s)-\pi_{\tilde\mu}(\cdot\mid s)\|_1}_{\text{policy sensitivity}},
\end{align}
since $\|p_{\tilde\mu}(\cdot\mid s,a)\|_1=1$. Hence a sufficient condition for \eqref{eq:alg_LK_assump} is the existence
of Lipschitz selections $\mu\mapsto p_\mu$ and $\mu\mapsto \pi_\mu$.



We then verify (or relax) the kernel Lipschitz condition
\eqref{eq:alg_LK_assump} in Assumption~\ref{ass:alg_contractivity} for several standard ambiguity correspondences.
We focus on the ambiguity families: total-variation balls and Wasserstein balls.

\paragraph{Notation.}
Fix $(s,a)\in S\times A$. For a population distribution $\mu\in\Delta(\mathcal{S})$ and a bounded value function $v\in\mathbb{R}^S$,
define the one-step \emph{cost vector}
\begin{align}
c_{\mu,v}^{s,a}(s') ~:=~ r(s,a,s',\mu)+\gamma v(s'),\qquad s'\in S.
\label{eq:alg_cost_vector}
\end{align}
For a fixed selection rule $\mu\mapsto (\pi_\mu,p_\mu)$, recall $K_\mu(\cdot\mid s)=\sum_a \pi_\mu(a\mid s)p_\mu(\cdot\mid s,a)$.

\subsubsection{From policy/kernel Lipschitzness to $L_K$}
\label{subsubsec:alg_LK_decomp}

\begin{lemma}[A convenient bound for $L_K$]
\label{lem:alg_LK_decomp}
Assume there exist constants $L_\pi,L_p\ge 0$ such that for all $\mu,\tilde\mu\in\Delta(\mathcal{S})$,
\begin{align}
\max_{s\in S}\|\pi_\mu(\cdot\mid s)-\pi_{\tilde\mu}(\cdot\mid s)\|_1
&\le L_\pi \|\mu-\tilde\mu\|_1,
\label{eq:alg_Lpi}
\\
\max_{(s,a)\in S\times A}\|p_\mu(\cdot\mid s,a)-p_{\tilde\mu}(\cdot\mid s,a)\|_1
&\le L_p \|\mu-\tilde\mu\|_1.
\label{eq:alg_Lp}
\end{align}
Then the induced kernel $K_\mu$ satisfies \eqref{eq:alg_LK_assump} with
\begin{align}
L_K \le L_\pi + L_p.
\label{eq:alg_LK_bound_LpiLp}
\end{align}
\end{lemma}

\begin{proof}
Fix $\mu,\tilde\mu$ and $s\in S$. Add and subtract $\sum_a \pi_\mu(a\mid s)\,p_{\tilde\mu}(\cdot\mid s,a)$:
\begin{align}
\|K_\mu(\cdot\mid s)-K_{\tilde\mu}(\cdot\mid s)\|_1
\le
\Big\|\sum_a \pi_\mu(a\mid s)\big(p_\mu(\cdot\mid s,a)-p_{\tilde\mu}(\cdot\mid s,a)\big)\Big\|_1
+\Big\|\sum_a \big(\pi_\mu(a\mid s)-\pi_{\tilde\mu}(a\mid s)\big)p_{\tilde\mu}(\cdot\mid s,a)\Big\|_1.
\end{align}
The first term is at most $\max_a\|p_\mu(\cdot\mid s,a)-p_{\tilde\mu}(\cdot\mid s,a)\|_1$.
The second term is at most $\|\pi_\mu(\cdot\mid s)-\pi_{\tilde\mu}(\cdot\mid s)\|_1$ since $\|p_{\tilde\mu}(\cdot\mid s,a)\|_1=1$.
Taking maxima over $s$ and using \eqref{eq:alg_Lpi}--\eqref{eq:alg_Lp} yields \eqref{eq:alg_LK_bound_LpiLp}.
\end{proof}

\paragraph{Policy Lipschitzness: gap vs.\ softmax.}
In general, argmax-based policies can be discontinuous in $\mu$ due to ties.
Two standard sufficient conditions to obtain \eqref{eq:alg_Lpi} are:
(i) \emph{local action-gap/uniqueness} near $\mu^\star$ (then $L_\pi=0$ locally),
or (ii) \emph{entropy-regularized policies} (softmax), which are globally Lipschitz.

\begin{lemma}[Softmax policy is Lipschitz in $Q$]
\label{lem:alg_softmax_Lip}
Fix $\tau>0$. For each $s\in S$ define
\begin{align}
\pi^{\tau}(a\mid s;Q) ~:=~ \frac{\exp(Q(s,a)/\tau)}{\sum_{b\in A}\exp(Q(s,b)/\tau)}.
\end{align}
Then for any two $Q,\tilde Q$,
\begin{align}
\|\pi^{\tau}(\cdot\mid s;Q)-\pi^{\tau}(\cdot\mid s;\tilde Q)\|_1
\le \frac{2}{\tau}\,\max_{a\in A}|Q(s,a)-\tilde Q(s,a)|.
\label{eq:alg_softmax_Lip}
\end{align}
Consequently, if $\max_{s,a}|Q_\mu(s,a)-Q_{\tilde\mu}(s,a)|\le L_Q\|\mu-\tilde\mu\|_1$, then \eqref{eq:alg_Lpi} holds with
$L_\pi \le 2L_Q/\tau$.
\end{lemma}

\begin{proof}
Let $u(a)=Q(s,a)/\tau$ and $\tilde u(a)=\tilde Q(s,a)/\tau$. The softmax map $u\mapsto \mathrm{softmax}(u)$ has Jacobian
$J(u)=\mathrm{diag}(\pi)-\pi\pi^\top$ whose operator norm from $\ell_\infty$ to $\ell_1$ is at most $2$.
Thus $\|\pi(u)-\pi(\tilde u)\|_1\le 2\|u-\tilde u\|_\infty = \frac{2}{\tau}\|Q(s,\cdot)-\tilde Q(s,\cdot)\|_\infty$,
which is \eqref{eq:alg_softmax_Lip}.
\end{proof}

Thus, we can replace the $\arg\max$ in \eqref{eq:alg_policy_selection} by a softmax policy. This allows us to consider more verifiable assumptions and still derive the convergence. See our discussion in \Cref{subsec:soft_rb_picard}.

It then suffices to verify the Lipschitz continuity of the worst-case kernel. We study it under three different distributional uncertainty sets.

\subsubsection{Total variation balls}
\label{subsubsec:alg_TV}

Consider the total variation ambiguity family:
\begin{align}
\mathfrak{P}_{\mathrm{TV}}(s,a,\mu)
~:=~
\Big\{p\in\Delta(\mathcal{S}):\ \|p-\bar p_{s,a}(\mu)\|_1 \le \varepsilon_{s,a}\Big\},
\qquad \varepsilon_{s,a}\ge 0.
\label{eq:alg_TV_ball}
\end{align}
Assume the center is Lipschitz:
\begin{align}
\max_{(s,a)}\|\bar p_{s,a}(\mu)-\bar p_{s,a}(\tilde\mu)\|_1 \le L_{\bar p}\|\mu-\tilde\mu\|_1.
\label{eq:alg_center_Lip_TV}
\end{align}

Any selection $p_\mu(\cdot\mid s,a)\in\mathfrak{P}_{\mathrm{TV}}(s,a,\mu)$ satisfies
\begin{align}
\|p_\mu(\cdot\mid s,a)-p_{\tilde\mu}(\cdot\mid s,a)\|_1
\le \|\bar p_{s,a}(\mu)-\bar p_{s,a}(\tilde\mu)\|_1 + 2\varepsilon_{s,a}
\le L_{\bar p}\|\mu-\tilde\mu\|_1 + 2\varepsilon_{s,a}.
\label{eq:alg_TV_inexact_Lip}
\end{align}
Consequently,
\begin{align}
\max_{s}\|K_\mu(\cdot\mid s)-K_{\tilde\mu}(\cdot\mid s)\|_1
\le (L_\pi+L_{\bar p})\|\mu-\tilde\mu\|_1 + 2\bar\varepsilon,
\qquad \bar\varepsilon:=\max_{s,a}\varepsilon_{s,a}.
\label{eq:alg_TV_inexact_LK}
\end{align}
Plugging \eqref{eq:alg_TV_inexact_LK} into the recursion yields convergence to an $O(\bar\varepsilon)$-neighborhood
via Theorem~\ref{thm:alg_inexact} (interpret $2\bar\varepsilon$ as a deterministic per-iteration update error).

\subsubsection{Wasserstein--1 balls}
\label{subsubsec:alg_Wass}

Consider the $W_1$ ambiguity family: fix a ground metric $d$ on $S$ and define
\begin{align}
\mathfrak{P}_{\mathrm{W}}(s,a,\mu)
~:=~
\Big\{p\in\Delta(\mathcal{S}):\ W_1\big(p,\bar p_{s,a}(\mu)\big)\le \rho_{s,a}(\mu)\Big\},
\qquad \rho_{s,a}(\mu)\ge 0.
\label{eq:alg_W_ball}
\end{align}
Let
\begin{align}
\underline d ~:=~ \min\{d(x,y):x\neq y\} \in (0,\infty),
\qquad
\bar\rho ~:=~ \max_{s,a,\mu}\rho_{s,a}(\mu).
\label{eq:alg_dmin_rhomax}
\end{align}

\begin{lemma}
\label{lem:alg_W1_to_L1}
For all $p,q\in\Delta(\mathcal{S})$,
\begin{align}
\|p-q\|_1 \le \frac{2}{\underline d}\,W_1(p,q).
\label{eq:alg_W1_to_L1}
\end{align}
\end{lemma}

\begin{proof}
For any coupling $\gamma\in\Gamma(p,q)$, we have $d(X,Y)\ge \underline d\,\mathbf{1}\{X\neq Y\}$ almost surely,
so $\mathbb{E}[d(X,Y)]\ge \underline d\,\mathbb{P}(X\neq Y)\ge \underline d\,d_{\mathrm{TV}}(p,q)=\underline d\,\|p-q\|_1/2$.
Taking the infimum over $\gamma$ yields \eqref{eq:alg_W1_to_L1}.
\end{proof}

Assume the center is Lipschitz in $W_1$:
\begin{align}
\max_{(s,a)} W_1\big(\bar p_{s,a}(\mu),\bar p_{s,a}(\tilde\mu)\big)
\le L_{\bar p}^{\mathrm{W}}\|\mu-\tilde\mu\|_1.
\label{eq:alg_center_Lip_W1}
\end{align}
Then for any selections $p_\mu(\cdot\mid s,a)\in\mathfrak{P}_{\mathrm{W}}(s,a,\mu)$ and $p_{\tilde\mu}(\cdot\mid s,a)\in\mathfrak{P}_{\mathrm{W}}(s,a,\tilde\mu)$,
\begin{align}
\|p_\mu(\cdot\mid s,a)-p_{\tilde\mu}(\cdot\mid s,a)\|_1
&\le \frac{2}{\underline d}\,W_1\big(p_\mu(\cdot\mid s,a),p_{\tilde\mu}(\cdot\mid s,a)\big)
\nonumber\\
&\le \frac{2}{\underline d}\Big(\rho_{s,a}(\mu) + W_1\big(\bar p_{s,a}(\mu),\bar p_{s,a}(\tilde\mu)\big) + \rho_{s,a}(\tilde\mu)\Big)
\nonumber\\
&\le \frac{2}{\underline d}\Big(L_{\bar p}^{\mathrm{W}}\|\mu-\tilde\mu\|_1 + 2\bar\rho\Big).
\label{eq:alg_W_inexact_Lp}
\end{align}
Therefore,
\begin{align}
\max_s\|K_\mu(\cdot\mid s)-K_{\tilde\mu}(\cdot\mid s)\|_1
\le \Big(L_\pi+\frac{2}{\underline d}L_{\bar p}^{\mathrm{W}}\Big)\|\mu-\tilde\mu\|_1 + \frac{4\bar\rho}{\underline d}.
\label{eq:alg_W_inexact_LK}
\end{align}
As in the TV case, \eqref{eq:alg_W_inexact_LK} yields convergence to an $O(\bar\rho)$-neighborhood via Theorem~\ref{thm:alg_inexact}.
Exact contraction (with no additive slack) typically requires additional structure (e.g., a unique stable optimizer selection or
a regularized OT-based inner minimization).

\subsection{Convergence}

\begin{lemma}
Let $\mu\in\Delta(\mathcal{S})$ and let $(\pi_\mu,p_\mu)$ be any selections as in \eqref{eq:alg_policy_selection}--\eqref{eq:alg_kernel_selection}, i.e., $\mathrm{supp}\,\pi_\mu(\cdot|s)\subseteq\mathsf{D}(s,\mu)$ and $p_\mu(\cdot|s,a)\in\widehat{\mathfrak{P}}(s,a,\mu)$ for all $(s,a)$. Then:
\begin{enumerate}[label=(\roman*)]
\item (robust optimality and worst-case attainment at $\mu$)
\[
V_\mu\;=\;\inf_{p\in\mathfrak{P}(\mu)}J_\mu(\pi_\mu,p)\;=\;J_\mu(\pi_\mu,p_\mu)\;=\;\inner{\mu}{v_\mu}.
\]
\item If in addition $\mu=F(\mu)=\mu K_\mu$, then $(\mu,\pi_\mu,p_\mu)$ is a stationary robust mean-field equilibrium in the sense of Definition~\ref{def:smfe}.
\end{enumerate}
\end{lemma}
 
\begin{proof}
Fix the population $\mu$ and apply Proposition~\ref{prop:robust-dp} at $\mu$: since $\pi_\mu$ is supported on the greedy sets $\mathsf{D}(\cdot,\mu)$ and $p_\mu(\cdot|s,a)\in\widehat{\mathfrak{P}}(s,a,\mu)$ pointwise, Proposition~\ref{prop:robust-dp} gives, for every $s\in\mathcal{S}$,
\begin{equation}\label{eq:lem1-perstate}
v_\mu(s)\;=\;\inf_{p\in\mathfrak{P}(\mu)}J_\mu(s;\pi_\mu,p)\;=\;J_\mu(s;\pi_\mu,p_\mu),
\qquad\text{and}\qquad
\inf_{p\in\mathfrak{P}(\mu)}J_\mu(s;\pi',p)\;\le\;v_\mu(s)\ \ \forall\pi'.
\end{equation}
Averaging the second identity in \eqref{eq:lem1-perstate} against $\mu$ gives $J_\mu(\pi_\mu,p_\mu)=\sum_s\mu(s)\,J_\mu(s;\pi_\mu,p_\mu)=\inner{\mu}{v_\mu}$.

Then, for any admissible stationary kernel $p\in\mathfrak{P}(\mu)$, the first identity in \eqref{eq:lem1-perstate} gives $J_\mu(s;\pi_\mu,p)\ge v_\mu(s)$ for every $s$, hence $J_\mu(\pi_\mu,p)\ge\inner{\mu}{v_\mu}$. 

Therefore, the infimum over $p$ of $J_\mu(\pi_\mu,p)$ equals $\inner{\mu}{v_\mu}$ and is attained at $p_\mu$. Finally, Proposition~\ref{prop:statewise-vs-distributional} gives $V_\mu=V_\mu(\mu)=\inner{\mu}{v_\mu}$, completing (i); in particular Definition~\ref{def:smfe}(i)--(ii) hold at the population $\mu$.
 
For (ii), it remains only to check Definition~\ref{def:smfe}(iii). The hypothesis $\mu=\mu K_\mu$ reads, coordinatewise,
\[
\mu(s')\;=\;\sum_{s\in\mathcal{S}}\mu(s)\sum_{a\in\mathcal{A}}\pi_\mu(a|s)\,p_\mu(s'|s,a)\qquad\forall s'\in\mathcal{S},
\]
which is the consistency condition for the triple $(\mu,\pi_\mu,p_\mu)$.
\end{proof}

\begin{lemma}[Contraction of the population operator]
\label{lem:alg_F_contraction}
Under Assumption~\ref{ass:alg_contractivity}, the map $F(\mu)=\mu K_\mu$ is a contraction in $\|\cdot\|_1$:
\begin{align}
\|F(\mu)-F(\tilde\mu)\|_1 \le \rho\,\|\mu-\tilde\mu\|_1,\qquad \forall \mu,\tilde\mu\in\Delta(\mathcal{S}),
\label{eq:alg_F_contr}
\end{align}
where $\rho=\rho_{\mathrm{mix}}+L_K<1$.
\end{lemma}

\begin{proof}
Add and subtract $\mu K_{\tilde\mu}$:
\begin{align}
\|F(\mu)-F(\tilde\mu)\|_1
=
\|\mu K_\mu - \tilde\mu K_{\tilde\mu}\|_1
\le
\|\mu K_\mu-\mu K_{\tilde\mu}\|_1 + \|\mu K_{\tilde\mu}-\tilde\mu K_{\tilde\mu}\|_1.
\end{align}
For the second term, 
first note that the Dobrushin contraction inequality implies
\begin{align}
\|\nu K - \tilde\nu K\|_1 \le \alpha(K)\,\|\nu-\tilde\nu\|_1,\quad \forall \nu,\tilde\nu\in\Delta(\mathcal{S}).
\label{eq:alg_dobrushin_contr}
\end{align}
Apply \eqref{eq:alg_dobrushin_contr} and \eqref{eq:alg_mix_assump}:
\begin{align}
\|\mu K_{\tilde\mu}-\tilde\mu K_{\tilde\mu}\|_1 \le \alpha(K_{\tilde\mu})\|\mu-\tilde\mu\|_1
\le \rho_{\mathrm{mix}}\|\mu-\tilde\mu\|_1.
\end{align}
For the first term,
\begin{align}
\|\mu K_\mu-\mu K_{\tilde\mu}\|_1
=
\left\|\sum_{s\in S}\mu(s)\big(K_\mu(\cdot\mid s)-K_{\tilde\mu}(\cdot\mid s)\big)\right\|_1
\le
\sum_{s\in S}\mu(s)\|K_\mu(\cdot\mid s)-K_{\tilde\mu}(\cdot\mid s)\|_1
\le
\max_{s}\|K_\mu(\cdot\mid s)-K_{\tilde\mu}(\cdot\mid s)\|_1.
\end{align}
Now apply \eqref{eq:alg_LK_assump} to get $\|\mu K_\mu-\mu K_{\tilde\mu}\|_1\le L_K\|\mu-\tilde\mu\|_1$.
Combining the two bounds yields \eqref{eq:alg_F_contr}.
\end{proof}

\begin{theorem}\label{thm:alg_picard_convergence}
Under Assumption~\ref{ass:alg_contractivity}, the map $F$ has a unique fixed point $\mu^\star$.
Moreover, the damped RB-Picard update
\begin{align}
\mu^{k+1} = (1-\alpha)\mu^k + \alpha F(\mu^k)
\label{eq:alg_damped_picard}
\end{align}
converges linearly to $\mu^\star$:
\begin{align}
\|\mu^{k}-\mu^\star\|_1
~\le~
\big(1-\alpha(1-\rho)\big)^{k}\,\|\mu^{0}-\mu^\star\|_1,
\qquad \forall k\ge 0.
\label{eq:alg_linear_rate}
\end{align}
In particular, with $\alpha=1$ (undamped iteration), $\|\mu^{k}-\mu^\star\|_1\le \rho^k\|\mu^0-\mu^\star\|_1$.
\end{theorem}

\begin{proof}
By Lemma~\ref{lem:alg_F_contraction}, $F$ is a contraction on the complete metric space $(\Delta(\mathcal{S}),\|\cdot\|_1)$,
so Banach's fixed point theorem yields a unique fixed point $\mu^\star$ and linear convergence for $\alpha=1$.
For the damped update \eqref{eq:alg_damped_picard},
\begin{align}
\|\mu^{k+1}-\mu^\star\|_1
\le (1-\alpha)\|\mu^k-\mu^\star\|_1 + \alpha\|F(\mu^k)-F(\mu^\star)\|_1
\le (1-\alpha+\alpha\rho)\|\mu^k-\mu^\star\|_1
=
\big(1-\alpha(1-\rho)\big)\|\mu^k-\mu^\star\|_1,
\end{align}
and iterating gives \eqref{eq:alg_linear_rate}.
\end{proof}

\begin{theorem}[Stability under inexact updates]
 
Assume Assumption~\ref{ass:alg_contractivity} and that Algorithm~\ref{alg:rb_picard} uses the inexact update
$\mu^{k+1}=(1-\alpha)\mu^k+\alpha\widehat F(\mu^k)$ with \eqref{eq:alg_inexact_update}.
Then, letting $q:=1-\alpha(1-\rho)\in(0,1)$,
\begin{align}
\|\mu^{k}-\mu^\star\|_1
~\le~
q^{k}\|\mu^{0}-\mu^\star\|_1
~+~
\alpha\sum_{j=0}^{k-1} q^{k-1-j}\varepsilon_j.
\label{eq:alg_inexact_bound}
\end{align}
In particular, if $\sup_j\varepsilon_j\le \bar\varepsilon$, then
$\limsup_{k\to\infty}\|\mu^{k}-\mu^\star\|_1 \le \bar\varepsilon/(1-\rho)$.
\end{theorem}

\begin{proof}
Write the recursion with the fixed point $\mu^\star=F(\mu^\star)$:
\begin{align}
\mu^{k+1}-\mu^\star
=
(1-\alpha)(\mu^k-\mu^\star)+\alpha\big(\widehat F(\mu^k)-F(\mu^\star)\big).
\end{align}
Add and subtract $F(\mu^k)$ and use triangle inequality:
\begin{align}
\|\mu^{k+1}-\mu^\star\|_1
\le
(1-\alpha)\|\mu^k-\mu^\star\|_1 + \alpha\|F(\mu^k)-F(\mu^\star)\|_1 + \alpha\|\widehat F(\mu^k)-F(\mu^k)\|_1.
\end{align}
Apply Lemma~\ref{lem:alg_F_contraction} and \eqref{eq:alg_inexact_update}:
\begin{align}
\|\mu^{k+1}-\mu^\star\|_1 \le (1-\alpha+\alpha\rho)\|\mu^k-\mu^\star\|_1 + \alpha\varepsilon_k = q\|\mu^k-\mu^\star\|_1+\alpha\varepsilon_k.
\end{align}
Unrolling yields \eqref{eq:alg_inexact_bound}. The $\limsup$ statement follows by bounding the geometric series.
\end{proof}

\subsection{Convergence under mixing setting}
Assume $\alpha(K_\mu)\le \rho_{\mathrm{mix}}<1$ for all $\mu$ so that each $K_\mu$ admits a unique stationary distribution.
Define the stationary-distribution map $\Phi:\Delta(\mathcal{S})\to\Delta(\mathcal{S})$ by
\begin{align}
\Phi(\mu) ~:=~ \text{the unique } \hat\mu \in \Delta(\mathcal{S})\text{ such that }\hat\mu=\hat\mu K_\mu.
\label{eq:alg_Phi}
\end{align}

\begin{theorem}[Contraction of the stationary-distribution map]
\label{thm:alg_stationary_map}
Suppose \eqref{eq:alg_mix_assump} and \eqref{eq:alg_LK_assump} hold with $\rho_{\mathrm{mix}}<1$.
Then $\Phi$ is Lipschitz with
\begin{align}
\|\Phi(\mu)-\Phi(\tilde\mu)\|_1 \le \frac{L_K}{1-\rho_{\mathrm{mix}}}\,\|\mu-\tilde\mu\|_1.
\label{eq:alg_Phi_Lip}
\end{align}
Consequently, if $L_K<1-\rho_{\mathrm{mix}}$ (equivalently $\rho_{\mathrm{mix}}+L_K<1$), then $\Phi$ is a contraction.
\end{theorem}

\begin{proof}
Let $\hat\mu=\Phi(\mu)$ and $\hat{\tilde\mu}=\Phi(\tilde\mu)$. Then
$\hat\mu=\hat\mu K_\mu$ and $\hat{\tilde\mu}=\hat{\tilde\mu}K_{\tilde\mu}$, so
\begin{align}
\hat\mu-\hat{\tilde\mu} = \hat\mu K_\mu - \hat{\tilde\mu}K_{\tilde\mu}
= (\hat\mu-\hat{\tilde\mu})K_\mu + \hat{\tilde\mu}(K_\mu-K_{\tilde\mu}).
\end{align}
Taking $\|\cdot\|_1$ norms and applying \eqref{eq:alg_dobrushin_contr} and \eqref{eq:alg_LK_assump} yields
\begin{align}
\|\hat\mu-\hat{\tilde\mu}\|_1 \le \alpha(K_\mu)\|\hat\mu-\hat{\tilde\mu}\|_1
+ \max_{s}\|K_\mu(\cdot\mid s)-K_{\tilde\mu}(\cdot\mid s)\|_1
\le \rho_{\mathrm{mix}}\|\hat\mu-\hat{\tilde\mu}\|_1 + L_K\|\mu-\tilde\mu\|_1.
\end{align}
Rearranging gives \eqref{eq:alg_Phi_Lip}.
\end{proof}


\subsection{Softmax robust best response and Soft Algorithm}
\label{subsec:soft_rb_picard}

Fix $\tau>0$. For a given population distribution $\mu\in\Delta(\mathcal{S})$ and $v\in\mathbb{R}^S$, define the robust $Q$-operator
\begin{align}
(Q_\mu v)(s,a) := \min_{P\in\mathfrak{P}(s,a,\mu)} \sum_{s'\in S} P(s')\big(r(s,a,s',\mu)+\gamma v(s')\big),
\end{align}
and the \emph{soft} robust Bellman operator
\begin{align}
(T_\mu^\tau v)(s) := \tau\log\sum_{a\in A}\exp\Big(\frac{(Q_\mu v)(s,a)}{\tau}\Big).
\end{align}

\begin{lemma} 
For each fixed $\mu$, the map $T_\mu^\tau$ is a contraction on $(\mathbb{R}^S,\|\cdot\|_\infty)$ with modulus $\gamma$.
Hence there exists a unique $v_\mu^\tau$ such that $v_\mu^\tau=T_\mu^\tau v_\mu^\tau$.
\end{lemma}

\begin{proof}
For any $s,a$ and any $v,w$, one has
$|(Q_\mu v)(s,a)-(Q_\mu w)(s,a)|\le \gamma\|v-w\|_\infty$.
Also, $\mathrm{LSE}_\tau(x):=\tau\log\sum_a e^{x_a/\tau}$ satisfies
$|\mathrm{LSE}_\tau(x)-\mathrm{LSE}_\tau(y)|\le \|x-y\|_\infty$.
Combining yields $\|T_\mu^\tau v-T_\mu^\tau w\|_\infty\le \gamma\|v-w\|_\infty$.
\end{proof}

Define the soft robust $Q$-function $Q_\mu^\tau(s,a):=(Q_\mu v_\mu^\tau)(s,a)$ and the SoftMax policy
\begin{align}
\pi_\mu^\tau(a\mid s):=\frac{\exp(Q_\mu^\tau(s,a)/\tau)}{\sum_{b\in A}\exp(Q_\mu^\tau(s,b)/\tau)}.
\end{align}
Choose a worst-case kernel selector
\begin{align}
p_\mu^\tau(\cdot\mid s,a)\in
\arg\min_{P\in\mathfrak{P}(s,a,\mu)}\sum_{s'\in S}P(s')\big(r(s,a,s',\mu)+\gamma v_\mu^\tau(s')\big).
\end{align}
Let $K_\mu^\tau(s'\mid s):=\sum_a \pi_\mu^\tau(a\mid s)p_\mu^\tau(s'\mid s,a)$ and define the population operator
$F^\tau(\mu):=\mu K_\mu^\tau$.

\begin{algorithm}[!htb]
\caption{Soft RB-Iteration}
\label{alg:soft_rb_picard}
\begin{algorithmic}[1]
\Require initial $\mu^0\in\Delta(\mathcal{S})$; temperature $\tau>0$; stepsize $\alpha\in(0,1]$;
tolerances $\varepsilon_{\mathrm{DP}},\varepsilon_\mu$.
\For{$k=0,1,2,\ldots$}
    \State \textbf{(Soft robust DP at $\mu^k$)} iterate $v^{(m+1)}\leftarrow T_{\mu^k}^\tau v^{(m)}$ until
    $\|v^{(m+1)}-v^{(m)}\|_\infty\le \varepsilon_{\mathrm{DP}}$; set $v_{\mu^k}^\tau\leftarrow v^{(m+1)}$.
    \State \textbf{(SoftMax policy)} set $Q^k(s,a):=(Q_{\mu^k}v_{\mu^k}^\tau)(s,a)$ and
    $\pi^k(a\mid s)\propto \exp(Q^k(s,a)/\tau)$.
    \State \textbf{(Worst-case kernel)} for each $(s,a)$ choose
    $p^k(\cdot\mid s,a)\in \arg\min_{P\in\mathfrak{P}(s,a,\mu^k)}\sum_{s'}P(s')\big(r+\gamma v_{\mu^k}^\tau\big)$.
    \State form $K^k(s'\mid s)=\sum_a \pi^k(a\mid s)p^k(s'\mid s,a)$ and $\tilde\mu^{k+1}\leftarrow \mu^kK^k$.
    \State \textbf{(Damped update)} $\mu^{k+1}\leftarrow (1-\alpha)\mu^k+\alpha\,\tilde\mu^{k+1}$.
    \If{$\|\mu^{k+1}-\mu^k\|_1\le \varepsilon_\mu$} \State \textbf{break} \EndIf
\EndFor
\State \Return $(\mu^{k+1},\pi^k,p^k)$.
\end{algorithmic}
\end{algorithm}

\begin{assumption}[Soft contractivity for $F^\tau$]
\label{ass:soft_contract}
There exist $\rho_{\mathrm{mix}}\in[0,1)$, $L_Q^\tau\ge 0$, and $L_p^\tau\ge 0$ such that for all $\mu,\tilde\mu$:
(i) $\alpha(K_\mu^\tau)\le \rho_{\mathrm{mix}}$;
(ii) $\max_{s,a}|Q_\mu^\tau(s,a)-Q_{\tilde\mu}^\tau(s,a)|\le L_Q^\tau\|\mu-\tilde\mu\|_1$;
(iii) $\max_{s,a}\|p_\mu^\tau(\cdot\mid s,a)-p_{\tilde\mu}^\tau(\cdot\mid s,a)\|_1\le L_p^\tau\|\mu-\tilde\mu\|_1$.
Assume $\rho^\tau:=\rho_{\mathrm{mix}}+L_p^\tau+\frac{2L_Q^\tau}{\tau}<1$.
\end{assumption}

\begin{theorem}[Linear convergence]
Under Assumption~\ref{ass:soft_contract}, $F^\tau$ has a unique fixed point $\mu^{\tau,\star}$.
Moreover, the update $\mu^{k+1}=(1-\alpha)\mu^k+\alpha F^\tau(\mu^k)$ satisfies
\begin{align}
\|\mu^k-\mu^{\tau,\star}\|_1\le (1-\alpha(1-\rho^\tau))^k\|\mu^0-\mu^{\tau,\star}\|_1.
\end{align}
\end{theorem}

\begin{proof}
Combine: (a) Dobrushin contraction in the initial measure with factor $\rho_{\mathrm{mix}}$,
(b) the kernel Lipschitz bound
$\max_s\|K_\mu^\tau(\cdot\mid s)-K_{\tilde\mu}^\tau(\cdot\mid s)\|_1\le L_p^\tau\|\mu-\tilde\mu\|_1 + \frac{2L_Q^\tau}{\tau}\|\mu-\tilde\mu\|_1$
(using SoftMax Lipschitz), to show $\|F^\tau(\mu)-F^\tau(\tilde\mu)\|_1\le \rho^\tau\|\mu-\tilde\mu\|_1$.
Then apply Banach fixed point theorem and the standard damped contraction recursion.
\end{proof}


We moreover quantify the closeness of the soft limit and the original one. 
\begin{proposition}[Bias of the soft equilibrium]\label{prop:softMFE}
Let $\tau>0$, let $\mu^{\tau,\star}=F^\tau(\mu^{\tau,\star})$ with associated $(\pi^\tau,p^\tau)$, and set $\varepsilon(\tau):=\dfrac{\tau\log|\mathcal{A}|}{1-\gamma}$. Write $\mu:=\mu^{\tau,\star}$. Then:
\begin{enumerate}[label=(\roman*)]
\item $v_\mu\ \le\ v^\tau_\mu\ \le\ v_\mu+\varepsilon(\tau)$ pointwise;
\item $v^\tau_\mu-\varepsilon(\tau)\ \le\ u^{\pi^\tau}_\mu\ \le\ v^\tau_\mu$, hence $\inf_p J_\mu(s;\pi^\tau,p)\ \ge\ v_\mu(s)-\varepsilon(\tau)$ for all $s$;
\item $u^{\pi^\tau}_\mu\ \le\ J_\mu(\cdot;\pi^\tau,p^\tau)\ \le\ u^{\pi^\tau}_\mu+\dfrac{2\gamma\,\varepsilon(\tau)}{1-\gamma}$ pointwise.
\end{enumerate}
Consequently $(\mu^{\tau,\star},\pi^\tau,p^\tau)$ is an $\bigl(\varepsilon(\tau),\,\tfrac{2\gamma}{1-\gamma}\varepsilon(\tau)\bigr)$-robust MFE, and $\tau$ trades equilibrium accuracy against the verifiability  ($L_\pi=2L^\tau_Q/\tau$).
\end{proposition}
 
\begin{proof}
Throughout, $\varepsilon_0:=\tau\log|\mathcal{A}|$, so $\varepsilon(\tau)=\varepsilon_0/(1-\gamma)$. We use the elementary perturbation fact: if $S$ is a monotone $\gamma$-contraction with $S(v+c\textbf{1})=Sv+\gamma c\textbf{1}$, and $Sw\ge w-\varepsilon_0\textbf{1}$ pointwise, then its fixed point $u$ satisfies $u\ge w-\varepsilon(\tau)\textbf{1}$\footnote{The results are from induction:  $S^{n+1}w\ge S(w-\varepsilon_0\sum_{i<n}\gamma^i\textbf{1})=Sw-\varepsilon_0\sum_{1\le i\le n}\gamma^i\textbf{1}\ge w-\varepsilon_0\sum_{i\le n}\gamma^i\textbf{1}$; let $n\to\infty$.}. Symmetrically $Sw\le w+\varepsilon_0\textbf{1}$ implies $u\le w+\varepsilon(\tau)\textbf{1}$, and $S_1\le S_2$ pointwise (both monotone contractions) implies their fixed points are ordered.
 
(i) $\max_a x_a\le\text{LSE}_\tau(x)\le\max_a x_a+\tau\log|\mathcal{A}|$ gives $\mathcal{T}_\mu v\le\mathcal{T}^\tau_\mu v\le\mathcal{T}_\mu v+\varepsilon_0\textbf{1}$ for every $v$; the ordering of fixed points gives $v_\mu\le v^\tau_\mu$, and applying the perturbation fact to $S=\mathcal{T}_\mu$, $w=v^\tau_\mu$ (note $\mathcal{T}_\mu v^\tau_\mu\ge\mathcal{T}^\tau_\mu v^\tau_\mu-\varepsilon_0\textbf{1}=v^\tau_\mu-\varepsilon_0\textbf{1}$) gives $v_\mu\ge v^\tau_\mu-\varepsilon(\tau)$.
 
(ii) The Fenchel identity $\text{LSE}_\tau(x)=\sum_a\pi^\tau_a x_a+\tau H(\pi^\tau)$ with $\pi^\tau=\mathrm{softmax}(x/\tau)$ and $H\le\log|\mathcal{A}|$ gives, at $x=Q^\tau_\mu(s,\cdot)$,
$(\mathcal{T}^{\pi^\tau}_\mu v^\tau_\mu)(s)=v^\tau_\mu(s)-\tau H(\pi^\tau(\cdot|s))\in[v^\tau_\mu(s)-\varepsilon_0,\ v^\tau_\mu(s)]$.
Apply the perturbation fact (both directions) to $S=\mathcal{T}^{\pi^\tau}_\mu$, $w=v^\tau_\mu$: its fixed point $u^{\pi^\tau}_\mu$ satisfies $v^\tau_\mu-\varepsilon(\tau)\le u^{\pi^\tau}_\mu\le v^\tau_\mu$.
 
(iii) The lower bound is $J_\mu(\cdot;\pi^\tau,p^\tau)\ge\inf_p J_\mu(\cdot;\pi^\tau,p)=u^{\pi^\tau}_\mu$. For the upper bound, write $u:=u^{\pi^\tau}_\mu$, $c_v:=r+\gamma v$. For each $(s,a)$, with $p^\tau$ the minimizer of $\langle\cdot,c_{v^\tau_\mu}\rangle$ over $\mathfrak{P}(s,a,\mu)$,
\[
\langle p^\tau, c_u\rangle\le\langle p^\tau,c_{v^\tau_\mu}\rangle+\gamma\|u-v^\tau_\mu\|_\infty
=\min_P\langle P,c_{v^\tau_\mu}\rangle+\gamma\|u-v^\tau_\mu\|_\infty
\le\min_P\langle P,c_u\rangle+2\gamma\|u-v^\tau_\mu\|_\infty .
\]
Since $\|u-v^\tau_\mu\|_\infty\le\varepsilon(\tau)$ by (ii), averaging over $\pi^\tau(\cdot|s)$ yields $\mathcal{T}^{\pi^\tau,p^\tau}_\mu u\le\mathcal{T}^{\pi^\tau}_\mu u+2\gamma\varepsilon(\tau)\textbf{1}=u+2\gamma\varepsilon(\tau)\textbf{1}$. The perturbation fact applied to $S=\mathcal{T}^{\pi^\tau,p^\tau}_\mu$, whose fixed point is $J_\mu(\cdot;\pi^\tau,p^\tau)$, gives the claim with constant $2\gamma\varepsilon(\tau)/(1-\gamma)$.
\end{proof}

\end{document}